\documentclass{article}

\usepackage{microtype}
\usepackage{graphicx}
\usepackage{subfigure}
\usepackage{booktabs} %

\usepackage{hyperref}
\usepackage{nicefrac}

\usepackage{multirow}

\usepackage[accepted]{icml2025}

\usepackage{amsmath}
\usepackage{amssymb}
\usepackage{mathtools}
\usepackage{amsthm}

\usepackage{amsfonts}
\usepackage{cancel}
\usepackage{blindtext}
\usepackage{setspace}
\usepackage{dsfont}
\usepackage{listings}
\usepackage{silence}
\usepackage{lscape}
\usepackage{xcolor}
\usepackage{titletoc}
\usepackage{tikz}
\usepackage{pgfplots}
\usepackage{pgfmath}
\pgfplotsset{compat=1.18} 
\usepackage{caption}

\usepackage{xspace} %

\allowdisplaybreaks
\usepackage[capitalize,noabbrev]{cleveref}

\theoremstyle{plain}
\newtheorem{theorem}{Theorem}[section]
\newtheorem{proposition}[theorem]{Proposition}
\newtheorem{lemma}[theorem]{Lemma}

\theoremstyle{definition}
\newtheorem{definition}[theorem]{Definition}
\newtheorem{assumption}[theorem]{Assumption}

\theoremstyle{remark}
\newtheorem{remark}[theorem]{Remark}

\crefname{axiom}{Axiom}{Axioms}
\crefname{assumption}{Assumption}{Assumptions}

\DeclareMathOperator*{\argmin}{arg\,min}

\newcommand{\acro}[1]{\textsc{#1}\xspace}
\newcommand{\FedGVI}{\acro{FedGVI}}
\newcommand{\FedAvg}{\acro{FedAvg}}
\newcommand{\FedPA}{\acro{FedPA}}
\newcommand{\FashionMNIST}{\acro{FashionMNIST}}
\newcommand{\MNIST}{\acro{MNIST}}
\newcommand{\betaPredBayes}{$\beta$--\acro{PredBayes}}

\newcommand{\BlackBox}{\rule{1.5ex}{1.5ex}}  %
\ifdefined\proof
\renewenvironment{proof}{\par\noindent{\bf Proof\ }}{\hfill\BlackBox\\[2mm]}
\else
\newenvironment{proof}{\par\noindent{\bf Proof\ }}{\hfill\BlackBox\\[2mm]}
\fi

\newcommand{\mopen}{$\mathcal{M}$--open }

\newcommand{\N}{\mathcal{N}}

\newcommand{\RoT}[4][]{P_{#1}(#2,#3,#4)}

\newcommand{\ti}{t}
\newcommand{\timax}{T}
\newcommand{\arbitraryindex}{i}
\newcommand{\client}{m}
\newcommand{\maxclient}{M}
\newcommand{\otherclient}{k}
\newcommand{\server}{s}
\newcommand{\approxDist}{q}%
\newcommand{\priorDist}{\pi}
\newcommand{\likelihoodDist}{p}
\newcommand{\DGP}{\mathbb{P}}
\newcommand{\dimension}{d}

\newcommand{\Q}{\mathcal{Q}}%
\newcommand{\K}{\boldsymbol{K}}%
\newcommand{\dataspace}{\Xi}
\newcommand{\datasigmaalgebra}{\mathcal{X}}
\newcommand{\outputspace}{\Upsilon}
\newcommand{\outputsigmaalgebra}{\mathcal{Y}}
\newcommand{\parameterspace}{\Theta}
\newcommand{\PTheta}{\mathcal{P}(\parameterspace)}%
\newcommand{\trueloss}{L}
\newcommand{\calL}{\mathcal{L}}
\newcommand{\smallloss}{\ell}
\newcommand{\kullbackleibler}{KL}
\newcommand{\divergence}[1][]{D}
\newcommand{\normaliser}{Z}
\newcommand{\wKLparam}{w}
\newcommand{\weightedKL}[1][\wKLparam]{\frac{1}{#1}\divergence_{\kullbackleibler}}
\newcommand{\GBIparam}{\beta}
\newcommand{\dampingparam}{\tau}

\newcommand{\operatordist}{\xi}
\newcommand{\covA}{\Lambda}
\newcommand{\covB}{\Sigma}
\newcommand{\meanA}{\boldsymbol{\nu}}
\newcommand{\meanB}{\boldsymbol{\mu}}

\newcommand{\weightfctindex}[2][]{{w_{#1}^{#2}}}

\newcommand{\tempcovA}[2]{{\covA_{#1}^{#2}}}
\newcommand{\tempcovB}[2]{{\covB_{#1}^{#2}}}
\newcommand{\tempmeanA}[2]{{\meanA_{#1}^{#2}}}
\newcommand{\tempmeanB}[2]{{\meanB_{#1}^{#2}}}
\newcommand{\auxfctname}{\gamma}
\newcommand{\auxfct}[2]{\auxfctname_{#1}^{#2}(\bt)}

\newcommand{\natparam}{\boldsymbol{\eta}}
\newcommand{\suffiecientstas}{\boldsymbol{\phi}}
\newcommand{\lognormaliser}{A}
\newcommand{\refmeasure}{h}
\newcommand{\expfamDGPy}[1][]{\likelihoodDist_\bt(\y_{#1})}
\newcommand{\expfamDGP}[1][]{\likelihoodDist_\bt(\x_{#1})}
\newcommand{\tempq}{\approxDist}

\newcommand{\tempnormaliser}{\tilde{\normaliser}}

\newcommand{\bk}{{\boldsymbol{\kappa}}} %
\newcommand{\bt}{{\boldsymbol{\theta}}} %
\newcommand{\indata}{x}
\newcommand{\outdata}{y}
\newcommand{\x}{\mathbf{\indata}} 
\newcommand{\y}{\mathbf{\outdata}}

\newcommand{\datasize}{n}
\newcommand{\clientdatasize}{\datasize_\client}

\newcommand{\R}{\mathbb{R}}%
\newcommand{\E}{\mathbb{E}} %
\newcommand{\Eq}{\mathbb{E}_{\approxDist(\bt)}}%
\newcommand{\Ebq}[1][]{\mathbb{E}_{\approxDist(\bt)}\left[#1\right]}%

\newcommand{\transpose}{\top}

\newcommand{\jacobian}[1][\bk]{\nabla_{#1}}
\newcommand{\hessian}[1][\bk]{\nabla\nabla_{#1}}

\newcommand{\D}[3][]{\divergence_{#1}(#2 : #3)}%
\newcommand{\kl}[2][\approxDist]{\divergence_{\kullbackleibler}(#1 : #2)}
\newcommand{\wkl}[3][\wKLparam]{\weightedKL[#1](#2 : #3)}

\newcommand{\kllarge}[2][\approxDist]{\divergence_{\kullbackleibler}\left(#1 : #2\right)}

\newcommand{\prior}[1][]{\priorDist^{#1}(\bt)}
\newcommand{\q}[2][]{\approxDist_{#1}^{#2}(\bt)}
\newcommand{\nll}[2]{-\log \likelihoodDist(#1|#2)}

\newcommand{\cavityDist}[1]{\approxDist^{\backslash #1}}
\newcommand{\cavity}[1]{\approxDist^{\backslash #1}(\bt)}
\newcommand{\paramq}[2][]{\approxDist_{#1}^{#2}(\bt|\bk^{#2})}
\newcommand{\paramqapprox}[1][]{\approxDist^{#1}(\bt|\bk^{#1})}

\newcommand{\qtilde}[2][]{\tempq_{#1}^{#2}(\bt)}
\newcommand{\conditionalq}[3][]{\approxDist_{#2}^{#1}(\bt|#3)}%
\newcommand{\classificationloss}[1][]{\trueloss_{#1}(\y_{#1};\bt,\x_{#1})}
\newcommand{\classificationlosstime}[2][]{\trueloss_{#1}^{#2}(\y_{#1};\bt,\x_{#1})}

\newcommand{\sameclassificationloss}[1][]{\trueloss(\y_{#1};\bt,\x_{#1})}
\newcommand{\sameclassificationlossjoint}[2]{\trueloss(\y_{#1}^{#2};\bt,\x_{#1}^{#2})}
\newcommand{\likelihood}[1][]{\likelihoodDist (\y_{#1}|\bt,\x_{#1})}

\newcommand{\lossapproximation}[2]{\smallloss_#1^{#2}(\bt)}

\newcommand{\paramlossapproximation}[2][]{\smallloss_#2^{#1}(\bt|\bk^{#1})}
\newcommand{\clientupdate}[2]{\Delta_{#1}^{#2}(\bt)}

\newcommand{\semicolonq}[3]{\approxDist_{#1}^{#2}(\bt;#3)}
\newcommand{\semicolonloss}[3]{\trueloss_{#1}^{#2}(\bt;#3)}
\newcommand{\semicolonapproxloss}[3]{\smallloss_{#1}^{#2}(\bt;#3)}
\newcommand{\semicoloncavity}[3]{\approxDist_{#2}^{\backslash #1}(\bt;#3)}
\newcommand{\semicolonupdate}[3]{\Delta_{#1}^{#2}(\bt;#3)}

\newcommand{\CavityOperator}[2][]{\operatordist_{#1}[#2](\bt)}

\newcommand{\Z}[2][]{\normaliser^{#1}_{#2}}

\newcommand{\inverse}{^{-1}}

\newcommand{\localobjective}[2]{\mathrm{Obj}(#1,#2)}
\newcommand{\globalobjective}[1]{\mathrm{Obj}(#1)}

\newcommand{\minQ}{\argmin_{\approxDist\in\Q}}

\newcommand{\minS}{\argmin_{\approxDist\in\PTheta}}

\newcommand{\renyi}{R\'enyi }

\newcommand\inputpgf[2]{{
		\let\pgfimageWithoutPath\pgfimage
		\renewcommand{\pgfimage}[2][]{\pgfimageWithoutPath[##1]{#1/##2}}
		\let\includegraphicsWithoutPath\includegraphics
		\renewcommand{\includegraphics}[2][]{\includegraphicsWithoutPath[##1]{#1/##2}}
		\input{#1/#2}
}}

\definecolor{ibm0}{RGB}{255,176,0}
\definecolor{ibm1}{RGB}{254,97,0}
\definecolor{ibm2}{RGB}{220,38,127}
\definecolor{ibm3}{RGB}{120,94,240}
\definecolor{ibm4}{RGB}{100,143,255}

\usepackage[textsize=tiny]{todonotes}

\icmltitlerunning{Federated Generalised Variational Inference: A Robust Probabilistic Federated Learning Framework}

\begin{document}

\twocolumn[
\icmltitle{Federated Generalised Variational Inference: \\A Robust Probabilistic Federated Learning Framework}

\icmlsetsymbol{equal}{*}

\begin{icmlauthorlist}
\icmlauthor{Terje Mildner}{wcomp}
\icmlauthor{Oliver Hamelijnck}{wcomp}
\icmlauthor{Paris Giampouras}{wcomp}
\icmlauthor{Theodoros Damoulas}{wcomp,wstats}
\end{icmlauthorlist}

\icmlaffiliation{wcomp}{University of Warwick, Department of Computer Science, Coventry, United Kingdom}
\icmlaffiliation{wstats}{University of Warwick, Department of Statistics, Coventry, United Kingdom}

\icmlcorrespondingauthor{Terje Mildner}{Terje.Mildner@warwick.ac.uk}
\icmlkeywords{Federated Learning, Probabilistic Machine Learning, Model Misspecification, Robustness, Generalised Variational Inference, ICML}

\vskip 0.3in
]

\printAffiliationsAndNotice{}  %

\begin{abstract}

We introduce \FedGVI, a probabilistic Federated Learning (FL) framework that is
robust to both prior and likelihood misspecification. \FedGVI addresses limitations in both frequentist and Bayesian FL by providing unbiased predictions under model misspecification, with calibrated uncertainty quantification. Our approach generalises previous FL approaches, specifically Partitioned  Variational Inference \citep{ashman2022}, by allowing robust and conjugate updates, decreasing computational complexity at the clients. We offer theoretical analysis in terms of fixed-point convergence, optimality of the cavity distribution, and provable robustness to likelihood misspecification. Further, we empirically demonstrate the effectiveness of \FedGVI in terms of improved robustness and predictive performance on multiple synthetic and real world classification data sets.

\end{abstract}

\section{Introduction}

Federated learning (FL) is a framework for the collaborative training of a global model by a collection of clients, without requiring proprietary data to be shared with a central server or other participating clients \citep{mcmahan2017}. This decentralised 
approach allows FL to be used on applications with strict data privacy constraints, such as in finance or healthcare \citep{kairouz2021}. However, due to the sensitive nature and complexity of these domains, both privacy and robustness to model misspecification are paramount.

The frequentist formulation of FL aims to minimise a global loss function by aggregating local gradients from clients. Early works include Federated Averaging \citep[\FedAvg,][]{mcmahan2017} which iterates between training clients locally and averaging updates on the server. This has sparked a large body of research on issues such as communication efficiency, data privacy, and data heterogeneity across clients \citep{hamer2020, malinovsky2020, reddi2021, chen2022, tenison2023, tziotis2023, li2024, demidovich2024}. There has been some work addressing robustness to adversarial clients \citep{allouah2024, bao2024} and data and system heterogeneity \citep{chen2022, zhao2023, heikillae2023}. However, these only provide point estimates, and do not allow principled uncertainty quantification, as required in many FL applications  \citep{jonker2024}.
In contrast, Bayesian FL approaches aim to update beliefs of a global model with data partitioned across clients. This largely builds on distributed inference methods such as the Bayesian Committee Machine \citep{tresp2000},  parallel MCMC \citep{ahn2014,mesquita2020}, or  Divide\&Conquer SMC \citep{chan2023}. Expectation Propagation \citep{minka2001b, vehtari2020} is naturally applicable to the distributed setting where local sites are iteratively refined.
This requires computing the cavity distribution that removes local sites from the current approximation. 
Partitioned Variational Inference \citep[PVI,][]{bui2018, ashman2022} takes this idea and proposes a distributed variational inference algorithm, which has been extended through MCMC \citep{guo2023} and Stochastic Gradient Langevin Dynamics (SGLD) \citep{mekkaoui2021}. Whilst these approaches quantify uncertainty, they are susceptible to model misspecification which can lead to inaccurate, overconfident predictions \citep{bernardo2000, bissiri2016,jeremias2022}. 

Current approaches to FL are inherently non-robust to model misspecification which leads to compromised performance and uncalibrated uncertainty quantification. We address these challenges by departing from the traditional Bayesian paradigm and propose a distributed Generalised Variational Inference framework that allows us to deal with model misspecification. In summary, our contributions are:

\begin{itemize}
    \item We introduce Federated Generalised Variational Inference (\FedGVI), a family of robust probabilistic algorithms for federated learning. 
    \item We prove that \FedGVI is robust to likelihood misspecification (\cref{thm:robustness}).
    \item  We demonstrate that \FedGVI generalises standard approaches such as PVI and \FedAvg (\cref{rem:pvi_rec,rem:fedavg}) and theoretically justify the use of the cavity distribution (\cref{thm:cavity}).
    \item We prove that, under suitable conditions, \FedGVI converges to Generalised Bayesian posteriors (\cref{col:GBI} and \cref{prop:conjugate}) that are computationally tractable.
    \item  We evaluate \FedGVI on a range of synthetic and real-world datasets, across multiple models, demonstrating improved robustness and predictive performance. 
\end{itemize}
In \cref{sec:prelim} we define model misspecification and recall methods that mitigate it in the non--distributed setting. \cref{sec:FedGVI} introduces our framework, which builds on these concepts and extends them to the federated setting.
We analyse the theoretical properties of \FedGVI in \cref{sec:theory}, including provable robustness. Finally, \cref{sec:Experiments} studies the empirical performance and gains of \FedGVI with multiple models and real world datasets such as Bayesian Neural Networks on  \MNIST and \FashionMNIST.
\footnote{Code to reproduce experiments can be found at \url{https://github.com/Terje-M/FedGVI}.}
\subsection{Related Work}

\paragraph{Robust Frequentist Federated Learning} In the frequentist setting, building on the seminal paper of \citet{mcmahan2017}, many approaches have aimed at mitigating challenges in FL, such as robustness to adversarial servers through secure aggregation \citep{chen2022}, to stragglers \citep{tziotis2023}, heterogenous data in out--of--distribution generalisation \citep{tenison2023}, heterogeneous and asynchronous clients \citep{fraboni2023}, or finding weaknesses in communications \citep{zhu2019, zhao2023}. 
More recently, %
work on robust server aggregations achieves robustness against Byzantine clients that aim to deteriorate model performance \citep{allouah2024, bao2024}. However these do not allow principled uncertainty quantification.

\paragraph{Federated Bayesian Inference}
Federated and distributed Bayesian methods aim to approximate the posterior as if it had been computed with the data of all clients available at a central server. 
Early work on distributed Bayesian inference includes Bayesian opinion pools \citep{genest1984, carvalho2023}, and the Bayesian Committee machine \citep{tresp2000}, which aim to find a consensus among a collection of Bayesian beliefs. 
Works that aim to operationalise this in the distributed setting, where data is split IID across clients, include Expectation Propagation \citep{minka2001b, opper2005, hasenclever2017, vehtari2020}, and consensus based Monte Carlo \citep{scott2016}.
In the Federated setting this assumption is often violated, as data is not split homogeneously and IID across participating devices. 
From this perspective, most approaches to Bayesian FL can be categorised into finding an approximate posterior through variational inference \citep{corinzia2021, ashman2022, kassab2022,heikillae2023, hassan2024, vedadi2024, swaroop2025},  Markov Chain Monte Carlo \citep{al-shedivat2021,mekkaoui2021,kotelevskii2022, guo2023, hasan2024}, Gaussian Processes \citep{achituve2021}, or directly learning a Bayesian neural network \citep{yurochkin2019, zhang2022}. Personalised or hierarchical Bayesian FL \citep{kotelevskii2022, zhang2022, kim2023, hassan2023, hassan2024, vedadi2024} allows for additional expressibility of client posteriors, especially under heterogeneity. 
However, none of these are inherently robust to contamination and model misspecification.

\paragraph{Robust Bayesian Inference} 
Although the existing Bayesian FL methods address some of the challenges of federated learning, such as communication constraints and data heterogeneity, they still aim to approximate the Bayesian posterior, which in itself is a flawed objective under model misspecification \citep{walker2013, berk1966, bernardo2000}.
In the global, non-federated case, several methods have been proposed to combat misspecification in the Bayesian setting \citep{gruenwald2012}, with the most promising direction being Generalised Bayesian Inference \citep{hooker2014, bissiri2016, ghosh2016a, jewson2018, miller2021, alquier2021,jeremias2022, matsubara2022}. In this work we capitalise on this front and bring robustness to model misspecification in the federated setting.

\section{Preliminaries}\label{sec:prelim}
\subsection{Notation and Model Misspecification}

Let $(\Omega, \mathcal{F}, P_0)$ be a probability space where $P_0$ is the data generating process, generating the observable random variables $X_1,...,X_n\equiv X_1^n$ taking values in the measurable space $(\dataspace, \datasigmaalgebra)$. 
Further, let $Y_1^n$ be observable random variables depending on $X_1^n$ respectively, taking values in $(\outputspace,\outputsigmaalgebra)$.
Denote their realisations $\{X_i=x_i,\, Y_i=y_i\}_{i=1}^n$, which are assumed to be partitioned across $M$ clients $\{\x_\client, \y_\client\}_{\client=1}^\maxclient$ each of size $n_m$. Consider hypothesis measures $P_\bt$ where $\bt$ takes values in $(\parameterspace,\mathcal{T})$, a measurable space, admitting densities $\likelihoodDist_\bt$. 
We study elements of $\PTheta$, the set of all probability measures on $(\parameterspace,\mathcal{T})$,
starting with prior $\Pi$ and updated to $Q$, dominated by some common measure $\mu$, and admitting densities $\priorDist$ and $\approxDist$ respectively. Naive Bayes updates $\prior$ to $\approxDist_B(\bt)$ through
\begin{equation}\label{eqn:Bayesian_posterior}
	\approxDist_B(\bt)= 
        \prior \, \textstyle{\prod_{\client=1}^M} \, 
        \likelihoodDist_\bt(\y_\client;\x_\client)
        \, 
     / \normaliser
\end{equation}
where $\normaliser=\int_\parameterspace \prod_{\client=1}^M \likelihoodDist_\bt(\y_\client;\x_\client) \, \Pi(d\bt)$ is the marginal likelihood.
Since we do not suppose that the prior $\Pi$, nor the likelihood $P_\bt$ are well specified, i.e. $P_0 \notin \PTheta$, we are in the \mopen setting \citep{bernardo2000}, the model misspecified, and the Bayesian posterior inappropriate.
\subsection{Model Misspecification}%

There are several different ways we can think about model misspecification under the \mopen assumption.

\paragraph{Prior Misspecification} 
The traditional Bayesian paradigm assumes that the prior encodes the best available judgement about $\bt$, which beyond simple settings, is never realised \citep{berger1985,jeremias2018}. Such misspecification is common; e.g. it is standard to use zero--mean Gaussian distributions on the weights of Bayesian Neural networks. This can have dire effects, for instance \citet{diaconis1986} demonstrate that multimodal priors in a location model can cause the posterior to not accumulate around $P_0$, even when the DGP is well specified, i.e. when $P_0\in\PTheta$. 
\paragraph{Likelihood Misspecification} 
One such example is where the hypothesis of interest is contaminated
, and an $\varepsilon$ fraction of the data (input and/or output variables) has some unknown data source. Formalising this we follow the definition of \citet{huber1964}:
\begin{definition}[Huber contamination] \label{def:huber}
Given an $\varepsilon\in(0,\frac12)$ and the uncontaminated distribution $P_\bt$ of inliers and some contaminating distribution $G$ of outliers, then $P_0$ is said to be an \textit{$\varepsilon$-corrupted version of $P_\bt$}; $P_0:=(1-\varepsilon)P_\bt+\varepsilon G$.
\end{definition} 
\begin{algorithm}[t]
\caption{\FedGVI \textsc{Server}}
\begin{algorithmic}[1]
    \STATE \textbf{Input:} $\prior$, $\Q$, $\divergence_\server$%
    \STATE \textbf{Define:} $\lossapproximation{\client}{(0)}=0$, $\lossapproximation{\server}{(0)}=0$, $\q[\server]{(0)}=\prior$
    \FOR{$\ti=1,...,\timax$}
    \FOR{$\client=1,...,\maxclient$ in parallel}
    \STATE$\clientupdate{\client}{(\ti)}\leftarrow$\texttt{CLIENT}($\q[\server]{(\ti-1)}$, $\Q$, $\client$)
    \ENDFOR
    \STATE Set $\lossapproximation{\server}{(\ti)}\leftarrow\lossapproximation{\server}{(\ti-1)}+\sum_{\client=1}^\maxclient\clientupdate{\client}{(\ti)}$
    \STATE Optimise $\q[\server]{(\ti)}$ according to \cref{eqn:server_optim}
    \ENDFOR
\end{algorithmic}
\label{alg:fedgvi_server}
\end{algorithm}
\subsection{Robust Bayesian Methods}
\paragraph{Generalised Bayesian Inference (GBI)} Instead of linking the parameter and data through likelihoods, \citet{bissiri2016} and \citet{miller2021} formalised a coherent Bayesian framework using loss functions leading to Gibbs posteriors \citep{alquier2016}. This was further utilised to deal with likelihood misspecification through robust losses, e.g \citet{jeremias2018}. Let $\trueloss:\parameterspace\times\dataspace\times\outputspace\rightarrow\R$ be such a loss, then the GBI posterior is given by:
\begin{equation}\label{eqn:GBI_posterior}
\approxDist_{\mathrm{GBI}}(\bt)=\prior\exp\left\{-\GBIparam\textstyle{\sum_{\client=1}^\maxclient}\sameclassificationloss[\client]\right\}/\normaliser
\end{equation}
with $\normaliser=\int_\parameterspace \exp\{-\GBIparam{\sum_{\client=1}^\maxclient}\sameclassificationloss[\client]\}\Pi(d\bt)$. Here, $\beta\in\mathbb{R}_{>0}$ is a learning rate parameter that determines how much weight we place on the observed data, similar to power posteriors in VI \citep{gruenwald2012, kallioinen2024}. This recovers $\approxDist_{B}(\bt)$ when the loss is the negative log--likelihood and $\GBIparam=1$.

\paragraph{Generalised Variational Inference (GVI)} 
In \citet{jeremias2022} GBI is generalised within a variational framework that explicitly accounts for prior and likelihood misspecification. Let  $\divergence:\PTheta\times\PTheta\rightarrow\R_+$ be a divergence then the GVI posteriors are defined as:
\begin{equation*}
\begin{aligned}
\approxDist_{\mathrm{GVI}}(\bt)=\argmin_{\approxDist\in\Q}\left\{\Ebq[{\trueloss(\y_1^\maxclient;\bt,\x_1^\maxclient)}] + \D{\approxDist}{\priorDist}\right\}
\end{aligned}
\end{equation*}
where $\Q\subset \PTheta$, making inference tractable. 
This allows for targeting a larger subspace of posteriors,
and through different divergences the effect of the prior can be controlled.

\section{Federated Generalised Variational Inference}\label{sec:FedGVI}
\begin{algorithm}[t]
\caption{\FedGVI \textsc{Client}}
\begin{algorithmic}[1]
    \STATE \textbf{Input:} $\q[\server]{(\ti-1)}$, $\Q$, $\{\x_\client,\y_\client\}$, $\trueloss_\client$, $\lossapproximation{\client}{(\ti-1)}$, $\divergence$
    \STATE Optimise $\cavity{\client}$ according to \cref{eqn:cavity}
    \STATE Optimise $\q[\client]{(\ti)}$ according to \cref{eqn:local_optim}
    \STATE Set $\clientupdate{\client}{(\ti)}$ according to \cref{eqn:local_update}
    \STATE Set $\lossapproximation{\client}{(\ti)}\leftarrow \lossapproximation{\client}{(\ti-1)}+ \clientupdate{\client}{(\ti)}$
    \STATE\textbf{return:} Communicate $\clientupdate{\client}{(\ti)}$ to \texttt{SERVER}
\end{algorithmic}
\label{alg:fedgvi_client}
\end{algorithm}%
\subsection{Methodology}\label{sec:method}
In this section, we present the proposed federated learning framework, named \FedGVI, that explicitly addresses likelihood and prior misspecification. We aim to learn a robust approximate posterior $q_{\server}(\bt)$ using partitioned observations across M clients. \FedGVI  iterates consist of two steps: a) sending of the current approximate posterior to each client, which is updated through a robust variational objective, and b) aggregating the updates on the server, resulting in a robust approximate posterior; summarised in \cref{alg:fedgvi_server,,alg:fedgvi_client}.
\paragraph{Initialisation} 
We set the initial server posterior as the prior, $\q[\server]{(0)}=\prior$, and the local and server loss approximations to be zero, $\lossapproximation{\client}{(0)}=0$ and $\lossapproximation{\server}{(0)}=0$ respectively; $\client$ denotes a specific client and 
$\server$ the server. 
\paragraph{Until Convergence} For $\ti=1,2,...,\timax$, we synchronously compute updates locally at each client, and accumulate these at the server to form the new global posterior $\q[\server]{(\ti)}$.

\paragraph{Client }

The client receives the current approximate posterior from the server. This will be used as the prior from which a client can compute an updated posterior using their local data. First, however the information of the client's data must be removed by computing the cavity distribution. The cavity distribution acts as the local prior incorporating all previous information from all other clients and is given by:
\begin{equation}\label{eqn:cavity}
	\cavity{\client}\propto \frac{\q[\server]{(\ti-1)}}{\exp\{-\lossapproximation{\client}{(\ti-1)}\}}
\end{equation}
The client then computes a robust local approximate posterior with it's local data set $\{\x_\client,\y_\client\}$ and it's loss function $\trueloss_\client^{(\ti)}(\cdot)$, which is regularised by the divergence, $\divergence$, and cavity distribution
\begin{equation}\label{eqn:local_optim}
	\begin{aligned}
		\q[\client]{(\ti)}=\minQ \Ebq[{\classificationlosstime[\client]{(\ti)}}]+\D{\approxDist}{{\cavityDist{\client}}}.
	\end{aligned}
\end{equation}
This GVI style objective allows the client to be robust to both likelihood misspecification as well as prior misspecification arising due to the cavity.
To update the global posterior at the server, the client computes the negative log ratio of the local and global posteriors.
In line with existing Bayesian FL \citep{ashman2022,guo2023}, we use a damping parameter $\dampingparam_\client\in (0,1]$, which is analogous to a learning rate as in frequentist FL, to compute the update:
\begin{equation}\label{eqn:local_update}
	\clientupdate{\client}{(\ti)}=-\dampingparam_\client\log\frac{\q[\client]{(\ti)}}{\q[\server]{(\ti-1)}} 
\end{equation}
The client stores $\lossapproximation{\client}{(\ti)}:= \lossapproximation{\client}{(\ti - 1)} +  \clientupdate{\client}{(\ti)}$ and communicates $\clientupdate{\client}{(\ti)}$ to the server.

\paragraph{Server}
The loss at the server is updated based on the received client updates,
\begin{equation}\label{eqn:server_update}
	\lossapproximation{\server}{(\ti)}=\lossapproximation{\server}{(\ti-1)}+ \textstyle{\sum_{\client=1}^\maxclient}\clientupdate{\client}{(\ti)}
\end{equation}
By only incorporating clients' updates that have changed we can trivially allow for batched and asynchronous scheduling of clients. The updated loss is then used to compute the new server posterior though a GVI optimisation procedure:
\begin{equation}\label{eqn:server_optim}
	\q[\server]{(\ti)}=\minQ \Ebq[{\lossapproximation{\server}{(\ti)}}]+\D[\server]{\approxDist}{\priorDist}
\end{equation}
This posterior and loss are passed back to the clients for further refinement at the next iteration until convergence.
\subsubsection{Hyperparameters}
\citet{ashman2022} set the damping parameter to $\dampingparam\propto\frac1\maxclient$ throughout their experiments. This turns out, see \cref{lem:damping}, to be a reasonable choice when $\dampingparam=\frac1\maxclient$ in combination with $\divergence_\server=\divergence_{\kullbackleibler}$ since this causes the posterior at the server to be a logarithmic opinion pool induced by an externally Bayesian pooling operator  \citep{genest1986}, ensuring stable convergence.
Other hyperparameters arising from the choice of losses and divergences are dependent on the expected amount of model misspecification.

\subsection{Robustness to Likelihood Misspecification}\label{sec:robust_model}
Within our framework we are free to choose the client side losses. We consider the Density--Power divergence based loss \citep{ghosh2016b}, often referred to as $\beta$--divergence loss $\mathcal{L}_\beta$, the $\gamma$--divergence based losses \citep{hung2018}, $\mathcal{L}_{\gamma}$, as well as a score matching loss, $\mathcal{L}_{SM}$, based on the Hyv{\"a}rinen divergence \citep{hyvarinen2005,altamirano2023}. In the classification setting, we consider the generalised cross--entropy loss
\begin{equation}
\mathcal{L}_{GCE}^{(\delta)}(y_i;\bt,x_i) = \frac{(1-p_\bt(y=y_i;x_i)^\delta)}{\delta}
\end{equation}
for some $\delta\in(0,1]$ \citep{zhang2018}. These losses are robust to misspecification because they have a finite supremum (see \cref{def:robust_loss}). %
It is important to highlight that GVI and \FedGVI may underperform when using robust losses in the case of correct likelihood specification; see \citet{jeremias2022}.
We can use a Sequential Monte Carlo sampler to estimate the $\beta$ or $\gamma$ hyperparameters in $\mathcal{L}_\beta$ and $\mathcal{L}_\gamma$ \citep{yonekura2023} or use cross validation to select optimal parameters \citep{altamirano2024}.

\subsection{Robustness to Prior Misspecification}\label{sec:divergences}

We mainly consider the weighted Kullback--Leiber divergence, $\weightedKL$, \cite{kullback1951}
\begin{equation*}
    {\weightedKL}(\approxDist:\priorDist):= \frac{1}{w} \, \Ebq[{\log\frac{\approxDist(\bt)}{\prior}}]
    ,
\end{equation*}
and the Alpha--\renyi divergence, $\divergence_{AR}^{(\alpha)}$,
\begin{equation*}
    \divergence_{AR}^{(\alpha)}(\approxDist:\priorDist):=\frac{1}{\alpha(\alpha-1)}\log\left(\E_{\prior}\left[\left(\frac{\approxDist(\bt)}{\prior}\right)^\alpha\right]\right).
\end{equation*}
As examined in \citet{jeremias2022}, $\divergence_{AR}^{(\alpha)}$ allows for different prior regularisation depending on how much we trust the prior by placing different weights on it. In future work it would be simple to explore other divergences such as the $f$--divergences, $\divergence_f$, \citep{amari2016, alquier2021}. 
Similarly to the losses, we can perform cross validation to select the $\alpha$ parameter, however as demonstrated in the ablation study (\cref{fig:ablation_study}) FedGVI performs favourably under a range of $\alpha$ (and $
\delta$) values.
\section{Theoretical Results}\label{sec:theory}
We now present a theoretical analysis of \FedGVI.
We begin by examining the relationship of \FedGVI with other FL algorithms while recovering some of them as special cases, we study the damping parameter, and examine the convergence behaviour of \FedGVI.
Then, we turn our attention on robustness to likelihood misspecification, where we first study \FedGVI as distributed GBI, from which we derive a theorem on the necessity of the cavity distribution.
Finally, we derive a result for computationally tractable and conjugate \FedGVI, enabling us to present the main theorem on bias--robustness of \FedGVI.

Since it is an open problem where global GVI posteriors converge to under arbitrary divergences, we often have to restrict ourselves to consider the server divergence to be the Kullback--Leibler divergence. This ensures that the posterior at the server will have the structure of a GBI posterior,
\begin{equation*}\label{eqn:posterior_form}
	\q[\server]{(\timax)}\propto \exp\left\{\textstyle{-\sum_{\client=1}^\maxclient}\lossapproximation{\client}{(\timax)}\right\}\prior
\end{equation*}
where we incorporate prior robustness and tractability through the approximate losses.
\subsection{Recovering Existing Methods as a Special Case}
By choosing specific divergences, loss functions, and variational families, we can recover existing methods as special cases of our framework, which we summarise in \cref{fig:fedgvi_relation_small}:

\begin{remark}\label{rem:pvi_rec}
	Choosing the Kullback--Leibler divergence and the negative log--likelihood as a loss function recovers the PVI algorithm of \citet{ashman2022}.
\end{remark}
\begin{remark}\label{rem:fedavg}
    When $\divergence=\divergence_\server=0$, and $\Q=\{\delta_{\hat{\bt}}(\bt): \hat{\bt}\in\parameterspace\}$, with $\delta_{\hat{\bt}}$ being the Dirac--delta measure at some element $\hat{\bt}$, we recover \FedAvg of \citet{mcmahan2017}.
\end{remark}
\begin{figure}[h]
\centering
\begin{tikzpicture}

\node[align=center] at (0,0) {\FedGVI\\
$L, D,\Q, \maxclient, D_\server$};
\node[align=center] at (2.5,-1.2) {VI\\
$-\log\likelihoodDist_\bt, D_{KL}, \Q,$\\$ \maxclient=1,D_\server=D_{KL}$};
\node[align=center] at (2.5,1.2) {PVI\\
$-\log\likelihoodDist_\bt, D_{KL}, \Q$,\\$M,D_\server=D_{KL}$};
\node[align=center] at (-2.5,-1.2) {ERM\\
$L, D=0, \{\delta_\bt\}$, \\$\maxclient=1, D_\server=0$};
\node[align=center] at (-2.5,1.2) {\FedAvg\\
$L, D=0, \{\delta_\bt\},$\\ $M,
D_\server=0$};

\draw[line width=0.75pt, ->] (-0.05,-.45) .. controls (-.05,-1.35) .. (-1.2,-1.35);
\draw[line width=0.75pt, ->] (0.05,-.45) .. controls (.05,-1.35) .. (1.2,-1.35);

\draw[line width=0.75pt, ->] (-0.05,.45) .. controls (-0.05,1.45) .. (-1.4,1.45);
\draw[line width=0.75pt, ->] (0.05,.45) .. controls (0.05,1.45) .. (1.4,1.45);

\draw[line width=0.75pt, <-] (-2.5,-.5) -- (-2.5,.5);
\draw[line width=0.75pt, <-] (2.5,-.5) -- (2.5,.5);

\end{tikzpicture}
\caption{We illustrate the relationship of \FedGVI---characterised by the loss $L$, the client divergence $D$, the variational family $\Q$, the number of clients $M$, and the divergence at the server $D_\server$---to Partitioned Variational Inference (PVI), Variational Inference (VI), Federated Averaging (\FedAvg), and Empirical Risk Minimisation (ERM).}
\label{fig:fedgvi_relation_small}
\end{figure}

\subsection{Damping as a Bayesian Logarithmic Opinion Pool}
Choosing the damping parameter to be $\dampingparam=1/\maxclient$ results in a logarithmic opinion pool. 
In fact choosing damping parameters such that all of them sum to unity also forms a valid logarithmic opinion pool \citep{genest1986}.
\begin{proposition}\label{lem:damping}
Assume $\divergence_\server=\divergence_{\kullbackleibler}$, and that $\sum_{\client}\dampingparam_\client=1$ where $\dampingparam_\client\ge0$ $\forall\client$, then the posterior at the server is an externally Bayesian logarithmic opinion pool of the form
\begin{equation*}
    \q[\server]{(\ti)}=\frac{\prod_{\client=1}^\maxclient\bigl(\q[\client]{(\ti)}\bigr)^{\dampingparam_\client}}{\int_\parameterspace\prod_{\client=1}^\maxclient\bigl(\q[\client]{(\ti)}\bigr)^{\dampingparam_\client} d\bt},\; \bt-a.e.
\end{equation*}
\end{proposition}
See \cref{apx:damping} for the proof. 
This results provides a theoretical justification on the previously heuristic use  of the damping parameter  \citep[as used in PVI,][]{ashman2022}. Specifically it ensures  that  this selection of $\tau$ leads to a valid distribution and results in more stable convergence.

\subsection{Fixed Points of \FedGVI}
In this section we study the properties of \FedGVI posteriors when these converge to some fixed point. Specifically, we generalise the fixed point result of PVI \citep[Property 2.3]{ashman2022} to arbitrary losses.
\begin{proposition}\label{thm:fixed_points}
Let $\divergence_\server=\divergence_{\kullbackleibler}$, $\divergence=\weightedKL$, $\wKLparam>0$, and $\Q\subset\PTheta$, then if $\q[\server]{*}=\prior\exp\{-\lossapproximation{\server}{*}\}/{\Z{\approxDist^*}}$ such that $\forall\client\in[\maxclient]$, $\clientupdate{\client}{*}=0$, then $\q[\server]{*}$ is a local minimiser of the following GVI objective:
\begin{equation}\label{eqn:gvi_objective}
\Ebq[{\sum_{\client=1}^\maxclient\classificationloss[\client]}]+\wkl{\approxDist}{\priorDist}
\end{equation}
\end{proposition}
\begin{remark}
If the loss in \cref{eqn:gvi_objective} is convex, then a fixed point of \FedGVI is a global minimum of GVI.
\end{remark}

This illustrates that if \FedGVI converges, then the posterior is a (local) minimiser of the GVI objective. We refer to such distributions as fixed points. This recovers \citet[Theorem 1]{kassab2022} (which deals with the restricted case of $\Q=\PTheta$) with a novel proof; see \cref{apx:fixed}.

\subsection{Generalised Bayesian Inference}
As a consequence of \cref{thm:fixed_points} and \cref{rem:pvi_rec}, \FedGVI will recover the GBI posterior when $\Q=\PTheta$.
\begin{lemma}\label{col:GBI}
	Assuming $\Q=\PTheta$, $\divergence=\weightedKL[\GBIparam]$ with $\GBIparam>0$, $\divergence_\server=\divergence_{\kullbackleibler}$, and $\dampingparam=1$, then \FedGVI will recover the GBI posterior after the first iteration.
	\begin{flalign*}
		&\q[\server]{(1)}=\conditionalq{GBI}{{\{\x_\client,\y_\client\}_{\client=1}^\maxclient}}\\&=\exp\{-\GBIparam\textstyle{\sum_{\client=1}^\maxclient}{\sameclassificationloss[\client]}\}\prior /\normaliser
	\end{flalign*}
	This posterior is invariant under subsequent iterations of \FedGVI, having reached a fixed point.
    
    Moreover, for a damping rate $\dampingparam=1/\maxclient$, the posterior at the server converges pointwise a.e. in $\parameterspace$ to the GBI posterior, 
    \begin{equation*}
        \q[\server]{(\timax)}\overset{\timax\rightarrow\infty}{\longrightarrow}\conditionalq{GBI}{{\{\x_\client,\y_\client\}_{\client=1}^\maxclient}},\, \bt-a.e.
    \end{equation*}
\end{lemma}
This result, proven in \cref{apx:GBI}, is the first step towards likelihood robustness. 
If we were able to find the GBI posterior efficiently with some robust loss, then the posterior would be robust and computable. 
Here however, the loss may not vary over different iterations of \FedGVI as in \cref{eqn:local_optim} and the normaliser may be intractable.
\subsection{The Cavity Distribution is Necessary}
By further investigating the relationship of \FedGVI with the GBI posterior, we can extend \cref{col:GBI} and derive a Theorem under which we are required to use the cavity distribution to regularise the client update.
This is in contrast to both PVI, where it's use is heuristically justified, and to other Bayesian FL approaches where the previous posterior is used instead.
For this we recall two natural assumptions that any such distribution must satisfy in a federated setting.
\begin{assumption}\label{axm:data} 
	No client can have access to the data set of another client.
\end{assumption}
\begin{assumption}\label{axm:update}
	Each client generates their update equivalently to other clients.
\end{assumption}
These assumptions combined with \cref{col:GBI} lead us to the necessity of the cavity distribution.
\begin{theorem}\label{thm:cavity}
	Let the assumptions be as in \cref{col:GBI} with $\dampingparam=1$, and assume that the \cref{axm:data,axm:update} are satisfied, then $(1.)$ holds if and only if $(2.)$ holds.
    $1.$ \FedGVI recovers the generalised Bayesian posterior $\approxDist_{\mathrm{GBI}}(\bt)$ which is invariant under further \FedGVI updates.

	$2.$ The cavity regularises the client optimisation problem.
\end{theorem}
This provides a principled justification for the use of the cavity distribution, as defined in \cref{eqn:cavity}, in \FedGVI.
We provide the proof in \cref{apx:cavity}.

\subsection{Conjugate Client Updates}
Before we present our main result on provable robustness to likelihood misspecification, we first show that we can find a GBI posterior under specific losses in a computationally tractable manner.
Assuming that the data generating process has some exponential family distribution, where $\y\sim\expfamDGPy$,
\begin{equation*}
	\expfamDGPy=\exp\{\natparam(\bt)^\transpose\suffiecientstas(\y)-\lognormaliser(\natparam(\bt))+\refmeasure(\y)\},
\end{equation*} 
such that this is differentiable in $\y$, by using the weighted score matching loss of \citet{altamirano2023}, $\mathcal{L}_{SM}^w$, then client updates, using the weighted KL divergence locally, are available in closed form. 
If we further assume that our model is Gaussian, or has the form of a squared exponential, and that the natural parameters of the DGP are $\natparam(\bt)=\bt$, then the client approximation will have a conjugate form.
\begin{proposition}\label{prop:conjugate}
Assume that the hypothesis $\expfamDGPy$ has differentiable, exponential family distribution with $\natparam(\bt)=\bt$, $\trueloss_\client^{(\ti)}=\mathcal{L}_{SM}^{w_\client^\ti}$, and $\divergence=\frac1\GBIparam\divergence_{\kullbackleibler}$, and the variational family $\Q$ is the multivariate Gaussians, then the local posteriors at the clients are conjugate Gaussians. Moreover, \cref{eqn:server_optim} will have closed form if $\divergence_\server$ has closed form between Gaussian distributions.
\end{proposition}
See \cref{apx:conjugate} for the proof.
The loss may now depend on the client and iteration $\ti$. 
Most exponential family distributions satisfy the conditions of the proposition, and there are several divergences that allow closed form expressions between Gaussians, such as the Alpha--\renyi, or the $\alpha,\beta,\gamma$--divergences of \citet{chichocki2010}. Further, this enables the use of intractable likelihood models.
\subsection{Provable Robustness to Outliers}
For a robust loss function at the clients, and using the weighted KL divergence at the clients and the KL divergence at the server, guarantees that after $\timax$ iterations, the posterior computed at the server will also be robust to outliers. 
This means we can achieve robustness at the server by leveraging the robust losses that were derived for GVI.
In this, we mean robustness as defined by \citet{ghosh2016a} and further developed in \citet{matsubara2022}.
We define the empirical DGP of a client as $\DGP_{n_\client}:=\frac{1}{n_\client}\sum_{i=1}^{n_\client}\delta_{\x_i}$, and of the entire data set as $\DGP_{n}:=\frac{1}{n}\sum_{\client=1}^{\maxclient}n_\client\DGP_{n_\client}$. 
When this is contaminated by some $\varepsilon$ fraction of data centred at some adversarially chosen data point $z\in\dataspace$, the misspecified DGP is defined as $\DGP_{n,\varepsilon,z}:=(1-\varepsilon)\DGP_n+\varepsilon\delta_z$.
\begin{definition}
\label{def:robust_loss}
We say that a loss $\trueloss_\client^{(\ti)}(\bt;\DGP_{n_\client,\varepsilon,z})$, w.r.t. some prior distribution $\prior$, is robust to outliers, if the following hold:
\begin{align*}
    1.&\quad \sup_{z\in\dataspace}\left|\frac{d}{d\varepsilon}\trueloss_\client^{(\ti)}(\bt;\DGP_{n_\client,\varepsilon,z})\big|_{\varepsilon=0}\right|\le\auxfct{{(\client)}}{(\ti)},\\
    2.&\quad \sup_{\bt\in\parameterspace}\prior\auxfct{{(\client)}}{(\ti)}<\infty, \;\mathrm{and}\\
    3.&\quad\int_\parameterspace\prior\auxfct{{(\client)}}{(\ti)}\mu(d\bt)<\infty
\end{align*}
\end{definition}
These conditions ensure that the influence of arbitrary contamination on the local posterior is not arbitrarily bad. 
In particular the auxiliary function $\gamma_m^{(t)}$ ensures that the influence of an adversarial data point $z$ on the posterior over infinitesimal contaminations, $\frac d{d\epsilon}q_m^{(t)}(\theta;\mathbb{P}_{n_m,\epsilon, z})|_{\epsilon=0}$, are finite over all $\theta$ and $z$. Condition 2 ensures the loss increases slowly enough for the local posterior to concentrate around the data, and condition 3 ensures the resulting posterior will be normalisable.
\begin{theorem}\label{thm:robustness}
	Let $\divergence_\server=\divergence_{\kullbackleibler}$, $\divergence=\weightedKL$, $\Q=\PTheta$, further assume that the prior is upper bounded and the loss is lower bounded, then if $\forall\ti\in[\timax]$ and $\forall\client\in[\maxclient]$ $\trueloss_\client^{(\ti)}(\bt;\DGP_{n_\client,\varepsilon,z})$ is robust, then the posterior generated by \FedGVI will be robust to outliers.%
\end{theorem}
The proof is in \cref{apx:provablerobustness}. This result together with \cref{prop:conjugate} is significant as we have robustness under intractable optimisation, \emph{and} we can choose a provably robust, conjugate loss to generate robust \FedGVI posteriors, which are then computationally efficient to compute.

\section{Experiments}\label{sec:Experiments}
We evaluate \FedGVI against several other methods, specifically PVI \citep{ashman2022}, \FedAvg \citep{mcmahan2017}, the nonparametric DSVGD \citep{kassab2022}, the distributed MCMC based DSGLD \citep{ahn2014}, federated MCMC based \FedPA \citep{al-shedivat2021}, and the one shot BCM based approach \betaPredBayes \citep{hasan2024}. We provide further details about experiments in \cref{apx:additional_exp_details}. %

\subsection{1D Clutter Problem}
\begin{figure}[h]
	\centering
	\input{{./figs/like_misspec.pgf}}
	\caption{Robustness to outliers can be achieved through varying losses with \FedGVI, while traditional Bayesian methods fail. %
    }
	\label{fig:loc_misspec}
\end{figure}
We first examine the effect of misspecified likelihoods through the well known clutter problem \citep{minka2001b}. We generate 100 observations from a Gaussian location model that is contaminated through \cref{def:huber} with $\varepsilon=0.25$ Gaussian noise. The aim is to infer the location parameter $\bt$ of the uncontaminated data. We compare \FedGVI with both $\mathcal{L}_\beta$ and $\mathcal{L}_{SM}$ vs PVI with and without misspecification. We also provide the corresponding MLE results. See \cref{fig:loc_misspec}. Under misspecification both the MLE and PVI fail to recover the true $\bt$, whereas \FedGVI can easily handle different levels of contamination.%

\subsection{Influence Function}

\begin{figure}[h]
	\centering
	\input{{./figs/if.pgf}}
	\caption{We plot the influence of a single outlier on the server posterior. PVI is not robust to likelihood misspecification through outliers, because it uses the negative log--likelihood (NLL). 
    }
	\label{fig:if_pvi}
\end{figure}
To demonstrate robustness to likelihood misspecification as in \cref{thm:robustness}, we consider the influence of a single outlier at one of seven clients on the server posterior. \cref{fig:if_pvi} demonstrates that the negative log likelihood is not robust in the federated setting, whereas different robust divergence based losses allow only limited influence of outliers on the posterior. We plot this as the divergence between the posterior, had we observed the outlier value at the true mean, against the posteriors that have the outlier be farther from the true mean, using the Fisher--Rao distance \citep{nielsen2023}.

\subsection{2D Misspecified Logistic Regression}
\begin{figure}[th]
	\centering
	\centerline{\includegraphics{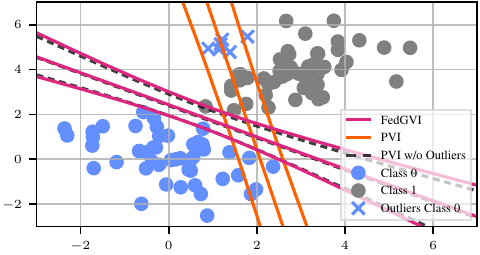}}
	\caption{Logistic Regression decision boundaries (0.2, 0.5, 0.8) for PVI without outliers, \textcolor{ibm1}{PVI with misspecification}, and \textcolor{ibm2}{\FedGVI with misspecification}.
    The synthetic data set is split homogeneously across 5 clients where PVI negatively skews the decision boundary, while \FedGVI does not. }
	\label{fig:synth_log_reg}
\end{figure}
We next consider a 2D logistic regression example where we generate 100 linearly separable  samples from a Gaussian mixture distribution. We inject outliers generated by a third Gaussian distribution and assign them to one of the classes so that the data is no longer linearly separable. 
We compare \FedGVI with $\mathcal{L}_\beta^{(0.7)}$ and $\divergence_{AR}^{(1.5)}$ against PVI, both with 5 clients. Again, the target is given by PVI only trained on the uncontaminated data. As expected PVI is severely impacted by outliers, whereas \FedGVI is robust to them and closely recovers the target posterior.

\subsection{Real-World Cover Type Dataset} 

\begin{figure}[h]
	\centering
	\input{{./figs/log_reg_fedgvi_competing.pgf}}
	\caption{Results on the  \textsc{Covertype} data set.
		We place a Gaussian distribution over the weights and average over 10 different train/test splits; see \cref{apx:additional_exp_details} for details.}
	\label{fig:logreg_1}
\end{figure}

In this experiment we follow the experimental setup of \citet{kassab2022} and average accuracy over 10 random 80/20 train-test splits, where
the training data is split homogeneously across 2 clients. We do not add any label contamination. The results are plotted in \cref{fig:logreg_1}. The non-robust methods all eventually achieve similar accuracy, however \FedGVI is able to outperform all competing methods, which we argue is due to \FedGVI putting less weight on data points that are less likely to belong to the class. 

\subsection{Bayesian Neural Networks on \MNIST and \FashionMNIST}

\begin{table}[h]
\caption{Classification accuracy (highest in bold) on uncontaminated test data after training on 10\% contaminated \MNIST data. We report the best performance across all server iterations.}
\label{tab:bnn}
\begin{center}
\begin{small}
\begin{sc}
\centering
\begin{tabular}{ccc}
\toprule
\multicolumn{1}{c}{\multirow{2}{*}{Model}} & \multicolumn{2}{c}{Accuracy + Std.}  \\
 \cmidrule(lr){2-3}
  & 10 Clients & 3 Clients\\
\midrule
\FedAvg    & 96.64$\pm$ 0.07& 96.34 $\pm$ 0.20\\
\FedPA & 94.25$\pm$ 0.39& 95.31$\pm$ 0.35\\
\betaPredBayes    & 94.90$\pm$ 0.08& 96.73$\pm$ 0.08\\
PVI    & 95.56$\pm$ 0.18& 96.68$\pm$ 0.07\\
\FedGVI $\divergence_{AR}$     & 96.36$\pm$ 0.09& 97.13 $\pm$ 0.13\\
\FedGVI $\trueloss_{GCE}$      & 97.06$\pm$ 0.03& 98.04 $\pm$ 0.07\\
\FedGVI $\divergence_{AR}$+$\trueloss_{GCE}$    & \textbf{97.50$\pm$ 0.07}& \textbf{98.13$\pm$ 0.08}\\
\midrule
VI (1 Client) & \multicolumn{2}{c}{(96.96$\pm$ 0.17)}\\
GVI (1 Client) & \multicolumn{2}{c}{(\textbf{98.13$\pm$ 0.07})}\\
\bottomrule
\end{tabular}
\end{sc}
\end{small}
\end{center}
\vskip -0.1in
\end{table}

\begin{figure}[t]
    \centering
    \graphicspath{{./figs_rebuttal/}}
    \input{{./figs_rebuttal/ablation_study.pgf}}
    \caption{An ablation study on the hyperparameters of \FedGVI with $\mathcal{L}_{GCE}^{(\delta)}$ and $D_{AR}^{(\alpha)}$. We plot the maximum results achieved as percentage errors on uncontaminated test data after training 5 clients on 10\% contaminated data.}
    \label{fig:ablation_study}
\end{figure}

\begin{table}
\caption{Classification accuracy (highest in bold) on uncontaminated test data after training on different amounts of contaminated \FashionMNIST data. For \FedGVI we have fixed $\alpha=2.5$ for the $\alpha-$\renyi divergence. Each Method has data split homogeneously across 3 Clients. We report the best performance during all server iterations.}
\label{tab:contamination}
\begin{center}
\scriptsize
\begin{sc}
\centering
\begin{tabular}{ccccc}%
\toprule
\multicolumn{1}{c}{\multirow{2}{*}{Model}} & \multicolumn{4}{c}{Contamination}  \\
 \cmidrule(lr){2-5}  
  & 0\% & 10\% & 20\% & 40\% \\
\midrule
\FedAvg & 85.7$\pm$0.5 & 79.0$\pm$1.9 & 71.2$\pm$1.5 & 49.0$\pm$6.5\\
\FedPA & 88.1$\pm$0.3 & 87.4$\pm$0.2 & 86.5$\pm$0.2 & 85.4$\pm$0.5\\
\betaPredBayes & 87.6$\pm$0.1 & 87.2$\pm$0.1 & 86.8$\pm$0.1 & 85.8$\pm$0.1\\
PVI    & 86.2$\pm$0.2 & 85.1$\pm$0.1 & 84.4$\pm$0.1 & 82.8$\pm$0.1\\
\FedGVI $\delta=0.0$     & 87.1$\pm$0.1 & 86.2$\pm$0.2 & 85.6$\pm$0.1 & 83.8$\pm$0.1\\
\FedGVI $\delta=0.4$ & {88.7$\pm$0.2} & \textbf{88.6$\pm$0.1} & 87.0$\pm$0.4 & 78.1$\pm$0.4\\
\FedGVI $\delta=0.5$ & \textbf{89.0$\pm$0.2} & {88.6$\pm$0.2} & \textbf{88.4$\pm$0.2} & 85.1$\pm$0.7\\
\FedGVI $\delta=0.8$      & 88.6$\pm$0.0 & 88.4$\pm$0.1 & 88.0$\pm$0.0 & \textbf{87.2$\pm$0.1}\\
\FedGVI $\delta=1.0$    & 88.1$\pm$0.1 & 87.8$\pm$0.1 & 87.5$\pm$0.2 & 86.0$\pm$0.3\\
\bottomrule
\end{tabular}
\end{sc}
\end{center}
\vskip -0.1in
\end{table}

\begin{figure*}[t]
\centering
\subfigure{\input{{./figs/bnn_results_val_acc_10_clients.pgf}}}\hfill
\subfigure{\input{{./figs/bnn_results_val_nll_10_clients.pgf}}}\\
\subfigure{\input{{./figs/bnn_results_val_acc_3_clients.pgf}}}\hfill
\subfigure{\input{{./figs/bnn_results_val_nll_3_clients.pgf}}}\\
\caption{Accuracy (\% Error) and Negative Log Likelihood (NLL) results when running fully connected BNNs, with a Mean--Field Gaussian distribution, on the MNIST data set with \FedGVI. The training data set is contaminated by 10\% random label flipping, fixed across all repetitions. We average over five runs with random, homogeneous client splits.}
\label{fig:bnn_results}
\end{figure*}

We create label contamination by adding noise to the training set while leaving the test set unchanged and evaluate performance in this. For \MNIST, we add 10\% of class dependent label noise, see \cref{fig:bnn_results} and \cref{tab:bnn}. 
We further carry out an ablation study on the hyperparameter selection in \FedGVI with the Alpha--\renyi divergence and the generalised cross entropy loss, see \cref{fig:ablation_study}. This demonstrates that \FedGVI performs well under a variety of different loss and divergence parameters. Note that $\alpha=1$ recovers the KL divergence, $\alpha=0$ the reverse KL divergence, i.e. $\divergence_{AR}^{(0)}(\approxDist:\priorDist)=\divergence_{R\kullbackleibler}(\approxDist:\priorDist)=\divergence_{\kullbackleibler}(\priorDist:\approxDist)$,
and that $\delta=0$ recovers the negative log--likelihood.

For \FashionMNIST, in \cref{tab:contamination}, we vary the amount of random label contamination, showcasing performance drops under different amounts of misspecification.
We use an MLP, for \FedGVI and PVI with 1 hidden layer of 200 neurons; for \FedAvg, \FedPA, and \betaPredBayes, two hidden layers with 100 neurons in each. Data is distributed homogeneously across clients, using 5 different, randomly chosen seeds. 
We demonstrate that under model misspecification, \FedGVI significantly outperforms competing FL methods.
Furthermore, \FedGVI incurs no additional computational complexity when compared to PVI. This is due to the KL and Alpha-\renyi divergences having closed form solutions between Multivariate Gaussians with complexity of $\mathcal{O}(1)$ in each other, and as we require $\mathcal{O}(1)$ additional, constant operations to get the GCE from the NLL.

We provide further experiments in \cref{apx:additional_exp_details} on the runtime of \FedGVI against PVI, learning rate selection, stability of posteriors under small perturbations in the robust loss parameters, and showing that using a single hidden layer NN for the competing methods would either negatively, or not significantly, affect their performance.

\section{Conclusions and Future Work}\label{sec:Conclusions}
We have introduced \FedGVI, a novel probabilistic approach to federated learning that is provably robust to model misspecification, and allows for faster, conjugate client updates. 
The theoretical analysis of \FedGVI demonstrates it's appealing properties; we easily recover existing methods as restricted cases, and characterise the convergence behaviour at fixed points of \FedGVI as solving a global GVI optimisation problem, extending existing theory. 
Our result on provable robustness to outliers through \FedGVI allows for closed form, conjugate posteriors that are computationally efficient, and robust to model misspecification. 
In deriving this, we have also shown that the cavity distribution is necessary as predictions would otherwise be overly confident and biased.
The robustness of \FedGVI was further demonstrated empirically on multiple synthetic and real--world data sets, showing outperformance of existing FL methods across model architectures and misspecification levels.

An interesting future direction is to extend FedGVI within personalised FL settings \citep{kotelevskii2022} and hierarchical Bayesian FL through latent variables \citep{kim2023} as well as through the use of a structured posterior approximation \citep{hassan2024}, in order to incorporate client level variations. Incorporating the hierarchical model structures and additional inductive biases from such settings, while maintaining conjugacy and favourable computational complexity, remain as open challenges.
In future work, we further aim to address the robust Bayesian nonparametric setting of FL through \FedGVI, as well as investigate other types of robustness, 
including to adversarial and Byzantine attacks, by for instance using a robust aggregator in \cref{eqn:server_update}, and addressing the open problem of provable robustness to prior misspecification in GVI.
\section*{Acknowledgements}
OH, TM and TD acknowledge support from a UKRI Turing AI Acceleration Fellowship [EP/V02678X/1] and a Turing Impact Award from the Alan Turing Institute. For the purpose of open access, the authors have applied a Creative Commons Attribution (CC-BY) license to any Author Accepted Manuscript version arising from this submission. The authors acknowledge the University of Warwick Research Technology Platform for assistance in the research described in this paper.

\section*{Impact Statement}

This paper presents work on robust federated learning, a framework that aims to not only advance the field of machine learning, but also to develop methods that ensure the privacy of data sources, whilst aiming to achieve optimal performance even under contamination of the data. This approach, however, may discard low probability, tail events that could represent minority groups. Hence, the trade off between robustness and inclusivity is a fundamental ethical challenge for decision makers.

\bibliographystyle{icml2025}

\newpage
\appendix
\startcontents[sections]
\onecolumn

\begin{center}
\section*{\centering Supplementary Material for: \\Federated Generalised Variational Inference: \\A Robust Probabilistic Federated Learning Framework}
\end{center}
The appendix is structured as follows:
\cref{apx:notation} summarises the notation used throughout the paper and in the proofs. 
In \cref{apx:proofs} we present complete proofs of all theorems, propositions and lemmas given in the main paper. 
\cref{apx:additional_details} clarifies the requirements of \cref{def:robust_loss}, the GBI learning rate, and places \FedGVI in the broader GVI literature. Lastly, \cref{apx:additional_exp_details} gives additional details about the implementation of \FedGVI and additional experiments.

\section{Notation}\label{apx:notation}
In this section, we give definitions of the symbols used throughout the paper and the appendix.

$P_0$\qquad The abstract and unknown probability measure, also called data generating process, acting on some abstract measurable space $(\Omega, \mathcal{F})$ which gives rise to the data

$\{x_i,\outdata_i\}_{\arbitraryindex=1}^n$\qquad Entire data set of all clients, also written as $\{x_{1}^n, {\outdata}_{1}^n\}$, for $x_i\in\dataspace$ and $y_i|x_i\in\outputspace$

$\{\x_\client,\y_\client\}_{\client=1}^\maxclient$\qquad The entire set of data points split across $\maxclient$ clients labelled $\client\in[\maxclient]:=\{1,2,...,\maxclient\}$

$\dataspace$\qquad The data space, which is assumed to have Polish topology

$\outputspace$\qquad The output space, which can be categorical such as in classification where $\outputspace=[C]$, or real valued as in regression $\outputspace=\R^C$, $C\in\mathbb{N}$  

$\parameterspace$\qquad In the parametric setting this is the parameter space $\bt\in\parameterspace$, assumed to admit Polish topology

$\PTheta$\qquad The space of probability measures over the measurable space $(\parameterspace, \mathcal{T})$. We refer to distributions in this space, where we mean distribution functions given rise to by measures in this space. Note that these need not be continuous, and could only be defined almost everywhere in $\bt$.

$\Q$\qquad A variational family of distributions such that $\Q\subset \PTheta$ and, in terms of distributions, $\Q=\{{\paramqapprox}\in\PTheta:\bk\in\K\}$, where $\K$ is a set of variational parameters

$\prior $\qquad The prior distribution, given rise to by the prior measure $\Pi$ on $(\parameterspace, \mathcal{T})$

${\classificationlosstime[\client]{(\ti)}}$\qquad The local loss of client $\client$, at iteration $\ti\in[\timax]$, on the local data set $\{\x_\client, \y_\client\}$, not necessarily the same across clients nor iterations, and associated with the parameters $\bt\in\parameterspace$

$\lossapproximation{\client}{(\ti)}$\qquad Local loss approximation of ${\classificationloss[\client]}$ and the impact of the data of client $\client$ on the posterior at the server

$\clientupdate{\client}{(\ti)}$\qquad Local update, \cref{eqn:local_update}, that represents the change in the approximate posteriors, and the de facto change in the local loss approximation. It has associated damping parameter $\dampingparam$.

${\lossapproximation{\server}{(\ti)}}$\qquad Global loss approximation of all clients aggregated at the server

$\q[\client]{(\ti)}$\qquad Local posterior computed through \cref{eqn:local_optim}

$\q[\server]{(\ti)}$\qquad Global approximate posterior after server--side optimisation step, \cref{eqn:server_optim}

$P(\trueloss,\divergence, \Q)$\qquad The Rule of Three \citep{jeremias2022} that defines a global GVI objective

$\divergence$\qquad Any statistical divergence $\divergence:\PTheta\times\PTheta\rightarrow\R_{\ge 0}$ \citep[for a detailed definition see][]{nielsen2020}; $\divergence_\server$ denotes the divergence at the server.

$\Eq$\qquad The expectation with respect to $\approxDist(\bt)$
\section{Proofs of Theorems, Propositions, and Lemmas}\label{apx:proofs}
Here, we provide the full proofs of the theorems stated in the paper. 
Throughout, we assume that all the losses, distributions and approximate losses, are measurable with respect to some dominating measure $\mu(d\bt)$. 
This can be the Lebesgue measure in finite dimensional spaces, or more generally the Haar measure. %
For infinite dimensional measure spaces, which are of interest in the study of Bayesian inverse problems and nonparametrics, we could assume $\mu(d\bt)$ to be a Gaussian measure as in \citet{pinski2015}.
\subsection{Equivalence Between The KL Divergence and Weighted KL Divergence}\label{apx:weighted}
First, we present a well known auxiliary lemma
that will be used throughout the proofs. 
It states that the weighted KL divergence is equivalent
to using a tempered or weighted likelihood in the optimisation procedure, and hence lead to equivalent inference problems \citep{jeremias2022, bissiri2016}. 
So without loss of generality, we can push the weighting term of the KL divergence inside the loss, by defining the loss to be $\trueloss=\wKLparam\cdot\trueloss$, which does not change the optimisation procedure.
We show this result for $f$--divergences, which we define as in \citet{ali1966} and \citet{amari2016}.
\begin{lemma}\label{lem:weighted}
	For $\wKLparam>0$ the posteriors computed by the weighted $f$--divergence, $\divergence=\frac{1}{\wKLparam}\divergence_f$ and loss $\trueloss$, and the posterior through the $f$--divergence $\divergence=\divergence_f$ and weighted loss $\wKLparam\cdot \trueloss$ are equivalent, i.e.,
	\begin{equation*}
		P(\trueloss, \frac{1}{\wKLparam}\divergence_f,\Q)=P(\wKLparam\cdot\trueloss,\divergence_f,\Q)
	\end{equation*}
\end{lemma}
\begin{proof}
	\begin{equation*}
		\begin{aligned}
			P(\trueloss, \frac{1}{\wKLparam}\divergence_f,\Q) &= \minQ\left\{ \Ebq[\classificationloss]+\frac1\wKLparam\divergence_f({\approxDist}:{\priorDist})\right\}\\
			&=\minQ\left\{ \Ebq[\classificationloss]+\frac1\wKLparam \Ebq[{f\left(\frac{\approxDist(\bt)}{\prior}\right)} ]\right\}\\
			&=\minQ\left\{ \frac1\wKLparam \Ebq[\wKLparam\cdot\classificationloss+f\left(\frac{\approxDist(\bt)}{\prior}\right)]\right\}\\
			&= \minQ\left\{\Ebq[\wKLparam\cdot\classificationloss+f\left(\frac{\approxDist(\bt)}{\prior}\right)]\right\} \\
			&=\minQ\left\{\Ebq[\wKLparam\cdot\classificationloss]+\divergence_f(\approxDist:\priorDist)\right\}:= P(\wKLparam\cdot\trueloss,\divergence_f,\Q)
		\end{aligned}
	\end{equation*}
\end{proof}
Therefore, when referring to the loss in the following we mean it to be the weighted loss so that we can utilise the weighted KL divergence. 
This easily recovers the \kullbackleibler--divergence for $f:u\mapsto-\log u$.

\subsection{\cref{lem:damping}: A Logarithmic Opinion Pool through Damping}\label{apx:damping}

\begin{proof}
Consider the server update at some iteration $\ti$, where we gather the client updates. Under the KL divergence, we then solve the server optimisation procedure as:
\begin{equation*}
    \q[\server]{(\ti)}=\minQ\left\{
    \Ebq[\lossapproximation{\server}{(\ti)}] + \kullbackleibler(\approxDist:\priorDist)
    \right\}=\minQ\left\{
    \Ebq[{\log\frac{\approxDist(\bt)}{\prior\exp\{-\lossapproximation{\server}{(\ti)}\}}}]
    \right\}
\end{equation*}
we know that this is minimised at:
\begin{gather*}
    \q[\server]{(\ti)}\propto \prior\exp\{-\lossapproximation{\server}{(\ti)}\}= \prior\exp\left\{-\lossapproximation{\server}{(\ti-1)}-{\sum_{\client=1}^\maxclient\clientupdate{\client}{(\ti)}}\right\}\\
    \propto\underbrace{\prior\exp\{-\lossapproximation{\server}{(\ti-1)}\}}_{\propto \q[\server]{(\ti-1)}}\exp\left\{-\sum_{\client=1}^\maxclient-\dampingparam_\client\log\frac{\q[\client]{(\ti)}}{\q[\server]{(\ti-1)}}
    \right\}
    \propto \q[\server]{(\ti-1)}\prod_{\client=1}^\maxclient\left(\frac{\q[\client]{(\ti)}}{\q[\server]{(\ti-1)}}\right)^{\dampingparam_\client}\\
    =\frac{\q[\server]{(\ti-1)}\prod_{\client=1}^\maxclient(\q[\client]{(\ti)})^{\dampingparam_\client}}{(\q[\server]{(\ti-1)})^{\sum_{\client=1}^\maxclient\dampingparam_\client}}
\end{gather*}
By assumption we have that $\sum_{\client=1}^\maxclient\dampingparam_\client=1$, therefore $(\q[\server]{(\ti-1)})^{\sum_{\client=1}^\maxclient\dampingparam_\client} = \q[\server]{(\ti-1)}$ and:
\begin{align*}
    \q[\server]{(\ti)}&\propto\prod_{\client=1}^\maxclient(\q[\client]{(\ti)})^{\dampingparam_\client}\\
    \q[\server]{(\ti)}&=\frac{\prod_{\client=1}^\maxclient(\q[\client]{(\ti)})^{\dampingparam_\client}}{\int_\parameterspace\prod_{\client=1}^\maxclient(\q[\client]{(\ti)})^{\dampingparam_\client}\,\mu(d\bt)}, \; \mu-a.e.
\end{align*}
This forms an externally Bayesian logarithmic opinion pool \citep{genest1984,genest1986}.
\end{proof}
\subsection{Proof of \cref{thm:fixed_points}}\label{apx:fixed}
The proof of \cref{thm:fixed_points} is adapted from that for Partitioned Variational Inference in \citet{ashman2022}.
We show the proof of \cref{thm:fixed_points} by comparing the derivatives with respect to the variational parameters of ${\paramqapprox}$ of the sum of local objectives with those of the global objective. 
This is motivated by the equivalence of a sum of local GVI objectives (from each client) with some added constants and the global GVI objective, demonstrated in \cref{apx:objec}. The main proof is in \cref{apx:proof_fixed_pt}.
\subsubsection{Recovering a Global GVI Objective From Local Objectives}\label{apx:objec}
First, we provide an analogue of \citet[Property 2]{ashman2022} which states that the sum of the local (client) \FedGVI objectives and some constant, which we find to be the negative log normalising constants of the cavity and the server distributions, equals the global GVI objective. We define the following: 
\begin{gather*}
	\q[\server]{(\ti)}=\frac{1}{\Z{\approxDist_{\server}^{(\ti)}}}\prior\exp\{-\sum_{\client=1}^\maxclient\lossapproximation{\client}{(\ti)}\}\\
	\cavity{\client}=\frac{1}{\Z{\cavityDist{\client}}}\prior\exp\{-\sum_{{\otherclient}\ne \client}\lossapproximation{{\otherclient}}{(\ti)}\}\propto\frac{\q[\server]{(\ti)}}{\exp\{-\lossapproximation{\client}{(\ti)}\}}\\
	\localobjective{\client}{\approxDist_\server^{(\ti)}}:=\Ebq[{\classificationloss[\client]}]+\wkl{\approxDist}{\cavityDist{\client}} \\
	\globalobjective{\approxDist_\server^{(\ti)}}:=\Ebq[\sum_{\client=1}^\maxclient{\classificationloss[\client]}]+\wkl{\approxDist}{\priorDist}
\end{gather*}
Then we can recover the global objective by summing over the local objectives and subtracting the log normalising constants of the cavity distributions and the current server posterior. 

\begin{align*}
	&\sum_{\client=1}^\maxclient\localobjective{\client}{\approxDist_\server^{(\ti)}}-\frac1\wKLparam(\log \Z{\approxDist_{\server}^{(\ti)}}+\textstyle{\sum_{\client=1}^\maxclient\log \Z{\cavityDist{\client}}}) \\
	&= \sum_{\client=1}^\maxclient \left(\Ebq[{\classificationloss[\client]}]+\frac1\wKLparam\kl{\cavityDist{\client}}\right)-\frac1\wKLparam(\log \Z{\approxDist_{\server}^{(\ti)}}+\textstyle{\sum_{\client=1}^\maxclient\log \Z{\cavityDist{\client}}}) \\
	&= \sum_{\client=1}^\maxclient \Ebq[{\classificationloss[\client]}]+\sum_{\client=1}^\maxclient\frac1\wKLparam\Ebq[\log\frac{\approxDist(\bt)}{\cavity{\client}}]-\frac1\wKLparam(\log \Z{\approxDist_{\server}^{(\ti)}}+\textstyle{\sum_{\client=1}^\maxclient\log \Z{\cavityDist{\client}}})\\
	&= \Ebq[ \sum_{\client=1}^\maxclient {\classificationloss[\client]}]+\frac1\wKLparam\Ebq[ \sum_{\client=1}^\maxclient\log\frac{\approxDist(\bt)}{\cavity{\client}}]-\frac1\wKLparam(\log \Z{\approxDist_{\server}^{(\ti)}}+\textstyle{\sum_{\client=1}^\maxclient\log \Z{\cavityDist{\client}}})\\
	&= \Ebq[ \sum_{\client=1}^\maxclient {\classificationloss[\client]}]+\frac1\wKLparam\Ebq[ \log{\prod_{\client=1}^\maxclient}\frac{\approxDist(\bt)\exp\{-\lossapproximation{\client}{(\ti)}\}}{\q[\server]{(\ti)}}]-\frac1\wKLparam(\log \Z{\approxDist_{\server}^{(\ti)}}+\textstyle{\sum_{\client=1}^\maxclient\log \Z{\cavityDist{\client}}})\\
	&= \Ebq[ \sum_{\client=1}^\maxclient {\classificationloss[\client]}]+\frac1\wKLparam\Ebq[ \log\frac{\approxDist(\bt)\exp\{-\textstyle{\sum_{\client=1}^\maxclient}\lossapproximation{\client}{(\ti)}\}}{\q[\server]{(\ti)}}]-\frac1\wKLparam\log \Z{\approxDist_{\server}^{(\ti)}}\\
	&= \Ebq[ \sum_{\client=1}^\maxclient {\classificationloss[\client]}]+\frac1\wKLparam\Ebq[ \log\frac{\approxDist(\bt)}{\prior/\Z{\approxDist_{\server}^{(\ti)}}}]-\frac1\wKLparam\log \Z{\approxDist_{\server}^{(\ti)}}= \globalobjective{\approxDist_\server^{(\ti)}}
\end{align*}

Hence, by using the weighted KL divergence at the clients optimisation step, can we recover a global GVI objective by summing over the local objectives and adding some constants independent of the variational parameters of interest in the optimisation problem. 
We note that the added logarithms of the normalising constants are independent of $\bk$, since these are fixed through the current posterior and cavity distribution and do not depend on the variational parameters.

\subsubsection{\cref{thm:fixed_points}: Fixed Points Recovers a Global Fixed Point}\label{apx:proof_fixed_pt}%
We denote a fixed point of the algorithm as $\paramq[\server]{*}$ such that for all $\client\in[\maxclient]$ we have $\paramq[\server]{*}\in\minQ\localobjective{\client}{\approxDist_\server^*}$, then we have the property that no update will change the posterior found. 
Recall:

{\textbf{\cref{thm:fixed_points}} \textit{
		Let  $\divergence=\weightedKL$ at the clients, local loss $\trueloss_\client$ and $\Q:=\{{\paramqapprox}:\bk\in\K\}\subset \PTheta$ as a variational family. Assume that \FedGVI finds a fixed point $\paramq[\server]{*}$, such that for all clients we have that $\paramq[\server]{*}\in\minQ\localobjective{\client}{\approxDist_\server^*}$. Then, it holds that  $\paramq[\server]{*}\in\minQ\globalobjective{\approxDist_\server^*}$.}}

\begin{proof}
	First we note that we consider only the KL divergence in this proof, which is equivalent to saying we modify the loss $\trueloss$ to be multiplied by $\wKLparam>0$, which results in the equivalent formulation, as shown in \citet{jeremias2022} where $P(\trueloss,\weightedKL,\Q)=P(\wKLparam\cdot\trueloss,\kullbackleibler,\Q)$, see also \cref{lem:weighted}.
	
	Note that the condition $\forall \client \in [\maxclient]$ we have that $\paramq[\server]{*}\in\minQ\localobjective{\client}{\approxDist_\server^*}$ is equivalent to requiring that $\clientupdate{\client}{*}=0$, since this means that the local loss approximations remain unchanged and hence $\lossapproximation{\server}{*}$ remains unchanged.
	This then implies that the posterior at the server will not change. 
	This is the same as saying that the client optimisation step has found the global solution and hence $\q[\client]{*}$ and $\q[\server]{*}$ will be the same which implies that $\clientupdate{\client}{*}=0$.

    In the following all integrals are assumed to be over the parameter space $\parameterspace$, even when we don't make it explicit.
    
	We can furthermore show that we can express the derivative of the local objective as a single integral under the weighted KL divergence.
	\begin{gather*}
		\jacobian \localobjective{\client}{\approxDist_\server^*} = \jacobian \left\{\Ebq[{\classificationloss[\client]}]+\kllarge{\frac{\paramq[\server]{*}}{\exp\{-\paramlossapproximation[*]{\client}\}\Z[*]{\approxDist_\server}}}\right\}\\
		=\jacobian \int {\paramqapprox}\log\frac{1}{\exp\{ -{\classificationloss[\client]}\}}+{\paramqapprox}\left(\log \frac{{\paramqapprox}\exp\{-\paramlossapproximation[*]{\client}\}}{\paramq[\server]{*}}+\log \Z[*]{\approxDist_\server}\right) \mu(d\bt)\\
		= \jacobian \int {\paramqapprox}\log\frac{{\paramqapprox}\exp\{-\paramlossapproximation[*]{\client}\}}{\paramq[\server]{*}\exp\{-{\classificationloss[\client]}\}}\,\mu(d\bt)+\cancelto{0}{\jacobian\log  \Z[*]{\approxDist_\server}\int {\paramqapprox}\,\mu(d\bt)}
	\end{gather*}
	Now we first show that the fixed point is an extremum of the global objective and then that it is a minimum. 
	We do this by first differentiating the local objective with respect to the variational parameters $\bk$ and then that the sum of the local derivatives evaluated at $\bk=\bk^*$ equal the derivative of the global objective. 
	\begin{equation*}
		\begin{aligned}
			\jacobian \localobjective{\client}{\approxDist_\server^*} &=\jacobian \int {\paramqapprox}\log\frac{{\paramqapprox}\exp\{-\paramlossapproximation[*]{\client}\}}{\paramq[\server]{*}\exp\{-{\classificationloss[\client]}\}}\,\mu(d\bt)\\
			&=\jacobian \int {\paramqapprox}({\classificationloss[\client]}-\paramlossapproximation[*]{\client})\,\mu(d\bt)+\jacobian \int {\paramqapprox}\log\frac{{\paramqapprox}}{\paramq[\server]{*}}\,\mu(d\bt) \\
			&=\jacobian \int {\paramqapprox}({\classificationloss[\client]}-\paramlossapproximation[*]{\client})\,\mu(d\bt) + \int (\jacobian{\paramqapprox})\log\frac{{\paramqapprox}}{\paramq[\server]{*}}\,\mu(d\bt) \\
			&\qquad +\int\cancelto{0}{\jacobian {\paramqapprox}\,\mu(d\bt)}
		\end{aligned}
	\end{equation*}
	where first line follows since we can compose the expectation and (weighted) KL divergence and the normalising constant of the cavity distribution is constant with respect to $\bk$. 
	The last line follows from the fact that $\frac{d}{d\indata}f(\indata)\log f(x)= f'(\indata)\log f(\indata) + f'(\indata)$ and that we can exchange the order of integration and differentiation. 
	We further note that at convergence, where $\bk=\bk^*$, that $\log\frac{{\paramqapprox}}{\paramq[\server]{*}}\big|_{\bk=\bk^*}=0$.
	Evaluating the expression above at $\bk=\bk^*$ then yields:
	\begin{equation*}
		\begin{aligned}
			\jacobian \localobjective{\client}{\approxDist_\server^*}\Big|_{\bk=\bk^*} = \jacobian \int {\paramqapprox}({\classificationloss[\client]}-\paramlossapproximation[*]{\client})\,\mu(d\bt)\Big|_{\bk=\bk^*}
		\end{aligned}
	\end{equation*}
	Summing over all these client objectives then yields the following expression:
	\begin{equation*}
		\begin{aligned}
			\sum_{\client=1}^\maxclient\jacobian\localobjective{\client}{\approxDist_\server^*}\Big|_{\bk=\bk^*} &=\sum_{\client=1}^\maxclient \jacobian \int {\paramqapprox}({\classificationloss[\client]}-\paramlossapproximation[*]{\client})\,\mu(d\bt)\Big|_{\bk=\bk^*} \\ 
			&=\jacobian \int {\paramqapprox}(\sum_{\client=1}^\maxclient {\classificationloss[\client]}-\sum_{\client=1}^\maxclient\paramlossapproximation[*]{\client})\,\mu(d\bt)\Big|_{\bk=\bk^*} \\
			&= \int (\jacobian{\paramqapprox})\log\frac{\paramq[\server]{*}}{\prior\exp\{\sum_{\client=1}^\maxclient {\classificationloss[\client]}\}}\,\mu(d\bt)\Big|_{\bk=\bk^*} \\
			&\qquad +  \cancelto{0}{\jacobian\int {\paramqapprox}\log \Z{\approxDist^*} \,\mu(d\bt)}
		\end{aligned}
	\end{equation*}
	To compare this with a global fixed point we differentiate the global objective at $\approxDist^*$, not yet assumed to be a minimiser of the global objective, with respect to the variational parameters.
	\begin{equation*}
		\begin{aligned}
			\jacobian \globalobjective{\approxDist_\server^*}&=\jacobian\int {\paramqapprox} \log\frac{{\paramqapprox}}{\prior\exp\{\sum_{\client=1}^\maxclient {\classificationloss[\client]}\}}\,\mu(d\bt) \\
			&=\int (\jacobian {\paramqapprox}) \log\frac{{\paramqapprox}}{\prior\exp\{\sum_{\client=1}^\maxclient {\classificationloss[\client]}\}}\,\mu(d\bt) + \cancelto{0}{\int \jacobian {\paramqapprox}\,\mu(d\bt)}
		\end{aligned}
	\end{equation*}
	Then,
	\begin{equation*}
		\jacobian \globalobjective{\approxDist_\server^*}\Big|_{\bk=\bk^*}=\int (\jacobian {\paramqapprox}) \log\frac{{\paramqapprox}}{\prior\exp\{\sum_{\client=1}^\maxclient {\classificationloss[\client]}\}}\,\mu(d\bt)\Big|_{\bk=\bk^*} = \sum_{\client=1}^\maxclient\jacobian\localobjective{\client}{\approxDist_\server^*}\Big|_{\bk=\bk^*}
	\end{equation*}
	And since $\paramq[\server]{*}$ is a fixed point of each client, we have that $\jacobian\localobjective{\client}{\approxDist_\server^*}\big|_{\bk=\bk^*}=0$. 
	Therefore,
	\begin{equation*}
		\sum_{\client=1}^\maxclient\jacobian\localobjective{\client}{\approxDist_\server^*}\Big|_{\bk=\bk^*}=0\qquad\implies\qquad  \jacobian \globalobjective{\approxDist_\server^*}\Big|_{\bk=\bk^*}=0
	\end{equation*}
	This means that $\paramq[\server]{*}$ is an extremum of \FedGVI, and further that it is also an extremum of GVI with $\divergence=\weightedKL$. 
	We now show that it is further a minimum of the global GVI objective. We consider the Hessian $\hessian$ and proceed like before.
	\begin{equation*}
		\begin{aligned}
			\hessian \localobjective{\client}{\approxDist_\server^*} &= \hessian \int {\paramqapprox}\log\frac{{\paramqapprox}\exp\{-\paramlossapproximation[*]{\client}\}}{\paramq[\server]{*}\exp\{-{\classificationloss[\client]}\}}\,\mu(d\bt)\\
			&={\hessian} \int {\paramqapprox}({\classificationloss[\client]}-\paramlossapproximation[*]{\client})\,\mu(d\bt)+{\hessian} \int {\paramqapprox}\log\frac{{\paramqapprox}}{\paramq[\server]{*}}\,\mu(d\bt) \\
			&={\hessian} \int {\paramqapprox}({\classificationloss[\client]}-\paramlossapproximation[*]{\client})\,\mu(d\bt) \\
			&\quad+\jacobian\left( \int (\jacobian {\paramqapprox})\log\frac{{\paramqapprox}}{\paramq[\server]{*}}\,\mu(d\bt) +\cancelto{0}{\int \jacobian \log {\paramqapprox}\,\mu}(d\bt)
			\right) \\
			&= {\hessian} \int {\paramqapprox}({\classificationloss[\client]}-\paramlossapproximation[*]{\client})\,\mu(d\bt) \\
			&\quad + \int (\hessian {\paramqapprox})\log\frac{{\paramqapprox}}{\paramq[\server]{*}}\,\mu(d\bt) + \int (\jacobian {\paramqapprox})(\jacobian\log{{\paramqapprox}})\,\mu(d\bt)
		\end{aligned}
	\end{equation*}
	\citet{ashman2022} point out that this last term can equivalently be expressed through it's transpose.
	\begin{gather*}
		\left(\int (\jacobian {\paramqapprox})(\jacobian\log{{\paramqapprox}})\,\mu(d\bt)\right)^\transpose \\= \jacobian \int {\paramqapprox}(\jacobian\log{{\paramqapprox}})\,\mu(d\bt) + \cancelto{0}{\int\hessian {\paramqapprox}\,\mu(d\bt)}\\
		= \jacobian \int {\paramqapprox}\frac{1}{{\paramqapprox}}\,\mu(d\bt)=\boldsymbol{0}
	\end{gather*}
	Evaluating this Hessian at $\bk=\bk^*$:
	\begin{gather*}
		{\hessian} \localobjective{\client}{\approxDist_\server^*}\Big|_{\bk=\bk^*}=\\
		{\hessian} \int {\paramqapprox}({\classificationloss[\client]}-\paramlossapproximation[*]{\client})\,\mu(d\bt)\Big|_{\bk=\bk^*} + \int \cancelto{0}{(\hessian {\paramqapprox})\log\frac{{\paramqapprox}}{\paramq[\server]{*}}\,\mu(d\bt)}\Big|_{\bk=\bk^*}
	\end{gather*}
	Therefore, when summing over the individual Hessians of the clients, we get:
	\begin{equation*}
		\begin{aligned}
			\sum_{\client=1}^\maxclient {\hessian} \localobjective{\client}{\approxDist_\server^*}\Big|_{\bk=\bk^*} &= \sum_{\client=1}^\maxclient{\hessian} \int {\paramqapprox}({\classificationloss[\client]}-\paramlossapproximation[*]{\client})\,\mu(d\bt)\Big|_{\bk=\bk^*}\\
			&= {\hessian} \int {\paramqapprox}(\sum_{\client=1}^\maxclient {\classificationloss[\client]}-\sum_{\client=1}^\maxclient\paramlossapproximation[*]{\client})\,\mu(d\bt)\Big|_{\bk=\bk^*}\\
			&={\hessian} \int {\paramqapprox}\log\frac{\paramq[\server]{*}}{\prior\exp\{\sum_{\client=1}^\maxclient {\classificationloss[\client]}\}}\,\mu(d\bt)\Big|_{\bk=\bk^*} \\
			&\qquad + \cancelto{0}{\hessian\int {\paramqapprox}\log \Z{\approxDist^*}\,\mu(d\bt)}\\
			&= \int ({\hessian} {\paramqapprox})\log\frac{\paramq[\server]{*}}{\prior\exp\{\sum_{\client=1}^\maxclient {\classificationloss[\client]}\}}\,\mu(d\bt)\Big|_{\bk=\bk^*}
		\end{aligned}
	\end{equation*}
	which is a sum of positive definite matrices, and therefore, the extremum at the fixed point is a minimum. 
	
	We now compare this with the Hessian of the global objective of GVI.
	\begin{equation*}
		\begin{aligned}
			\hessian \globalobjective{\approxDist_\server^*}&=\hessian \int {\paramqapprox}\log \frac{{\paramqapprox}}{\prior\exp\{-\sum_{\client=1}^\maxclient {\classificationloss[\client]}\}}\,\mu(d\bt) \\
			&=\jacobian\left(\int (\jacobian {\paramqapprox})\log \frac{{\paramqapprox}}{\prior\exp\{-\sum_{\client=1}^\maxclient {\classificationloss[\client]}\}}\,\mu(d\bt)\right. \\
			&\qquad+ \left.\cancelto{0}{\int (\jacobian \log {\paramqapprox}){\paramqapprox}\,\mu}(d\bt) \right) \\
			&= \int(\hessian {\paramqapprox})\log \frac{{\paramqapprox}}{\prior\exp\{-\sum_{\client=1}^\maxclient {\classificationloss[\client]}\}}\,\mu(d\bt) \\
			&\qquad +\cancelto{0}{\int (\jacobian {\paramqapprox})(\jacobian \log {\paramqapprox})\,\mu(d\bt)}
		\end{aligned}
	\end{equation*}
	Therefore, we can see that, evaluated at $\bk=\bk^*$, 
	\begin{equation*}
		\begin{aligned}
			\hessian \globalobjective{\approxDist_\server^*}\Big|_{\bk=\bk^*} &= \int(\hessian {\paramqapprox})\log \frac{{\paramqapprox}}{\prior\exp\{-\sum_{\client=1}^\maxclient {\classificationloss[\client]}\}}\,\mu(d\bt)\Big|_{\bk=\bk^*} \\ 
			&= \int(\hessian {\paramqapprox})\log \frac{\paramq[\server]{*}}{\prior\exp\{-\sum_{\client=1}^\maxclient {\classificationloss[\client]}\}}\,\mu(d\bt)\Big|_{\bk=\bk^*} \\&= \sum_{\client=1}^\maxclient {\hessian} \localobjective{\client}{\approxDist_\server^*}\Big|_{\bk=\bk^*}
		\end{aligned}
	\end{equation*}
	Hence, the Hessian of the global GVI objective is positive definite and therefore we have found a local minimum at $\paramq[\server]{*}$ through \FedGVI. %
\end{proof}

\subsection{Proof of \cref{col:GBI}}\label{apx:GBI}
By combining \cref{rem:pvi_rec} and \cref{thm:fixed_points}, we can show that, under infinite computational resources, specifically if we are able to optimise over the entire space of possible distribution parametrised by $\bt\in\parameterspace$, then we are able to recover the Generalised Bayesian Posterior of \citet{bissiri2016} in a distributed fashion by partitioning the input data and solving several smaller optimisation problems in parallel. 
This is achieved by using the weighted Kullback--Leibler divergence at the clients and the regular KL divergence at the server.

Under the assumption that the prior is not misspecified, we can perform distributed Bayesian updating with our framework, similar to the Bayesian Committee Machine \citep{tresp2000} where we combine local posterior distributions. 
We aim to recover the Generalised Bayesian Posterior \citep{bissiri2016}:
\begin{equation*}
	\conditionalq{GBI}{\bk}=\frac{\exp\{-\GBIparam \classificationloss\}\,\prior}{\int_\parameterspace\exp\{-\GBIparam \classificationloss\}\,\prior\mu(d\bt)}
\end{equation*}
where $\GBIparam$ is some parameter that controls the learning rate from the data. 

We will show that using $\wKLparam=\GBIparam$ at the clients will recover this GBI posterior after a single iteration of our algorithm, and further that the algorithm shows convergence for any subsequent iteration. 
We assume that $\Q=\PTheta$ and that $\conditionalq{GBI}{\y,\x}\in\Q$. 
Furthermore, for simplicity we assume that the loss function $\trueloss(\cdot)$ is the additive across clients and that the data set is partitioned such that there are no intersections.
\begin{proof}
	The $\maxclient$ clients have data sets $\{\x_\client,\y_\client\}_{\client=1}^\maxclient$ such that $\x_\otherclient\cap\x_j=\emptyset$ for all ${\otherclient}\ne j$ and we write $\cup_{\client=1}^\maxclient\x_\client=\x_1^\maxclient$ and $\cup_{\client=1}^\maxclient\y_\client=\y_1^\maxclient$ to symbolise the entire data set.
	
	Then we can rewrite the GBI posterior as:
	$$
	\conditionalq{GBI}{\y_1^\maxclient,\x_1^\maxclient}=\frac{\exp\{-\GBIparam\sum_{\client=1}^\maxclient {\sameclassificationloss[\client]}\}\,\prior}{\int_\parameterspace\exp\{-\GBIparam\sum_{\client=1}^\maxclient {\sameclassificationloss[\client]}\}\,\prior\mu(d\bt)}
	$$
	The \FedGVI approximation then takes the following form: $\q[\server]{(0)}=\prod_{\client=1}^{\maxclient}\exp\{-\lossapproximation{\client}{(0)}\}\prior/\Z{\approxDist_\server}$ and as we initiate $\lossapproximation{\client}{(0)}=0$ we have that $\q[\server]{(0)}=\prior$.
	
	Then in parallel, the each client $\client\in[\maxclient]$ carries out their optimisation step:
	
	The cavity distribution can be found through division as:
	$$
	\cavity{\client}\propto\frac{\q[\server]{(0)}}{\exp\{-\lossapproximation{\client}{(0)}\}}=\frac{\prior}{1}=\prior
	$$
	And the Generalised Variational Inference step with the cavity distribution as a local prior solves the following optimisation problem:
	\begin{align*}
		\q[\client]{(1)}&=\minQ\left\{\Ebq[{\sameclassificationloss[\client]}]+\frac1\GBIparam \kl{\priorDist}
		\right\}\\
		&\overset{(1)}{=}\minQ\Ebq[\log\frac{\approxDist(\bt)}{\prior\exp\{-\GBIparam {\sameclassificationloss[\client]}\}}]\\
		&\overset{(2)}{=}\prior\exp\{-\GBIparam {\sameclassificationloss[\client]}\}/\Z{\approxDist_\client}
	\end{align*}
	Where (1) follows through the equivalence between the weighted KL divergence and the tempered loss as discussed in \cref{apx:weighted}, and (2) follows due to the properties of a statistical divergence which is minimised when the inside of the expectation is zero and since $\Q=\PTheta$.
	
	This then implies that the update we send to the server is of the form:
	\begin{align*}
		\clientupdate{\client}{(1)}&=-\log\frac{\q[\client]{(1)}}{\q[\server]{(0)}}=-\log\frac{\prior\exp\{-\GBIparam {\sameclassificationloss[\client]}\}/\Z{\approxDist_\client}}{\prior}\\
		&=\GBIparam {\sameclassificationloss[\client]}+\log \Z{\approxDist_\client}
	\end{align*}
	At the server, we can combine these such that we get:
	\begin{align*}
		{\lossapproximation{\server}{(1)}}=\sum_{\client=1}^{\maxclient}\GBIparam {\sameclassificationloss[\client]}+\sum_{\client=1}^{\maxclient}\Z{\approxDist_\client} +\overbrace{\lossapproximation{\server}{(0)}}^{=0}=\GBIparam {\sameclassificationlossjoint{1}{\maxclient}}+\sum_{\client=1}^{\maxclient}\Z{\approxDist_\client}
	\end{align*}
	As GBI depends on the prior and hence trusts it, we use the KL divergence at the server, which is optimal with respect to the GBI posterior \citep{zellner1988,jeremias2022}. 
	Thus, the GVI objective at the server becomes:
	\begin{align*}
		\qtilde[\server]{(1)}&=\minQ\left\{\Ebq[{\lossapproximation{\server}{(1)}}]+\kl{\priorDist}\right\}\\
		&=\minQ\left\{\Ebq[\GBIparam {\sameclassificationlossjoint{1}{\maxclient}}+\sum_{\client=1}^{\maxclient}\Z{\approxDist_\client}]+\kl{\priorDist}\right\}\\
		&\overset{(3)}{=}\minQ\left\{\Ebq[\GBIparam {\sameclassificationlossjoint{1}{\maxclient}}]+\cancelto{0}{\sum_{\client=1}^{\maxclient}\Z{\approxDist_\client}}+\kl{\priorDist}\right\}\\
		&=\minQ\Ebq[\log\frac{\approxDist(\bt)}{\prior\exp\{-\GBIparam {\sameclassificationlossjoint{1}{\maxclient}}\}}]\\
		&\overset{(4)}{=}\prior\exp\{-\GBIparam {\sameclassificationlossjoint{1}{\maxclient}}\}/\Z{\tempq_\server^{(1)}}
	\end{align*}
	(3) follows since $\Z{\approxDist_\client}$ does not depend on $\bt$, nor the variational parameters, and hence does not affect our optimisation problem. 
	Line (4) is a result of $\Q=\PTheta$ and the assumption that the GBI posterior is contained within this set.
	
	This implies that the posterior that we find at the server is the Generalised Bayesian Inference posterior.
	\begin{equation*}
		\q[\server]{(1)}=\prior\exp\{-\GBIparam {\sameclassificationlossjoint{1}{\maxclient}}\}/\Z{\tempq_\server^{(1)}}
	\end{equation*}
	Thereby, we have shown that \FedGVI recovers the GBI posterior under the assumptions and that this occurs after the first iteration. 
	It remains to be shown that any further iteration steps will not change the posterior, and hence that we have recovered a fixed point as defined in \cref{thm:fixed_points}.
	
	We repeat the client optimisation steps in parallel. 
	We first find the cavity distribution:
	\begin{equation*}
		\cavity{\client}\propto\frac{\q[\server]{(1)}}{\exp\{-\GBIparam {\sameclassificationloss[\client]}\}}\propto 
		\frac{\prior\exp\{-\GBIparam \sum_{{\otherclient}=1}^\maxclient \sameclassificationloss[{\otherclient}]\}}{\exp\{-\GBIparam {\sameclassificationloss[\client]}\}}=\prior\exp\{-\GBIparam\sum_{{\otherclient}\ne \client}\sameclassificationloss[{\otherclient}]\}
	\end{equation*}
	Note that we ignore the normalising constant, since, similar to the server side optimisation step before, it does not depend on the variational parameters nor $\bt$.
	
	The optimisation step is then given through:
	\begin{align*}
		\q[\client]{(2)}&=\minQ\left\{\Ebq[{\sameclassificationloss[\client]}]+\frac1\GBIparam \kl{\cavityDist{\client}}\right\}\\
		&=\minQ \Ebq[\log\frac{\approxDist(\bt)}{\cavity{\client}\exp\{-\GBIparam {\sameclassificationloss[\client]}\}}]
	\end{align*}
	This statistical divergence is minimised at:
	\begin{align*}
		\q[\client]{(2)}&=\cavity{\client}  \exp\{-\GBIparam {\sameclassificationloss[\client]}\}/\tempnormaliser\\
		&=\prior\exp\{-\GBIparam\sum_{{\otherclient}\ne \client}\sameclassificationloss[{\otherclient}]\} \exp\{-\GBIparam {\sameclassificationloss[\client]}\}/\Z{\approxDist_\client^{(2)}}\\
		&=\prior\exp\{-\GBIparam \sum_{\client=1}^\maxclient {\sameclassificationloss[\client]}\}/\Z{\approxDist_\client^{(2)}}
	\end{align*}
	where we note that $\Z{\approxDist_\client^{(2)}}=\Z{\tempq_\server^{(1)}}$ and we have recovered the GBI posterior we currently have as our server distribution. 
	As a result, $\clientupdate{\client}{(2)}=-(\log \q[\client]{(2)}-\log \q[\server]{(1)})=-\log 1 = 0$ for all $\client\in[\maxclient]$.
	
	This satisfies the conditions for \cref{thm:fixed_points} and hence we have achieved a fixed point, which will not change the server distribution, since:
	\begin{equation*}
		{\lossapproximation{\server}{(2)}}=\underbrace{\sum_{\client=1}^\maxclient\overbrace{\clientupdate{\client}{(2)}}^{=0}}_{=0}+{\lossapproximation{\server}{(1)}}={\lossapproximation{\server}{(1)}}=\GBIparam {\sameclassificationlossjoint{1}{\maxclient}}
	\end{equation*}
	which means that the server optimisation routine would not be different from the one during the previous iteration. 
	\begin{equation*}
		\qtilde[\server]{(2)}=\minQ\left\{\Ebq[{\lossapproximation{\server}{(2)}}]+\kl{\priorDist}\right\}=\minQ\left\{\Ebq[{\lossapproximation{\server}{(1)}}]+\kl{\priorDist}\right\}=\qtilde[\server]{(1)}
	\end{equation*}
	And thus $\q[\server]{(2)}=\q[\server]{(1)}=\conditionalq[*]{GBI}{\y_1^\maxclient,\x_1^\maxclient}$.

    For the moreover part, we define the damping parameter $\delta=\frac1\maxclient$, and show that $\q[\server]{(\ti)}\rightarrow \conditionalq{GBI}{\y_1^\maxclient,\x_1^\maxclient}$ as $\ti\rightarrow\infty$.
    As the data here is implicit, we simplify notation by denoting the losses of a client as $\trueloss_{\client}(\bt)$ and the GBI posterior as $\approxDist_{GBI}(\bt)$. 
    Furthermore, we assume that the GBI learning rate parameter $\GBIparam$ is implicitly included in each client's loss. 
    Then by the usual modes of convergence, we show that:
    \begin{equation*}
        \left|\q[\server]{(\ti)} - \approxDist_{GBI}(\bt) \right| \rightarrow 0 
    \end{equation*}
    Note that under KL divergences at the server and client, we will have that $\lossapproximation{\server}{(\ti)}=\sum_{\client=1}^\maxclient\lossapproximation{\client}{(\ti)}$ (see proof of \cref{rem:pvi_rec}).
    \begin{equation*}
        \left|\prior\exp\left\{-\sum_{\client}^\maxclient\lossapproximation{\client}{(\ti)}\right\} - \prior\exp\left\{-\sum_{\client}^\maxclient\trueloss_{\client}(\bt)\right\} \right|= \prior \left|\exp\left\{-\sum_{\client}^\maxclient\lossapproximation{\client}{(\ti)}\right\} - \exp\left\{-\sum_{\client}^\maxclient\trueloss_{\client}(\bt)\right\} \right|
    \end{equation*}
    This converges when the exponents are equal, hence it is sufficient to prove that $\forall\client\in[\maxclient]$ we have $\lossapproximation{\client}{(\ti)}\rightarrow \trueloss_{\client}(\bt)$.

    Since for all $\client\in[\maxclient]$, at each iteration $\ti$ we have that under the KL divergences:
    \begin{align*}
        \cavity{\client}&\propto\frac{\q[\server]{(\ti-1)}}{\exp\{-\lossapproximation{\client}{(\ti-1)}\}}=\frac{\prior\exp\{-\textstyle{\sum_{\client=1}^\maxclient}\lossapproximation{\client}{(\ti-1)}\}}{\exp\{-\lossapproximation{\client}{(\ti-1)}\}} ={\prior\exp\left\{-{\sum_{\otherclient\ne\client}}\lossapproximation{\otherclient}{(\ti-1)}\right\}}\\
        \q[\client]{(\ti)}&\propto\exp\{\trueloss_{\client}(\bt)\}\cavity{\client}\\
        \clientupdate{\client}{(\ti)}&=-\frac1\maxclient\log\frac{\exp\{\trueloss_{\client}(\bt)\}\prior\exp\left\{-{\sum_{\otherclient\ne\client}}\lossapproximation{\otherclient}{(\ti-1)}\right\}}{\prior\exp\{-\textstyle{\sum_{\client=1}^\maxclient}\lossapproximation{\client}{(\ti-1)}\}}=\frac1\maxclient\trueloss_{\client}(\bt)-\frac1\maxclient\lossapproximation{\client}{(\ti-1)}\\
        \lossapproximation{\client}{(\ti)}&=\lossapproximation{\client}{(\ti-1)}+\clientupdate{\client}{(\ti)}=\frac1\maxclient\trueloss_{\client}(\bt)+\frac{\maxclient
        -1}{\maxclient}\lossapproximation{\client}{(\ti-1)}
    \end{align*}
    By expansion of $\lossapproximation{\client}{(\ti-1)}$, by recursively applying the definition above, we get the following closed form expression:
    \begin{equation*}
        \lossapproximation{\client}{(\ti-1)}=\left(\left(\frac{\maxclient-1}{\maxclient}\right)\frac1\maxclient + \left(\frac{\maxclient-1}{\maxclient}\right)\left(\frac{\maxclient-1}{\maxclient}\right)\frac1\maxclient+ ...+ \left(\frac{\maxclient-1}{\maxclient}\right)^\ti\frac1\maxclient\lossapproximation{\client}{(0)}
        \right)\trueloss_\client(\bt)
    \end{equation*}
    written as a summation and recalling that $\lossapproximation{\client}{(0)}=0$ by definition, we can interpret this as the series:
    \begin{equation*}
        \lossapproximation{\client}{(\ti)}=\trueloss_\client(\bt)\sum_{i=0}^{\ti-1} \frac1\maxclient\left(\frac{\maxclient-1}{\maxclient}\right)^\ti
    \end{equation*}
    which is a geometric series. 
    And since $\frac{\maxclient-1}{\maxclient}\in(0,1)$ by elementary analysis this converges, as $\ti\rightarrow\infty$, to the limit
    \begin{equation*}
        \lim_{\ti\rightarrow\infty} \lossapproximation{\client}{(\ti)}=\trueloss_\client(\bt) \frac1\maxclient \maxclient=\trueloss_\client(\bt).
    \end{equation*}
    Therefore, as $\ti\rightarrow\infty$  $\q[\server]{(\ti)}\rightarrow\approxDist_{GBI}(\bt)$ 
 $\bt$ almost everywhere. 
 We can only guarantee almost everywhere pointwise convergence, since integral operators such as the KL divergence only guarantee equivalence up to null sets.
\end{proof}
Notably, the reason for using the cavity distribution instead of some other effective prior for the client optimisation step is that we want to recover the (generalised) Bayesian posterior eventually with our framework assuming that we can optimise over the entire space of probability measures that characterise their respective probability distributions. 
We further assume that we can find a global minimiser of any optimisation problem. 
Then, under these assumptions, we would like to not change the current posterior any further after recovering the GBI posterior. 

We have previously shown that our algorithm achieves just this, and we can furthermore show that the cavity distribution is indeed the only choice in the client update that causes this.
\subsection{Proof of \cref{thm:cavity}}\label{apx:cavity}
We are interested in verifying whether the cavity distribution is necessary in \cref{eqn:local_optim}.
It acts to regularise the optimisation problem at the client, which we restate here, using some arbitrary probability density $\rho\in\PTheta$:
\begin{equation*}
\q[\client]{(\ti)}=\minQ\left\{\Ebq[{\classificationlosstime[\client]{(\ti)}}]+\D{\approxDist}{\rho}\right\}
\end{equation*}
where it is regularised by $\D{\cdot}{\rho}$. 
It is clear that this should not be the prior distribution after the server has additional information about client data available since we would not be doing anything different for subsequent updates and this would result in a Bayesian Committee Machine where each client does not learn from the others.
Therefore it is imperative to ask what this `effective prior' $\rho$ should be?
And in fact it turns out that it needs to be the cavity distribution.

We will approach this problem by considering the case where we know what we would want to target in the optimization problem and hence the sequence $\{\q[\server]{(\ti)}\}_{\ti\in\mathbb{N}_0}$ should converge to. 
We, however, have to restrict ourselves to the Federated Learning scenario and therefore any distribution that we come up with needs to satisfy the \cref{axm:data,axm:update}. For this we require the following assumption so that we are able to target the GBI posterior.
\begin{assumption}\label{asp:minimiser}
We are able to find global minimisers over the entire space of probability distributions parametrised by $\bt$, $\PTheta$.
\end{assumption}
Then it turns out that this regularising distribution is uniquely described by \cref{thm:cavity}, which we restate here.

\textbf{\cref{thm:cavity}}\textit{
		Let the assumptions be as in \cref{col:GBI}, i.e. $\Q=\PTheta$, $\divergence=\frac{1}{\GBIparam}\divergence_{\kullbackleibler}$ for $\GBIparam>0$, $\divergence_\server=\divergence_{\kullbackleibler}$, $\trueloss_\client^{(\ti)}=\trueloss$, and $\dampingparam_\client=1$, and further assume that \cref{axm:data,axm:update} are satisfied, then the following are equivalent:
		\begin{enumerate}
			\item $\exists \ti \in[\timax]$ for which $\q[\server]{(\ti)}=\approxDist_{\mathrm{GBI}}(\bt)$ (a.e.) is invariant under further \FedGVI updates.
			\item The cavity distribution regularises the client optimisation problem.
        \end{enumerate}
}

\begin{proof}
	$(2 \implies 1)$ This is a direct consequence of \cref{col:GBI} and can easily be seen by iterating through the algorithm with the cavity distribution.
	
	$(1 \implies 2)$ Without loss of generality we consider the GBI posterior to be found after the first iteration. 
	We show that the unique way that satisfies the axioms and does not change the GBI posterior at the second iteration (or any further iterations) is uniquely achieved by the cavity distribution.
	By the statement we have
	\begin{equation*}
		\q[\server]{(2)}=\exp\{-\lossapproximation{\server}{(2)}\}\prior/\Z[(2)]{\server}=\exp\{-\lossapproximation{\server}{(1)}\}\prior/\Z[(1)]{\server}=\q[\server]{(1)}.
	\end{equation*}
	We now need to relate this to the client updates and hence the solutions of the client optimization problem.
	\begin{align}
		\q[\server]{(2)}=\q[\server]{(1)}&\iff \exp\{-\lossapproximation{\server}{(2)}\}/\Z[(2)]{\server}=\exp\{-\lossapproximation{\server}{(1)}\}/\Z[(1)]{\server}\nonumber\\
		&\iff \lossapproximation{\server}{(2)} +\log\Z[(2)]{\server}=\lossapproximation{\server}{(1)} +\log\Z[(1)]{\server}\nonumber\\
		&\iff \lossapproximation{\server}{(2)}=\lossapproximation{\server}{(1)}+C,\quad C\in\R\nonumber\\
		&\overset{\mathrm{def}}{\iff}  \sum_{\client=1}^{\maxclient}\clientupdate{\client}{(2)}+\lossapproximation{\server}{(1)}=\lossapproximation{\server}{(1)}+C\nonumber\\
		&\iff \sum_{\client=1}^{\maxclient}\clientupdate{\client}{(2)}=C\nonumber\\
		&\iff \sum_{\client=1}^{\maxclient}\log\frac{\q[\client]{(2)}}{\q[\server]{(1)}}=C\nonumber\\
		&\iff \prod_{\client=1}^\maxclient \q[\client]{(2)}=K\left(\q[\server]{(1)}\right)^\maxclient, \quad K=e^C\label{eqn:cavity_proof}
	\end{align}
	
	Now, for some transformation operator $\operatordist_\client:\PTheta\rightarrow\PTheta$ acting on the information available at the client from the server in the form of the current approximate posterior, which we denote as $\CavityOperator[\client]{\approxDist_\server^{(1)}}$, that satisfies the \cref{axm:data,axm:update}, we get the client optimisation problem $\forall \client$:
	\begin{align*}
		\q[\client]{(2)}&=\minS\left\{\Ebq[{\sameclassificationloss[\client]}]+\wkl[\GBIparam]{\approxDist}{\operatordist_\client[\approxDist_\server^{(1)}]}\right\}\\
		&= \minS \left\{ \frac1\GBIparam\Ebq[-\GBIparam\log\exp\{-{\sameclassificationloss[\client]}\}]+\frac1\GBIparam\Ebq[\log\frac{\approxDist(\bt)}{\CavityOperator[\client]{\approxDist_\server^{(1)}}}]\right\}\\
		&= \minS\left\{\frac1\GBIparam \Ebq[\log\frac{\approxDist(\bt)}{\exp\{-\GBIparam {\sameclassificationloss[\client]}\}\CavityOperator[\client]{\approxDist_\server^{(1)}}}]\right\}\\
		\implies \q[\client]{(2)}&=\exp\{-\GBIparam {\sameclassificationloss[\client]}\}\CavityOperator[\client]{\approxDist_\server^{(1)}}/\Z[(2)]{\client}%
	\end{align*}
	
	Substituting this into \cref{eqn:cavity_proof} and using the definition of $\q[\server]{(1)}$ we can derive a relation between the individual client approximations.
	\begin{gather*}
		\prod_{\client=1}^\maxclient \q[\client]{(2)}=K\left(\q[\server]{(1)}\right)^\maxclient\\
		\prod_{\client=1}^{\maxclient}\frac{\CavityOperator[\client]{\approxDist_\server^{(1)}}\exp\{-\GBIparam {\sameclassificationloss[\client]}\}}{\Z[(2)]{\client}}=K(\prior)^\maxclient\exp\left\{-\maxclient\GBIparam\sum_{\client=1}^\maxclient {\sameclassificationloss[\client]}\right\}/\left(\Z[(1)]{\server}\right)^\maxclient\\
		{\prod_{\client=1}^\maxclient}\CavityOperator[\client]{\approxDist_\server^{(1)}}/\Z[(2)]{\client}=K(\prior)^\maxclient\exp\left\{-(M-1)\GBIparam\sum_{\client=1}^\maxclient {\sameclassificationloss[\client]}\right\}/\left(\Z[(1)]{\server}\right)^\maxclient\\
		\implies \prod_{\client=1}^{\maxclient}\CavityOperator[\client]{\approxDist_\server^{(1)}}\propto \prod_{\client=1}^{\maxclient}\prior\exp\left\{-\GBIparam\sum_{{\otherclient}\ne \client}\sameclassificationloss[{\otherclient}]\right\}\propto\prod_{\client=1}^{\maxclient}\frac{\q[\server]{(1)}}{\exp\{-\GBIparam {\sameclassificationloss[\client]}\}}
	\end{gather*}
	Here, proportional '$\propto$' means equivalent up to some constant independent of $\bt$. 
	To see that the cavity distribution is in fact the only choice that satisfies the above equation, we need to recall the two axioms: 
	(\cref{axm:update}) $\CavityOperator[\client]{\approxDist_\server^{(1)}}$ needs to be generated in the same way across clients, and 
	(\cref{axm:data}) since we are in federated learning, each client will only be able to access it's own data. 
	This implies that we can write $\CavityOperator[\client]{\approxDist_\server^{(1)}}$ as a function of the current approximation and the client data, $\CavityOperator[\client]{\approxDist_\server^{(1)}}=\CavityOperator{\approxDist_\server^{(1)},\y_\client,\x_\client}$.
	\begin{equation*}
		\prod_{\client=1}^{\maxclient}\CavityOperator{\approxDist_\server^{(1)},\y_\client,\x_\client}\propto\prod_{\client=1}^{\maxclient}\frac{\q[\server]{(1)}}{\exp\{-\GBIparam {\sameclassificationloss[\client]}\}}
	\end{equation*}
	The only client that would have access to an explicit expression for the denominator would be client $\client$, to which the data $\{\x_\client,\y_\client\}$ belongs, and hence it must be entirely contained within that client's regularisation term $\operatordist_\client$.
	Therefore, we can conclude that $\q[\client]{(2)}=\q[\server]{(1)}$ and find a closed form for $\CavityOperator[\client]{\approxDist_\server^{(1)}}$. Note that this implies $C=0$ and hence $K=1$.
	\begin{gather*}
		\exp\{-\GBIparam {\sameclassificationloss[\client]}\}\CavityOperator[\client]{\approxDist_\server^{(1)}}/\Z{\approxDist_\client^{(2)}}=\exp\{-\sum_{\client=1}^\maxclient\GBIparam {\sameclassificationloss[\client]}\}\prior/\Z{\approxDist_\server^{(1)}}\\
		\CavityOperator[\client]{\approxDist_\server^{(1)}}\propto\frac{\exp\{-\sum_{\client=1}^\maxclient\GBIparam {\sameclassificationloss[\client]}\}\prior/\Z{\approxDist_\server^{(1)}}}{\exp\{-\GBIparam {\sameclassificationloss[\client]}\}/\Z{\approxDist_\client^{(2)}}}
	\end{gather*}
	This is exactly the cavity distribution as described in \cref{eqn:cavity}.
\end{proof}
This gives a justification for using the cavity distribution in our algorithm, since under the assumption that the prior is well specified, we would like to converge to the generalised Bayesian posterior distribution. 
Furthermore, we can note that this single step of \FedGVI recovers the principle of the Bayesian Committee Machine (BCM) of \citet{tresp2000} where we use generalised loss functions instead of the negative log likelihood in our formulation.
Furthermore, a single pass through \FedGVI---with the divergences as described above--- will recover a generalised version of the BCM irregardless of the space we optimise over.
\begin{remark}
	For the last two proofs we have assumed that we can find the global minimisers of the equations.
	This isn't strictly necessary to have since the use of the (weighted) Kullback--Leibler divergence allows us to formulate a closed form expression for what these will look like.
\end{remark}

\subsection{Proof of \cref{prop:conjugate}}\label{apx:conjugate}
This proposition is a direct result of Proposition 3.1 in \citet{altamirano2023} and the proof is analogous, we merely include it here for completeness. 
And while the stated result is in a regression setting, it can be be extended to the classification setting similar to \citet{altamirano2024} where Gaussian Processes are considered. 

We assume that each client has a data set $\{\x_\arbitraryindex\}_{\arbitraryindex=1}^{\clientdatasize}$ of size $\clientdatasize$. 
The divergence operator $\nabla \cdot f(\x)$ is defined in the usual way as the inner product between the vector of partial derivative operators and the vector of some vector valued function $f(\x)$ as $\nabla \cdot f(\x)=\langle(\partial/\partial x_1,...,(\partial/\partial x_d )^\transpose ,(f_1(\x), ...,f_d(\x))^\transpose\rangle$, and $\jacobian[\x]g(\x)$ is the Jacobian, the vector of partial derivatives of $g(\x)$. 
We further assume that $\dataspace\subseteq \R^\dimension$, and that $\expfamDGP\in\PTheta$.
\begin{proof} 
	The loss of some client $\client\in[\maxclient]$ at some arbitrary iteration $\ti\in[\timax]$ is given by
	\begin{equation*}
		\hat{\divergence}(\bt,\DGP_{n_\client}):=\frac1\clientdatasize\sum_{\arbitraryindex=1}^{\clientdatasize}\underbrace{||\weightfctindex[\client]{(\ti)}^\transpose\jacobian[\x]\log \expfamDGP[\arbitraryindex]||_2^2}_{(1)}+2\underbrace{\nabla\cdot (\weightfctindex[\client]{(\ti)}\weightfctindex[\client]{(\ti)}^\transpose\jacobian[\x]\log \expfamDGP[\arbitraryindex])}_{(2)}
	\end{equation*}
	where $\jacobian[\x]\log \expfamDGP[\arbitraryindex]= \jacobian[\x]\natparam(\bt)^\transpose\suffiecientstas(\x_\arbitraryindex)+\jacobian[\x]\refmeasure(\x_\arbitraryindex)$. We can then expand the terms in the above terms which we then give equal up to an additive constant independent of $\bt$.
	\begin{align*}
		(1) &= (\weightfctindex[\client]{(\ti)}^\transpose (\jacobian[\x]\suffiecientstas(\x_\arbitraryindex)^\transpose\natparam(\bt)+\jacobian[\x]\refmeasure(\x_\arbitraryindex)))^\transpose(\weightfctindex[\client]{(\ti)}^\transpose (\jacobian[\x]\suffiecientstas(\x_\arbitraryindex)^\transpose\natparam(\bt)+\jacobian[\x]\refmeasure(\x_\arbitraryindex)))\\
		&= (\weightfctindex[\client]{(\ti)}^\transpose\jacobian[\x]\suffiecientstas(\x_\arbitraryindex)^\transpose\natparam(\bt))^\transpose(\weightfctindex[\client]{(\ti)}^\transpose\jacobian[\x]\suffiecientstas(\x_\arbitraryindex)^\transpose\natparam(\bt)) + (\weightfctindex[\client]{(\ti)}^\transpose\jacobian[\x]\refmeasure(\x_\arbitraryindex))^\transpose(\weightfctindex[\client]{(\ti)}^\transpose\jacobian[\x]\refmeasure(\x_\arbitraryindex))  \\&+2 (\weightfctindex[\client]{(\ti)}^\transpose\jacobian[\x]\suffiecientstas(\x_\arbitraryindex)^\transpose\natparam(\bt))^\transpose((\weightfctindex[\client]{(\ti)}^\transpose\jacobian[\x]\refmeasure(\x_\arbitraryindex))) \\
		&\overset{+c}{=}\natparam(\bt)^\transpose\jacobian[\x]\suffiecientstas(\x_\arbitraryindex) \weightfctindex[\client]{(\ti)}\weightfctindex[\client]{(\ti)}^\transpose\jacobian[\x]\suffiecientstas(\x_\arbitraryindex)^\transpose\natparam(\bt)+\natparam(\bt)^\transpose\jacobian[\x]\suffiecientstas(\x_\arbitraryindex) \weightfctindex[\client]{(\ti)}\weightfctindex[\client]{(\ti)}^\transpose\jacobian[\x]\refmeasure(\x_\arbitraryindex)
	\end{align*}
	where the last line follows since the middle terms are independent of $\bt$ as long as the weight function is independent of $\bt$.
	\begin{align*}
		(2) &= \nabla\cdot (\weightfctindex[\client]{(\ti)}\weightfctindex[\client]{(\ti)}^\transpose\jacobian[\x]\natparam(\bt)^\transpose\suffiecientstas(\x_\arbitraryindex)) + \nabla\cdot (\weightfctindex[\client]{(\ti)}\weightfctindex[\client]{(\ti)}^\transpose\jacobian[\x]\refmeasure(\x_\arbitraryindex))\\
		&\overset{+c}{=} \natparam(\bt)^\transpose(\nabla\cdot(\weightfctindex[\client]{(\ti)}\weightfctindex[\client]{(\ti)}^\transpose\jacobian[\x]\suffiecientstas(\x_\arbitraryindex)))
	\end{align*}
	Then, this has the form $\hat{\divergence}(\bt,\DGP)\overset{+c}{=}\natparam(\bt)^\transpose\tempcovA{\client}{(\ti)}\natparam(\bt)+\natparam(\bt)^\transpose\tempmeanA{\client}{(\ti)}$, where
	\begin{gather*}
		\tempcovA{\client}{(\ti)}:=\frac1\clientdatasize\sum_{\arbitraryindex=1}^{\clientdatasize}\jacobian[\x]\suffiecientstas(\x_\arbitraryindex) \weightfctindex[\client]{(\ti)}\weightfctindex[\client]{(\ti)}^\transpose\jacobian[\x]\suffiecientstas(\x_\arbitraryindex)^\transpose\qquad\mathrm{and}\qquad \tempmeanA{\client}{(\ti)}:=\frac2\clientdatasize\sum_{\arbitraryindex=1}^{\clientdatasize}\nabla\cdot(\weightfctindex[\client]{(\ti)}\weightfctindex[\client]{(\ti)}^\transpose\jacobian[\x]\suffiecientstas(\x_\arbitraryindex)).
	\end{gather*}
	The first art follows by setting $\q[\client]{(\ti)}\propto\cavityDist{\client}_{(\ti)} (\bt)\exp\{-\GBIparam\clientdatasize(\natparam(\bt)^\transpose\tempcovA{\client}{(\ti)}\natparam(\bt)+\natparam(\bt)^\transpose\tempmeanA{\client}{(\ti)})\}$.
	Then, if $\natparam(\bt)=\bt$, and the local cavity distribution has the form $\cavityDist{\client}_{(\ti)} (\bt)\propto\exp\{-\frac12(\bt-\tempmeanB{\backslash\client}{(\ti)})^\transpose\tempcovB{\backslash\client}{(\ti)}^{-1}(\bt-\tempmeanB{\backslash\client}{(\ti)})\}$, then the local posterior is conjugate and is given by $\q[\client]{(\ti)}\propto\exp\{-\frac12(\bt-\tempmeanB{\client}{(\ti)})^\transpose\tempcovB{\client}{(\ti)}^{-1}(\bt-\tempmeanB{\client}{(\ti)})\}$, where
	\begin{align*}
		\q[\client]{(\ti)}&\propto\cavityDist{\client}_{(\ti)} (\bt)\exp\{-\GBIparam\clientdatasize(\natparam(\bt)^\transpose\tempcovA{\client}{(\ti)}\natparam(\bt)+\natparam(\bt)^\transpose\tempmeanA{\client}{(\ti)})\}\\
		&\propto\exp\{-\frac12(\bt-\tempmeanB{\backslash\client}{(\ti)})^\transpose\tempcovB{\backslash\client}{(\ti)}^{-1}(\bt-\tempmeanB{\backslash\client}{(\ti)})\}\exp\{-\GBIparam\clientdatasize(\bt^\transpose\tempcovA{\client}{(\ti)}\bt+\bt^\transpose\tempmeanA{\client}{(\ti)})\}\\
		&\propto\exp\{-\frac12[\bt^\transpose\tempcovB{\backslash\client}{(\ti)}^{-1}\bt -2\bt^\transpose\tempcovB{\backslash\client}{(\ti)}^{-1}\tempmeanB{\backslash\client}{(\ti)}+ 2\GBIparam\clientdatasize\bt^\transpose\tempcovA{\client}{(\ti)}\bt + 2 \GBIparam\clientdatasize\bt^\transpose\tempmeanA{\client}{(\ti)}]\}\\
		&\propto\exp\{-\frac12(\bt-\tempmeanB{\client}{(\ti)})^\transpose\tempcovB{\client}{(\ti)}^{-1}(\bt-\tempmeanB{\client}{(\ti)})\}
	\end{align*}
	where in the last line, we complete the square and get parameters
	\begin{gather*}
		\tempcovB{\client}{(\ti)}^{-1}:= \tempcovB{\backslash\client}{(\ti)}^{-1}+\GBIparam\clientdatasize\tempcovA{\client}{(\ti)}\qquad\mathrm{and}\qquad \tempmeanB{\client}{(\ti)}:= \tempcovB{\client}{(\ti)}(\tempcovB{\backslash\client}{(\ti)}^{-1}\tempmeanB{\backslash\client}{(\ti)}-\tempcovA{\client}{(\ti)}\tempmeanA{\client}{(\ti)}).
	\end{gather*}

    The moreover part can now easily be seen. 
    The update will be quadratic in $\bt$ and hence summing these results in a quadratic function, and since the posterior will have Gaussian distribution, the expectation with respect to the posterior of this quadratic function will have closed form. 
    Therefore, if the divergence at the server allows for closed form solutions between Multivariate Gaussians, then the entire \cref{eqn:server_optim} will have a closed form optimisation procedure that does not require sampling to approximate integrals.
\end{proof}
Note we have implicitly used the weighted KL divergence with parameter $\GBIparam\clientdatasize$. 
Note also that this does not immediately follow from \cref{col:GBI} since the weighting function is allowed to change depending on the iteration and the client. 
In our experiments, we for instance use the weighting function as measuring some deviation of a data point to the cavity mean. Furthermore, the weighting function does depend on the data point, but we suppress this dependence here to lighten notation.
\subsection{Proof of \cref{thm:robustness}}\label{apx:provablerobustness}
This result is more involved to prove where we show by induction that at each iteration, the posterior generated at the server is robust to outliers through the robustness of each client's loss function to outliers. 
To prove this result, we first consider what we mean by robustness and introduce some terminology. 
We consider the empirical data distribution of all clients $\DGP_\datasize=\frac1\datasize\sum_{\arbitraryindex}^{\datasize}\delta_{\x_\arbitraryindex}$ which is perturbed by some Huber contamination with parameter $\varepsilon$ at some adversarially chosen data point $z\in\dataspace$ as $\DGP_{\datasize,\varepsilon,z}:= (1-\varepsilon)\DGP_\datasize+\varepsilon\delta_z$, where the subscript $\datasize$ indicates how many data points are drawn from the distribution. 
Note that $\DGP_\datasize=\frac1\datasize\sum_{\client=1}^{\maxclient}\sum_{\arbitraryindex=1}^{\clientdatasize}\delta_{{\x_\client}_\arbitraryindex}=\frac1\datasize\sum_{\client=1}^{\maxclient}\clientdatasize\DGP_{\clientdatasize}$.
We then write $\semicolonq{\server}{(\ti)}{\DGP_{\datasize,\varepsilon,z}}$ to indicate a distribution with respect to data generated from the specified DGP. 
We first recall the notion of robustness introduced by \citet{ghosh2016a}. Note that we suppress the measure $\mu(d\bt)$ in the following and simply write $d\bt$. The posterior influence is given by:
\begin{equation*}
	\mathrm{PIF}(z,\bt,\DGP_\datasize):=\lim_{\varepsilon\downarrow0}\frac{\semicolonq{\server}{(\ti)}{\DGP_{\datasize,\varepsilon,z}}- \semicolonq{\server}{(\ti)}{\DGP_{\datasize}}}{\varepsilon}=\frac{d}{d\varepsilon}\semicolonq{\server}{(\ti)}{\DGP_{\datasize,\varepsilon,z}}|_{\varepsilon=0}
\end{equation*}
where the last line follows by L'Hopital's rule. \citet{ghosh2016a} further show, and one can easily check, that for $\semicolonq{\server}{(\ti)}{\DGP_{\datasize,\varepsilon,z}}=\pi(\bt)\exp\{-\semicolonapproxloss{\server}{(\ti)}{\DGP_{\datasize,\varepsilon,z}}\}/\int\pi(\bt)\exp\{-\semicolonapproxloss{\server}{(\ti)}{\DGP_{\datasize,\varepsilon,z}}\}d\bt$ this is equal to:
\begin{equation*}
	\mathrm{PIF}(z,\bt,\DGP_\datasize)=\semicolonq{\server}{(\ti)}{\DGP_{\datasize}}\left(-\frac{d}{d\varepsilon}\semicolonapproxloss{\server}{(\ti)}{\DGP_{\datasize,\varepsilon,z}}\big|_{\varepsilon=0}+\int_\parameterspace\frac{d}{d\varepsilon}\semicolonapproxloss{\server}{(\ti)}{\DGP_{\datasize,\varepsilon,z}}\big|_{\varepsilon=0}\Pi(d\bt)\right)
\end{equation*}
We call loss robust if it has finite posterior influence, i.e. $\sup_{\bt\in\parameterspace}\sup_{z\in\dataspace}|\mathrm{PIF}(z,\bt,\DGP_\datasize)|<\infty$.
To this end, we now state a Lemma due to \citet{matsubara2022}, adapted to our notation for \FedGVI which we have rephrased into \cref{def:robust_loss}.
\begin{lemma}[\citet{matsubara2022}]\label{lem:robustness_condition}
	Let $\semicolonq{{(\cdot)}}{(\ti)}{\DGP_{\datasize}}$ be a posterior computed at the server or the client with fixed $\datasize\in\mathbb{N}$ with loss $\semicolonapproxloss{{(\cdot)}}{(\ti)}{\DGP_\datasize}$ and a prior $\prior$. Suppose that $\semicolonapproxloss{{(\cdot)}}{(\ti)}{\DGP_\datasize}$ is lower bounded and that $\prior$ is upper bounded over $\bt\in\parameterspace$, for any $\DGP_\datasize$. Then if there exists some function $\auxfctname_{(\cdot)}^{(\ti)}:\Theta\rightarrow\R$ such that
	\begin{align*}
		1.&\quad \sup_{z\in\dataspace}\left|\frac{d}{d\varepsilon}\semicolonapproxloss{(\cdot)}{(\ti)}{\DGP_{\datasize,\varepsilon,z}}\big|_{\varepsilon=0}\right|\le\auxfct{{(\cdot)}}{(\ti)},\\
		2.&\quad \sup_{\bt\in\parameterspace}\prior\auxfct{{(\cdot)}}{(\ti)}<\infty, \;\mathrm{and}\\
		3.&\quad\int_\parameterspace\prior\auxfct{{(\cdot)}}{(\ti)}d\bt<\infty
	\end{align*}
	hold, then $\semicolonq{(\cdot)}{(\ti)}{\DGP_{\datasize}}$ is globally bias--robust.
\end{lemma}
We provide further clarification on these conditions in \cref{apx:def_notes}, and we are now able to give the proof of \cref{thm:robustness}.
\begin{proof}
	By the \cref{lem:robustness_condition}, we need to show that 
	\begin{equation*}
		\sup_{z\in\dataspace}\left|\frac{d}{d\varepsilon}\semicolonapproxloss{\server}{(\ti)}{\DGP_{\datasize,\varepsilon,z}}\big|_{\varepsilon=0}\right|\le\auxfct{\server}{(\ti)}
	\end{equation*}
	and that this $\auxfct{\server}{(\ti)}$ satisfies conditions (2.) and (3.) of the Lemma.
	Per assumption we know that the clients are robust to likelihood misspecification, so we need to relate the server loss to the client posterior influence functions. 
	To this end, we consider the loss at the server.
	\begin{equation*}
		\semicolonapproxloss{\server}{(\ti)}{\DGP_{\datasize,\varepsilon,z}} = \sum_{\client=1}^{\maxclient}\semicolonupdate{\client}{(\ti)}{\DGP_{\datasize,\varepsilon,z}} + \semicolonapproxloss{\server}{(\ti-1)}{\DGP_{\datasize,\varepsilon,z}}
	\end{equation*}
	where for each client, the update is given through \cref{eqn:local_update}
	\begin{align*}
		\semicolonupdate{\client}{(\ti)}{\DGP_{\datasize,\varepsilon,z}} &=-\log\frac{\semicolonq{\client}{(\ti)}{\DGP_{\datasize,\varepsilon,z}}}{\semicolonq{\server}{(\ti-1)}{\DGP_{\datasize,\varepsilon,z}}}
		=-\log\cfrac{\cfrac{\semicoloncavity{\client}{(\ti)}{\DGP_{\datasize,\varepsilon,z}}\exp\{-\GBIparam\clientdatasize\semicolonloss{\client}{(\ti)}{\DGP_{\clientdatasize,\varepsilon,z}}\}}{\int\semicoloncavity{\client}{(\ti)}{\DGP_{\datasize,\varepsilon,z}}\exp\{-\GBIparam\clientdatasize\semicolonloss{\client}{(\ti)}{\DGP_{\datasize,\varepsilon,z}}\}d\bt}}{\semicolonq{\server}{(\ti-1)}{\DGP_{\datasize,\varepsilon,z}}}\\
		&=-\log\cfrac{\cfrac{\cfrac{\cfrac{\semicolonq{\server}{(\ti-1)}{\DGP_{\datasize,\varepsilon,z}}}{\exp\{-\beta\clientdatasize\semicolonloss{\client}{(\ti-1)}{\DGP_{\datasize,\varepsilon,z}}\}}}{\int \cfrac{\semicolonq{\server}{(\ti-1)}{\DGP_{\datasize,\varepsilon,z}}}{\exp\{-\beta\clientdatasize\semicolonloss{\client}{(\ti-1)}{\DGP_{\datasize,\varepsilon,z}}\}}d\bt}\exp\{-\GBIparam\clientdatasize\semicolonloss{\client}{(\ti)}{\DGP_{\clientdatasize,\varepsilon,z}}\}}{\int\cfrac{\cfrac{\semicolonq{\server}{(\ti-1)}{\DGP_{\datasize,\varepsilon,z}}}{\exp\{-\beta\clientdatasize\semicolonloss{\client}{(\ti-1)}{\DGP_{\datasize,\varepsilon,z}}\}}}{\int \cfrac{\semicolonq{\server}{(\ti-1)}{\DGP_{\datasize,\varepsilon,z}}}{\exp\{-\beta\clientdatasize\semicolonloss{\client}{(\ti-1)}{\DGP_{\datasize,\varepsilon,z}}\}}d\bt}\exp\{-\GBIparam\clientdatasize\semicolonloss{\client}{(\ti)}{\DGP_{\clientdatasize,\varepsilon,z}}\}d\bt}}{\semicolonq{\server}{(\ti-1)}{\DGP_{\datasize,\varepsilon,z}}}\\
		&=-\log \cfrac{\cfrac{\exp\{-\beta\clientdatasize\semicolonloss{\client}{(\ti)}{\DGP_{\clientdatasize,\varepsilon,z}}\}}{\exp\{-\beta\clientdatasize\semicolonloss{\client}{(\ti-1)}{\DGP_{\datasize,\varepsilon,z}}\}}}{\int \cfrac{\semicolonq{\server}{(\ti-1)}{\DGP_{\datasize,\varepsilon,z}}}{\exp\{-\beta\clientdatasize\semicolonloss{\client}{(\ti-1)}{\DGP_{\datasize,\varepsilon,z}}\}}\exp\{-\beta\clientdatasize\semicolonloss{\client}{(\ti)}{\DGP_{\clientdatasize,\varepsilon,z}}\}d\bt}\\
		&= -\log \frac{\exp\{-\beta\clientdatasize\semicolonloss{\client}{(\ti)}{\DGP_{\clientdatasize,\varepsilon,z}}\}}{\exp\{-\beta\clientdatasize\semicolonloss{\client}{(\ti-1)}{\DGP_{\datasize,\varepsilon,z}}\}}+\log \Z[(\ti)]{\client}(\DGP_{\datasize,\varepsilon,z})
	\end{align*}
	Therefore,
	\begin{align*}
		\semicolonapproxloss{\server}{(\ti)}{\DGP_{\datasize,\varepsilon,z}}&=\sum_{\client=1}^{\maxclient}-\log \frac{\exp\{-\beta\clientdatasize\semicolonloss{\client}{(\ti)}{\DGP_{\clientdatasize,\varepsilon,z}}\}}{\exp\{-\beta\clientdatasize\semicolonloss{\client}{(\ti-1)}{\DGP_{\datasize,\varepsilon,z}}\}}+\log \Z[(\ti)]{\client}(\DGP_{\datasize,\varepsilon,z}) + \semicolonapproxloss{\server}{(\ti-1)}{\DGP_{\datasize,\varepsilon,z}}\\
		&=\sum_{\client=1}^{\maxclient}-\log {\exp\{-\beta\clientdatasize\semicolonloss{\client}{(\ti)}{\DGP_{\clientdatasize,\varepsilon,z}}\}}+\sum_{\client=1}^{\maxclient}\sum_{\arbitraryindex=1}^{\ti}\log \Z[(\arbitraryindex)]{\client}(\DGP_{\datasize,\varepsilon,z})\\
		&=\sum_{\client=1}^{\maxclient}\beta\clientdatasize\semicolonloss{\client}{(\ti)}{\DGP_{\clientdatasize,\varepsilon,z}}+\sum_{\client=1}^{\maxclient}\sum_{\arbitraryindex=1}^{\ti}\log \Z[(\arbitraryindex)]{\client}(\DGP_{\datasize,\varepsilon,z})
	\end{align*}
	We will now show by induction on $\ti$ that the posterior at the server is robust. 
	
	Concretely we will show that $\forall \ti\in[\timax]$, $\timax\in\mathbb{N}$, and $\maxclient\in\mathbb{N}$ finite, then
	\begin{equation*}
		\sup_{z\in\dataspace}\left|\frac{d}{d\varepsilon}\semicolonapproxloss{\server}{(\ti)}{\DGP_{\datasize,\varepsilon,z}}\big|_{\varepsilon=0}\right|\le \beta\sum_{\client=1}^{\maxclient}\clientdatasize\sup_{z\in\dataspace}\left|\frac{d}{d\varepsilon}\semicolonloss{\client}{(\ti)}{\DGP_{\clientdatasize,\varepsilon,z}}\big|_{\varepsilon=0}\right|+\sum_{\client=1}^{\maxclient}\sum_{\arbitraryindex=1}^{\ti}\sup_{z\in\dataspace}\left|\frac{d}{d\varepsilon}\log \Z[(\arbitraryindex)]{\client}(\DGP_{\datasize,\varepsilon,z})\big|_{\varepsilon=0}\right|\le\auxfct{\server}{(\ti)}
	\end{equation*}
	such that this function $\auxfct{\server}{(\ti)}$ satisfies the conditions of \cref{lem:robustness_condition}.
	Note that the first inequality follows by Minkowski's inequality.
	
	We begin by considering the case where $\ti=1$, then we have $\semicolonq{\server}{(1-1)}{\DGP_{\datasize,\varepsilon,z}}=\prior$ and $\semicolonloss{\client}{(1-1)}{\DGP_{\datasize,\varepsilon,z}}=0$ as initialised in the algorithm. 
	
	Consider the term $\frac{d}{d\varepsilon}\log \Z[(\arbitraryindex)]{\client}(\DGP_{\datasize,\varepsilon,z})\big|_{\varepsilon=0}$, then we have
	\begin{align*}
		\frac{d}{d\varepsilon}\log \Z[(1)]{\client}(\DGP_{\datasize,\varepsilon,z})\big|_{\varepsilon=0} &=\frac{	\frac{d}{d\varepsilon}\Z[(1)]{\client}(\DGP_{\datasize,\varepsilon,z})|_{\varepsilon=0}}{\Z[(1)]{\client}(\DGP_{\datasize,\varepsilon,z})|_{\varepsilon=0}}\\
		&=\cfrac{\int\cfrac{d}{d\varepsilon}\cfrac{\semicolonq{\server}{(1-1)}{\DGP_{\datasize,\varepsilon,z}}\exp\{-\GBIparam\clientdatasize\semicolonloss{\client}{(1)}{\DGP_{\clientdatasize,\varepsilon,z}}\}}{\exp\{-\GBIparam\clientdatasize\semicolonloss{\client}{(1-1)}{\DGP_{-\clientdatasize,\varepsilon,z}}\}}\Big|_{\varepsilon=0}d\bt}{\Z[(1)]{\client}(\DGP_{\clientdatasize})}\\
		&=\int\frac{\frac{d}{d\varepsilon}{\prior\exp\{-\GBIparam\clientdatasize\semicolonloss{\client}{(1)}{\DGP_{\clientdatasize,\varepsilon,z}}\}}\big|_{\varepsilon=0}d\bt}{\Z[(1)]{\client}(\DGP_{\clientdatasize})}\\
		&=-\int\left(\frac{d}{d\varepsilon}\GBIparam\clientdatasize\semicolonloss{\client}{(1)}{\DGP_{\clientdatasize,\varepsilon,z}}\big|_{\varepsilon=0}\right)\semicolonq{\client}{(1)}{\DGP_{\clientdatasize}}d\bt
	\end{align*}
	where the last equation follows since $\frac{d}{dx}\exp\{f(x)\}=\exp\{f(x)\}\frac{d}{dx}f(x)$.
	
	Consequently, using Jensen's inequality
	\begin{gather*}
		\sup_{z\in\dataspace}\left|\frac{d}{d\varepsilon}\semicolonapproxloss{\server}{(1)}{\DGP_{\datasize,\varepsilon,z}}\big|_{\varepsilon=0}\right|\\\le \beta\sum_{\client=1}^{\maxclient}\clientdatasize\sup_{z\in\dataspace}\left|\frac{d}{d\varepsilon}\semicolonloss{\client}{(1)}{\DGP_{\clientdatasize,\varepsilon,z}}\big|_{\varepsilon=0}\right|+\sum_{\client=1}^{\maxclient}\sup_{z\in\dataspace}\left|\int\left(\frac{d}{d\varepsilon}\GBIparam\clientdatasize\semicolonloss{\client}{(1)}{\DGP_{\clientdatasize,\varepsilon,z}}\big|_{\varepsilon=0}\right)\semicolonq{\client}{(1)}{\DGP_{\clientdatasize}}d\bt\right|\\
		\le \beta\sum_{\client=1}^{\maxclient}\clientdatasize\left(\sup_{z\in\dataspace}\left|\frac{d}{d\varepsilon}\semicolonloss{\client}{(1)}{\DGP_{\clientdatasize,\varepsilon,z}}\big|_{\varepsilon=0}\right|+\int\sup_{z\in\dataspace}\left|\frac{d}{d\varepsilon}\semicolonloss{\client}{(1)}{\DGP_{\clientdatasize,\varepsilon,z}}\big|_{\varepsilon=0}\right|\semicolonq{\client}{(1)}{\DGP_{\clientdatasize}}d\bt\right)
	\end{gather*}
	Then, if $\semicolonloss{\client}{(1)}{\DGP_{\clientdatasize,\varepsilon,z}}$ is robust, then there exists some function $\auxfct{\client}{(1)}:\parameterspace\rightarrow\R$ such that $\sup_{z\in\dataspace}\left|\frac{d}{d\varepsilon}\semicolonloss{\client}{(1)}{\DGP_{\clientdatasize,\varepsilon,z}}\big|_{\varepsilon=0}\right|\le \auxfct{\client}{(1)}$ and which satisfies:
	\begin{equation*}
		\sup_{\bt\in\parameterspace}\prior\auxfct{\client}{(1)}<\infty,\quad\mathrm{and}\quad\int\prior\auxfct{\client}{(1)}d\bt<\infty.
	\end{equation*}
	Substituting this into the above, we have that
	\begin{gather*}
		\sup_{z\in\dataspace}\left|\frac{d}{d\varepsilon}\semicolonapproxloss{\server}{(1)}{\DGP_{\datasize,\varepsilon,z}}\big|_{\varepsilon=0}\right|\le \beta\sum_{\client=1}^{\maxclient}\clientdatasize\left(\auxfct{\client}{(1)}+\int\auxfct{\client}{(1)}\semicolonq{\client}{(1)}{\DGP_{\clientdatasize}}d\bt\right)
	\end{gather*}
	Now recall that $\semicolonq{\client}{(1)}{\DGP_{\clientdatasize}}=\prior\exp\{-\GBIparam\clientdatasize\semicolonloss{\client}{(1)}{\DGP_{\clientdatasize}}\}/\Z[\client]{(1)}(\DGP_{\clientdatasize})$, and per the assumption we have that the loss is lower bounded and that $0<\Z[\client]{(1)}(\DGP_{\clientdatasize})<\infty$, therefore $\semicolonq{\client}{(1)}{\DGP_{\clientdatasize}}\le \prior\exp\{-\GBIparam\clientdatasize\inf_{\bt\in\parameterspace}\semicolonloss{\client}{(1)}{\DGP_{\clientdatasize}}\}/\Z[\client]{(1)}(\DGP_{\clientdatasize})\le C_\client^{(1)}\prior$ so that,
	\begin{equation*}
		\sup_{z\in\dataspace}\left|\frac{d}{d\varepsilon}\semicolonapproxloss{\server}{(1)}{\DGP_{\datasize,\varepsilon,z}}\big|_{\varepsilon=0}\right|\le \beta\sum_{\client=1}^{\maxclient}\clientdatasize\left(\auxfct{\client}{(1)}+C_\client^{(1)}\int\auxfct{\client}{(1)}\prior d\bt\right)=:\auxfct{\server}{(1)}
	\end{equation*}
	We now verify that the conditions hold. For condition 2, we have
	\begin{align*}
		\sup_{\bt\in\parameterspace}\prior\auxfct{\server}{(1)}&\le\beta\sum_{\client=1}^{\maxclient}\clientdatasize\left(\left(\sup_{\bt\in\parameterspace}\prior\auxfct{\client}{(1)}\right)+\left(\sup_{\bt\in\parameterspace}\prior\right)C_\client^{(1)}\int\auxfct{\client}{(1)}\prior d\bt\right)<\infty
	\end{align*}
	which follows by the assumptions on the robustness of the loss and that the prior is upper bounded, as well as the finiteness of $\GBIparam$, $\clientdatasize$, and $C_\client^{(1)}$.
	
	Condition 3 follows similar reasoning.
	\begin{align*}
		\int\auxfct{\server}{(1)}\prior d\bt &= \int \beta\sum_{\client=1}^{\maxclient}\clientdatasize\left(\auxfct{\client}{(1)}+C_\client^{(1)}\int\auxfct{\client}{(1)}\prior d\bt\right)\prior d\bt\\
		&= \beta\sum_{\client=1}^{\maxclient}\clientdatasize\left(\int \prior\auxfct{\client}{(1)}d\bt+\int \prior C_\client^{(1)}\left(\int\auxfct{\client}{(1)}\prior d\bt\right)d\bt\right)\\
		&= \beta\sum_{\client=1}^{\maxclient}\clientdatasize\left(\int \prior\auxfct{\client}{(1)}d\bt+ C_\client^{(1)}\int\auxfct{\client}{(1)}\prior d\bt\right)<\infty
	\end{align*}
	Since the loss is robust, the integrals are finite, and since all other terms are finite, we conclude that condition 3 is also satisfied.
	Therefore, for $\ti=1$ the posterior computed at the server satisfies the conditions of \cref{lem:robustness_condition} and is therefore globally bias--robust. 
	It remains to be shown that this holds for all $\ti\in\mathbb{N}$ such that $\ti\le\timax$, i.e. is finite.
	
	We now show by induction that if the posterior at the server is robust for $\ti=\otherclient$, then it will also be robust for $\ti=\otherclient+1$.
	\begin{align*}
		\frac{d}{d\varepsilon}\semicolonapproxloss{\server}{(\otherclient+1)}{\DGP_{\datasize,\varepsilon,z}}&=\GBIparam\sum_{\client=1}^\maxclient\clientdatasize\frac{d}{d\varepsilon}\semicolonloss{\client}{(\otherclient+1)}{\DGP_{\clientdatasize,\varepsilon,z}}\big|_{\varepsilon=0}+\sum_{\client=1}^\maxclient\sum_{\ti=1}^{\otherclient+1}\frac{d}{d\varepsilon}\log\Z[(\ti)]{\client}(\DGP_{\datasize,\varepsilon,z})\big|_{\varepsilon=0}\\
		&=\GBIparam\sum_{\client=1}^\maxclient\clientdatasize\frac{d}{d\varepsilon}\semicolonloss{\client}{(\otherclient+1)}{\DGP_{\clientdatasize,\varepsilon,z}}\big|_{\varepsilon=0}+\sum_{\client=1}^\maxclient\sum_{\ti=1}^{\otherclient+1}\underbrace{\frac{\frac{d}{d\varepsilon}\Z[(\ti)]{\client}(\DGP_{\datasize,\varepsilon,z})|_{\varepsilon=0}}{\Z[(\ti)]{\client}(\DGP_{\datasize,\varepsilon,z})|_{\varepsilon=0}}}_{(1)}
	\end{align*}
	To show the boundedness of this, we need to consider the expansion of (1) above.
	\begin{align*}
		\frac{d}{d\varepsilon}\log\Z[(\ti)]{\client}(\DGP_{\datasize,\varepsilon,z})\big|_{\varepsilon=0} =\frac{\frac{d}{d\varepsilon}\Z[(\ti)]{\client}(\DGP_{\datasize,\varepsilon,z})|_{\varepsilon=0}}{\Z[(\ti)]{\client}(\DGP_{\datasize,\varepsilon,z})|_{\varepsilon=0}}=\cfrac{\int \cfrac{	d}{d\varepsilon}\cfrac{\semicolonq{\server}{(\ti-1)}{\DGP_{\datasize,\varepsilon,z}}\exp\{-\GBIparam\clientdatasize\semicolonloss{\client}{(\ti)}{\DGP_{\clientdatasize,\varepsilon,z}}\}}{\exp\{-\GBIparam\clientdatasize\semicolonloss{\client}{(\ti-1)}{\DGP_{\clientdatasize,\varepsilon,z}}\}}\big|_{\varepsilon=0} d\bt}{\Z[(\ti)]{\client}(\DGP_{\datasize})}
	\end{align*}
	Now we consider the integral in the numerator. Using the chain rule when differentiating under the integral sign:
	\begin{align*}
		\int \frac{	d}{d\varepsilon}&\frac{\semicolonq{\server}{(\ti-1)}{\DGP_{\datasize,\varepsilon,z}}\exp\{-\GBIparam\clientdatasize\semicolonloss{\client}{(\ti)}{\DGP_{\clientdatasize,\varepsilon,z}}\}}{\exp\{-\GBIparam\clientdatasize\semicolonloss{\client}{(\ti-1)}{\DGP_{\clientdatasize,\varepsilon,z}}\}}\Big|_{\varepsilon=0} d\bt \\ 
		= \int & \left[ \frac{\exp\{-\GBIparam\clientdatasize\semicolonloss{\client}{(\ti)}{\DGP_{\clientdatasize}}\}}{\exp\{-\GBIparam\clientdatasize\semicolonloss{\client}{(\ti-1)}{\DGP_{\clientdatasize}}\}}\frac{d}{d\varepsilon}\semicolonq{\server}{(\ti-1)}{\DGP_{\datasize,\varepsilon,z}}\Big|_{\varepsilon=0}\right. \\&\left.+
		\frac{\semicolonq{\server}{(\ti-1)}{\DGP_{\datasize}}\exp\{-\GBIparam\clientdatasize\semicolonloss{\client}{(\ti)}{\DGP_{\clientdatasize}}\}}{\exp\{-\GBIparam\clientdatasize\semicolonloss{\client}{(\ti-1)}{\DGP_{\clientdatasize}}\}}
		\frac{d}{d\varepsilon}(-\GBIparam\clientdatasize\semicolonloss{\client}{(\ti)}{\DGP_{\clientdatasize,\varepsilon,z}})\Big|_{\varepsilon=0}\right.
		\\ &\left.-
		\frac{\semicolonq{\server}{(\ti-1)}{\DGP_{\datasize}}\exp\{-\GBIparam\clientdatasize\semicolonloss{\client}{(\ti)}{\DGP_{\clientdatasize}}\}}{\exp\{-\GBIparam\clientdatasize\semicolonloss{\client}{(\ti-1)}{\DGP_{\clientdatasize}}\}}
		\frac{d}{d\varepsilon}(-\GBIparam\clientdatasize\semicolonloss{\client}{(\ti-1)}{\DGP_{\clientdatasize,\varepsilon,z}})\Big|_{\varepsilon=0} \right]d\bt
	\end{align*}
	Bringing the denominator back, and recalling the definition of $\semicolonq{\client}{(\ti)}{\DGP_{\clientdatasize}}$, then we can simplify.
	\begin{align*}
		\frac{d}{d\varepsilon}\log\Z[(\ti)]{\client}(\DGP_{\datasize,\varepsilon,z})\big|_{\varepsilon=0} = \int & \left[ \cfrac{\left(\cfrac{\exp\{-\GBIparam\clientdatasize\semicolonloss{\client}{(\ti)}{\DGP_{\clientdatasize}}\}}{\exp\{-\GBIparam\clientdatasize\semicolonloss{\client}{(\ti-1)}{\DGP_{\clientdatasize}}\}}\right)}{\Z[(\ti)]{\client}(\DGP_{\datasize})}\frac{d}{d\varepsilon}\semicolonq{\server}{(\ti-1)}{\DGP_{\datasize,\varepsilon,z}}\Big|_{\varepsilon=0}\right.\\&\left.-
		\underbrace{\cfrac{\cfrac{\semicolonq{\server}{(\ti-1)}{\DGP_{\datasize}}\exp\{-\GBIparam\clientdatasize\semicolonloss{\client}{(\ti)}{\DGP_{\clientdatasize}}\}}{\exp\{-\GBIparam\clientdatasize\semicolonloss{\client}{(\ti-1)}{\DGP_{\clientdatasize}}\}}}{\Z[(\ti)]{\client}(\DGP_{\datasize})}}_{=\semicolonq{\client}{(\ti)}{\DGP_{\clientdatasize}}}
		\frac{d}{d\varepsilon}(\GBIparam\clientdatasize\semicolonloss{\client}{(\ti)}{\DGP_{\clientdatasize,\varepsilon,z}})\Big|_{\varepsilon=0}\right.
		\\ &\left. +
		\underbrace{\cfrac{\cfrac{\semicolonq{\server}{(\ti-1)}{\DGP_{\datasize}}\exp\{-\GBIparam\clientdatasize\semicolonloss{\client}{(\ti)}{\DGP_{\clientdatasize}}\}}{\exp\{-\GBIparam\clientdatasize\semicolonloss{\client}{(\ti-1)}{\DGP_{\clientdatasize}}\}}}{\Z[(\ti)]{\client}(\DGP_{\datasize})}}_{=\semicolonq{\client}{(\ti)}{\DGP_{\clientdatasize}}}
		\frac{d}{d\varepsilon}(\GBIparam\clientdatasize\semicolonloss{\client}{(\ti-1)}{\DGP_{\clientdatasize,\varepsilon,z}})\Big|_{\varepsilon=0}\right] d\bt\\
		= \int & \left[ \biggl({\frac{\exp\{-\GBIparam\clientdatasize\semicolonloss{\client}{(\ti)}{\DGP_{\clientdatasize}}\}}{{\Z[(\ti)]{\client}(\DGP_{\datasize})}\exp\{-\GBIparam\clientdatasize\semicolonloss{\client}{(\ti-1)}{\DGP_{\clientdatasize}}\}}}\frac{d}{d\varepsilon}\semicolonq{\server}{(\ti-1)}{\DGP_{\datasize,\varepsilon,z}}\Big|_{\varepsilon=0}\right.\\&\left.-
		\semicolonq{\client}{(\ti)}{\DGP_{\clientdatasize}}
		\frac{d}{d\varepsilon}(\GBIparam\clientdatasize\semicolonloss{\client}{(\ti)}{\DGP_{\clientdatasize,\varepsilon,z}})\Big|_{\varepsilon=0}\right.
		\\ &\left. +
		\semicolonq{\client}{(\ti)}{\DGP_{\clientdatasize}}
		\frac{d}{d\varepsilon}(\GBIparam\clientdatasize\semicolonloss{\client}{(\ti-1)}{\DGP_{\clientdatasize,\varepsilon,z}})\Big|_{\varepsilon=0} \right]d\bt
	\end{align*}
	Consider now the derivative of the previous server posterior with respect to $\varepsilon$ evaluated at 0, which we can write as:
	\begin{align*}
		\frac{d}{d\varepsilon}\semicolonq{\server}{(\ti-1)}{\DGP_{\datasize,\varepsilon,z}}\Big|_{\varepsilon=0} &= \frac{d}{d\varepsilon}\frac{\prior\exp\{-\semicolonapproxloss{\server}{(\ti-1)}{\DGP_{\datasize,\varepsilon,z}}\}}{\Z[(\ti-1)]{\server}(\DGP_{\datasize,\varepsilon,z})}\big|_{\varepsilon=0}\\ 
		&=\prior\left(\frac{\exp\{-\semicolonapproxloss{\server}{(\ti-1)}{\DGP_{\datasize}}\}}{\Z[(\ti-1)]{\server}(\DGP_{\datasize})}\frac{d}{d\varepsilon}(-\semicolonapproxloss{\server}{(\ti-1)}{\DGP_{\datasize,\varepsilon,z}})\Big|_{\varepsilon=0}\right.\\ &-\left.\frac{\exp\{-\semicolonapproxloss{\server}{(\ti-1)}{\DGP_{\datasize}}\}}{(\Z[(\ti-1)]{\server}(\DGP_{\datasize}))^2}\int\prior\frac{d}{d\varepsilon}\exp\{-\semicolonapproxloss{\server}{(\ti-1)}{\DGP_{\datasize,\varepsilon,z}}\}\Big|_{\varepsilon=0}d\bt\right)\\
		&= -\semicolonq{\server}{(\ti-1)}{\DGP_\datasize}\left(\frac{d}{d\varepsilon}\semicolonapproxloss{\server}{(\ti-1)}{\DGP_{\datasize,\varepsilon,z}}\Big|_{\varepsilon=0}-\int\semicolonq{\server}{(\ti-1)}{\DGP_\datasize}\frac{d}{d\varepsilon}\semicolonapproxloss{\server}{(\ti-1)}{\DGP_{\datasize,\varepsilon,z}}\Big|_{\varepsilon=0}d\bt\right)
	\end{align*}
	where we have used the definition of $\semicolonq{\server}{(\ti-1)}{\DGP_\datasize}$ by distributing the common terms outside the brackets and for the second term, since the normalising constant does not depend on $\bt$, we can take one of them inside the integral.
	Furthermore, using the fact that
	\begin{equation*}
		\cfrac{\cfrac{\semicolonq{\server}{(\ti-1)}{\DGP_{\datasize}}\exp\{-\GBIparam\clientdatasize\semicolonloss{\client}{(\ti)}{\DGP_{\clientdatasize}}\}}{\exp\{-\GBIparam\clientdatasize\semicolonloss{\client}{(\ti-1)}{\DGP_{\clientdatasize}}\}}}{\Z[(\ti)]{\client}(\DGP_{\datasize})}=\semicolonq{\client}{(\ti)}{\DGP_{\clientdatasize}}
	\end{equation*}
	then substituting the result for $\frac{d}{d\varepsilon}\semicolonq{\server}{(\ti-1)}{\DGP_{\datasize,\varepsilon,z}}\big|_{\varepsilon=0}$ into $\frac{d}{d\varepsilon}\log\Z[(\ti)]{\client}(\DGP_{\datasize,\varepsilon,z})\big|_{\varepsilon=0}$, we get that:
	\begin{align*}
		\frac{d}{d\varepsilon}\log\Z[(\ti)]{\client}(\DGP_{\datasize,\varepsilon,z})\big|_{\varepsilon=0} =\int \semicolonq{\client}{(\ti)}{\DGP_{\clientdatasize}}\biggl(
		-&\frac{d}{d\varepsilon}\semicolonapproxloss{\server}{(\ti-1)}{\DGP_{\datasize,\varepsilon,z}}\Big|_{\varepsilon=0}+\int\semicolonq{\server}{(\ti-1)}{\DGP_\datasize}\frac{d}{d\varepsilon}\semicolonapproxloss{\server}{(\ti-1)}{\DGP_{\datasize,\varepsilon,z}}\big|_{\varepsilon=0}d\bt \\
		-&\GBIparam\clientdatasize\frac{d}{d\varepsilon}\semicolonloss{\client}{(\ti)}{\DGP_{\clientdatasize,\varepsilon,z}}\big|_{\varepsilon=0}+\GBIparam\clientdatasize\frac{d}{d\varepsilon}\semicolonloss{\client}{(\ti-1)}{\DGP_{\clientdatasize,\varepsilon,z}}\big|_{\varepsilon=0}\biggr)d\bt
	\end{align*}
	Substituting this expression back into the original equation for $\frac{d}{d\varepsilon}\semicolonapproxloss{\server}{(\ti)}{\DGP_{\datasize,\varepsilon,z}}|_{\varepsilon=0}$, taking the supremum over $z\in\dataspace$ of the absolute value of this, and applying Minkowski's inequality, results in the following upper bound.
	\begin{flalign*}
		\sup_{z\in\dataspace}\Big|\frac{d}{d\varepsilon}&\semicolonapproxloss{\server}{(\otherclient+1)}{\DGP_{\datasize,\varepsilon,z}}\Big|&& \\
		\le& \sum_{\client=1}^\maxclient\GBIparam\clientdatasize\sup_{z\in\dataspace}\Big|\frac{d}{d\varepsilon}\semicolonloss{\client}{(\otherclient+1)}{\DGP_{\clientdatasize,\varepsilon,z}}\big|_{\varepsilon=0}\Big|+\sum_{\client=1}^{\maxclient}\sum_{\ti=1}^{\otherclient+1}\left\{\int \semicolonq{\client}{(\ti)}{\DGP_{\clientdatasize}}\left[\sup_{z\in\dataspace}\Big|\frac{d}{d\varepsilon}\semicolonapproxloss{\server}{(\ti-1)}{\DGP_{\datasize,\varepsilon,z}}\big|_{\varepsilon=0}\Big|\right.\right.&&\\
		&+\left.\left.\left(\int_\parameterspace \sup_{z\in\dataspace}\Big|\frac{d}{d\varepsilon}\semicolonapproxloss{\server}{(\ti-1)}{\DGP_{\datasize,\varepsilon,z}}\big|_{\varepsilon=0}\Big|\semicolonq{\server}{(\ti-1)}{\DGP_{\datasize}}d\bt\right)+\GBIparam\clientdatasize\sup_{z\in\dataspace}\Big|\frac{d}{d\varepsilon}\semicolonloss{\client}{(\ti)}{\DGP_{\clientdatasize,\varepsilon,z}}\big|_{\varepsilon=0}\Big|\right.\right.&& \\
		&+\left.\left.\GBIparam\clientdatasize\sup_{z\in\dataspace}\Big|\frac{d}{d\varepsilon}\semicolonloss{\client}{(\ti-1)}{\DGP_{\clientdatasize,\varepsilon,z}}\big|_{\varepsilon=0}\Big|\right]d\bt\right\}&&
	\end{flalign*}
	By the inductive assumption $\forall \ti\in [\otherclient+1]$, $\exists \auxfct{\server}{(\ti-1)}$ such that $\sup_{z\in\dataspace}\big|\frac{d}{d\varepsilon}\semicolonapproxloss{\server}{(\ti-1)}{\DGP_{\datasize,\varepsilon,z}}|_{\varepsilon=0}\big|\le\auxfct{\server}{(\ti-1)}$. 
	Additionally, as, by assumption, the loss is lower bounded and robust $\exists \auxfct{\client}{(\ti)}$ $\forall\ti\in[\otherclient+1]$ such that $\sup_{z\in\dataspace}\big|\semicolonloss{\client}{(\ti)}{\DGP_{\clientdatasize,\varepsilon,z}}|_{\varepsilon=0}\big|\le \auxfct{\client}{(\ti)}$. 
	Furthermore, these functions satisfy the conditions of \cref{lem:robustness_condition}.
	Note also that $\semicolonq{\server}{(\ti-1)}{\DGP_\datasize}\le C_\server^{(\ti-1)}\prior$, since the normalising constant of this distribution is finite and the loss is lower bounded per the inductive assumption, so we get $\semicolonq{\server}{(\ti-1)}{\DGP_\datasize}\le \prior \exp\{-\inf_{\bt\in\parameterspace}\semicolonapproxloss{\server}{(\ti-1)}{\DGP_{\datasize}}\}/\Z[(\ti-1)]{\server}(\DGP_{\datasize})\le C_\server^{(\ti-1)}\prior$, as seen in similar arguments before.
	Utilising this, we conclude:
	\begin{align*}
		\le \sum_{\client=1}^\maxclient\GBIparam\clientdatasize\auxfct{\client}{(\otherclient+1)} + \sum_{\client=1}^\maxclient\sum_{\ti=1}^{\otherclient+1}\int\semicolonq{\client}{(\ti)}{\DGP_{\clientdatasize}}\bigg\{&\auxfct{\server}{(\ti-1)} +\left(\int \auxfct{\server}{(\ti-1)}C_\server^{(\ti-1)}\prior d\bt\right) &&\\
		&+\GBIparam\clientdatasize\auxfct{\client}{(\ti)}+\GBIparam\clientdatasize\auxfct{\client}{(\ti-1)}\bigg\} d\bt := \auxfct{\server}{(\otherclient+1)}&&
	\end{align*}
	We now need to show that this satisfies conditions (2) and (3) of \cref{lem:robustness_condition}. Let's recall what these conditions state:
	\begin{align}
		(2) &= \sup_{\bt\in\parameterspace}\prior\auxfct{\server}{(\otherclient+1)}<\infty\label{eqn:cndt2}\\
		(3) &= \int \prior\auxfct{\server}{(\otherclient+1)}d\bt<\infty\label{eqn:cndt3}
	\end{align}
	We first verify that condition (2) holds.
	\begin{flalign*}
		\sup_{\bt\in\parameterspace}\prior&\auxfct{\server}{(\otherclient+1)} =  \sup_{\bt\in\parameterspace}\prior\left\{\sum_{\client=1}^\maxclient\GBIparam\clientdatasize\auxfct{\client}{(\otherclient+1)}\right.&&\\
		&+ \left. \sum_{\client=1}^\maxclient\sum_{\ti=1}^{\otherclient+1}\int\semicolonq{\client}{(\ti)}{\DGP_{\clientdatasize}}\bigg\{\auxfct{\server}{(\ti-1)} +C_\server^{(\ti-1)}\int \auxfct{\server}{(\ti-1)}\prior d\bt+\GBIparam\clientdatasize\auxfct{\client}{(\ti)}+\GBIparam\clientdatasize\auxfct{\client}{(\ti-1)}\bigg\} d\bt
		\right\}&&\\
		\le&\GBIparam\clientdatasize\sum_{\client=1}^\maxclient \sup_{\bt\in\parameterspace}\prior \auxfct{\client}{(\otherclient+1)} +\sup_{\bt\in\parameterspace}\prior\left\{\sum_{\client=1}^\maxclient\sum_{\ti=1}^{\otherclient+1}\int\semicolonq{\client}{(\ti)}{\DGP_{\clientdatasize}}\left[\auxfct{\server}{(\ti)}\right.\right.&& \\
		&+\left.\left. C_\server^{(\ti-1)}\int\auxfct{\server}{(\ti-1)}\prior d\bt +\GBIparam\clientdatasize\auxfct{\client}{(\ti)}+\GBIparam\clientdatasize\auxfct{\client}{(\ti-1)}
		\right]d\bt\right\}<\infty
	\end{flalign*}
	Since $\GBIparam,\clientdatasize,$ and $\maxclient$ are finite, and any finite linear combination of finite terms is finite, we can easily see that the first part is finite. This follows since $\auxfct{\client}{(\otherclient+1)}$ satisfies condition (2) of \cref{lem:robustness_condition}. Furthermore, since $\prior$ is upper bounded, we now need to verify whether the  inside of the curly brackets is finite. Since this is a finite sum, we need to verify if $\forall \client\in[\maxclient]$ and $\forall\ti\in[\otherclient+1]$, the following holds:
	\begin{equation*}
		\int\semicolonq{\client}{(\ti)}{\DGP_{\clientdatasize}}\left[\auxfct{\server}{(\ti)}+ C_\server^{(\ti-1)}\int\auxfct{\server}{(\ti-1)}\prior d\bt +\GBIparam\clientdatasize\auxfct{\client}{(\ti)}+\GBIparam\clientdatasize\auxfct{\client}{(\ti-1)}
		\right]d\bt<\infty
	\end{equation*}
	By the inductive step, this is true $\forall\ti\in[\otherclient]$, so we need to show that it also holds for $\ti=\otherclient+1$. So,
	\begin{equation*}
		\int\semicolonq{\client}{(\otherclient+1)}{\DGP_{\clientdatasize}}\left[\auxfct{\server}{(\otherclient+1)}+ C_\server^{(\otherclient)}\int\auxfct{\server}{(\otherclient)}\prior d\bt +\GBIparam\clientdatasize\auxfct{\client}{(\otherclient+1)}+\GBIparam\clientdatasize\auxfct{\client}{(\otherclient)}
		\right]d\bt<\infty
	\end{equation*}
	Note that $\semicolonq{\client}{(\otherclient+1)}{\DGP_{\clientdatasize}}$ is equal to $\prior\exp\{-\GBIparam\clientdatasize\semicolonloss{\client}{(\otherclient+1)}{\DGP_{\clientdatasize}}\}\exp\{-\GBIparam\sum_{\arbitraryindex\ne \client}\datasize_\arbitraryindex\semicolonloss{\arbitraryindex}{(\otherclient)}{\DGP_{\datasize_\arbitraryindex}}\}/\Z[(\otherclient+1)]{\client}\Z[(\otherclient)]{\server}$, and since the normalising constants are finite and positive, and the losses are lower bounded, then we can write
	\begin{align*}
		\semicolonq{\client}{(\otherclient+1)}{\DGP_{\clientdatasize}} &\le \prior\exp\{-\GBIparam\clientdatasize\inf_{\bt\in\parameterspace}\semicolonloss{\client}{(\otherclient+1)}{\DGP_{\clientdatasize}}\}\exp\{-\GBIparam\sum_{\arbitraryindex\ne \client}\datasize_\arbitraryindex\inf_{\bt\in\parameterspace}\semicolonloss{\arbitraryindex}{(\otherclient)}{\DGP_{\datasize_\arbitraryindex}}\}/\Z[(\otherclient+1)]{\client}\Z[(\otherclient)]{\server}\\
		&\le C_\client^{(\otherclient+1)}\prior
	\end{align*}
	where $0< C_\client^{(\otherclient+1)}< \infty$. Thereby, we have
	\begin{gather*}
		\int\semicolonq{\client}{(\otherclient+1)}{\DGP_{\clientdatasize}}\left[\auxfct{\server}{(\otherclient+1)}+ C_\server^{(\otherclient)}\int\auxfct{\server}{(\otherclient)}\prior d\bt +\GBIparam\clientdatasize\auxfct{\client}{(\otherclient+1)}+\GBIparam\clientdatasize\auxfct{\client}{(\otherclient)}
		\right]d\bt \\
		\le C_\client^{(\otherclient+1)}\int\prior\auxfct{\server}{(\otherclient+1)}d\bt+C_\client^{(\otherclient+1)} C_\server^{(\otherclient)}\left(\int\auxfct{\server}{(\otherclient)}\prior d\bt\right)\left(\int\prior d\bt\right)\\ +\GBIparam\clientdatasize C_\client^{(\otherclient+1)}\int\prior\auxfct{\client}{(\otherclient+1)}d\bt+\GBIparam\clientdatasize C_\client^{(\otherclient+1)}\int\prior\auxfct{\client}{(\otherclient)}d\bt<\infty
	\end{gather*}
	This expression is finite since the individual integrals must be finite by the definition of the bounding functions $\auxfctname$, as these need to satisfy condition (3) of \cref{lem:robustness_condition} with the prior $\prior$. 
	Hence, we have shown that condition (2) holds for $\auxfct{\server}{(\otherclient+1)}$ and \cref{eqn:cndt2} is indeed finite.
	
	It remains to be shown that condition (3), \cref{eqn:cndt3}, also holds. Using the same expression for $\auxfct{\server}{(\otherclient+1)}$ as before, we have:
	\begin{align*}
		\int &\prior\auxfct{\server}{(\otherclient+1)}d\bt = \int\prior\left\{\sum_{\client=1}^\maxclient\GBIparam\clientdatasize\auxfct{\client}{(\otherclient+1)}\right.&&\\
		&+ \left. \sum_{\client=1}^\maxclient\sum_{\ti=1}^{\otherclient+1}\int\semicolonq{\client}{(\ti)}{\DGP_{\clientdatasize}}\bigg\{\auxfct{\server}{(\ti-1)} +C_\server^{(\ti-1)}\int \auxfct{\server}{(\ti-1)}\prior d\bt+\GBIparam\clientdatasize\auxfct{\client}{(\ti)}+\GBIparam\clientdatasize\auxfct{\client}{(\ti-1)}\bigg\} d\bt
		\right\}d\bt&&
	\end{align*}
	Since, the summations are finite, we can exchange the integrals and sums to get
	\begin{align*}
		=\Biggl(\sum_{\client=1}^\maxclient\GBIparam\clientdatasize \underbrace{\int\prior\auxfct{\client}{(\otherclient+1)}d\bt}_{<\infty\;\forall\client\in[\maxclient]}\Biggr)
		+\underbrace{\int\prior d\bt}_{=1}\Biggl( \sum_{\client=1}^\maxclient\sum_{\ti=1}^{\otherclient+1}\int\semicolonq{\client}{(\ti)}{\DGP_{\clientdatasize}}\bigg\{&\auxfct{\server}{(\ti-1)} +C_\server^{(\ti-1)}\int \auxfct{\server}{(\ti-1)}\prior d\bt\\
		&+\GBIparam\clientdatasize\auxfct{\client}{(\ti)}+\GBIparam\clientdatasize\auxfct{\client}{(\ti-1)}\bigg\} d\bt
		\Biggr)&&
	\end{align*}
	where the first part is finite since for each $\auxfct{\client}{(\otherclient+1)}$, we have by definition that this expression is finite as it needs to satisfy condition (3). 
	Therefore, we need to show that the summation is finite.
	By the inductive step, this is true $\forall\ti\in[\otherclient]$, and we will now show that $\forall\client\in[\maxclient]$ it also is finite for $\ti=\otherclient+1$.
	\begin{align*}
		\int\semicolonq{\client}{(\otherclient+1)}{\DGP_{\clientdatasize}}\bigg\{\auxfct{\server}{(\otherclient)} +C_\server^{(\otherclient)}\int \auxfct{\server}{(\otherclient)}\prior d\bt+\GBIparam\clientdatasize\auxfct{\client}{(\otherclient+1)}+\GBIparam\clientdatasize\auxfct{\client}{(\otherclient)}\bigg\} d\bt
	\end{align*}
	Recall from before that $\semicolonq{\client}{(\otherclient+1)}{\DGP_{\clientdatasize}}\le C_\client^{(\otherclient+1)}\prior$ and hence, it is now immediate to see that by the same argument as in the proof of condition (2), this integral is finite.
	Therefore, condition (3) of \cref{lem:robustness_condition} also holds and \cref{eqn:cndt3} is true.
	
	We conclude that all conditions of \cref{lem:robustness_condition} are satisfied.
	
	Therefore, by induction, as long as we have a robust loss function \citep[in the sense of][]{ghosh2016a, matsubara2022} at the clients, then irregardless of the current iteration by using the weighted KL divergence at the clients and the KL divergence at the server, \FedGVI achieves global bias robustness to outliers.
\end{proof}
Note that when assuming that $\semicolonq{\client}{(\otherclient+1)}{\DGP_{\clientdatasize}}\le C_\client^{(\otherclient+1)}\prior$, or similarly at the server in the uncontaminated case, we have used that the normalising constants in the well specified case are finite. 
This is necessary to hold, since otherwise we will not have valid distributions, and furthermore we can always choose a prior distribution that is bounded above so this will always be finite.
However, this finiteness is not assumed for the normalising constants that are contaminated by the outliers, so the proof is needed to show boundedness of the posterior influence under contamination.

\section{Additional Details on \FedGVI}\label{apx:additional_details}
We present some additional details on \FedGVI in order to aid clarity and contextualise it in the broader literature. 

\subsection{A Note on the Learning Rate Parameter in GBI}
The $\beta$ parameter comes from the power/cold/tempered posteriors of e.g. \citet{gruenwald2012}, where the likelihood in Bayesian posteriors is raised to some power of $\beta>0$. This was originally done to add some robustness to the posterior, down--weighting observations if $\beta< 1$ and up weighting these for $\beta > 1$. Through a known result \citep{jeremias2022} which we highlight in \cref{lem:weighted} in the Appendix, this is equivalent to having a weighted Kullback--Leibler divergence, $\frac1\beta\mathrm{KL}$. This also allows us to define if we want to trust the prior more $\beta<1$ or less $\beta>1$, since up weighting the data means down weighting the prior and vice versa.
\subsection{A Note on \cref{def:robust_loss}}\label{apx:def_notes}
The three conditions combined allow us to say whether the client posterior (or simply the posterior in a global, 1 Client, GBI setting) derived from such a robust loss is provably robust to Huber contamination.

From Condition 1 we are able to bound an infinitesimal change in the loss with the contaminating data point $z$ by some auxiliary function $\gamma$, possibly infinite for some values of $\bt$.

Condition 2 states that the product function, $\gamma(\bt)\pi(\bt)$ has finite uniform norm. 
This ensures that this product under the worst case contamination and the worst parameter $\bt$, is finite and hence it cannot be made arbitrarily bad, which does not hold for the negative log likelihood in general. Alternatively, the prior decays to zero faster than the auxiliary function can diverge to infinity in $\bt$.%

Condition 3 further says that $\gamma(\bt)\pi(\bt)$
is finitely integrable, i.e. that this is in $L^1(\Theta,\mu)$. This, in effect bounds the normalising constant of the contaminated posterior and will ensure that this is finite.

Taking all these conditions together tells us that the product function $\pi(\bt)\gamma(\bt)$ is in $L^1(\Theta,\mu)\cap L^\infty(\Theta,\mu)$, and that it is in fact finite everywhere. These two conditions that are mutually independent so both Condition 2 and 3 need to hold. Equivalently, we require that $\gamma(\bt)$ is bounded and integrable with respect to the prior probability measure $\prior \mu(d\bt)=:\Pi(d\bt)$.

These conditions characterise the notion of robustness we use for \cref{thm:robustness}, with derivation in \cref{apx:provablerobustness}, by considering the worst choice for the contamination $z$ and the parameter $\bt$ with respect to small perturbations of the resulting posterior through $\varepsilon$. 
The influence of the contamination $z$ and parameter $\bt$ on the posterior is defined as $\frac d{d\varepsilon}q_m^{(t)}(\bt;\mathbb{P}_{n_m,\varepsilon, z})|_{\varepsilon=0}$, which is bounded through the conditions. Our result then implies that the posterior is \textit{`globally bias robust'}, i.e. robust to Huber contamination.

\subsection{\FedGVI in the Context of GVI and FL}
When viewing GVI/GBI as an optimisation problem on the space of probability distributions $\PTheta$, Bayesian inference, VI, hierarchical Bayes/VI, all target a single element of this space. 
These methods either target the standard Bayesian posterior explicitly, or the posterior within some variational family with closest Kullback--Leibler distance to the Bayesian one \citep{blei2017,walker2013}. 
Through GBI and GVI we are able to target different elements of a  subspace of $\PTheta$, then simply a single point; in that regard, these approaches `generalise' Bayes. In this paper, `generalised' is inherited from GVI and GBI. We should note that in the \FedGVI setting, GBI and GVI allow us to generalise PVI or \FedAvg to a broader subspace of possible posteriors.
\cref{fig:fedgvi_relation_large} displays \FedGVI in regards to the related GVI and GBI literature as well as the FL literature as in \cref{fig:fedgvi_relation_small}.

\begin{figure}[h]
\centering
\begin{tikzpicture}

\node[align=center] at (0,0) {\FedGVI\\
$L, D,\Q, \maxclient, D_\server$};
\node[align=center] at (-5.5,4.5) {GVI\\
$L, D, \Q,$\\$ \maxclient=1,D_\server=D_{KL}$};
\node[align=center] at (0,4.5) {GBI\\
$L, D_{KL}, \PTheta,$\\$ \maxclient=1,D_\server=D_{KL}$};

\node[align=center] at (5.5,4.5) {\textsc{Bayes}\\
$-\log\likelihoodDist_\bt, D_{KL}, \PTheta,$\\$ \maxclient=1,D_\server=D_{KL}$};

\node[align=center] at (5.5,0) {VI\\
$-\log\likelihoodDist_\bt, D_{KL}, \Q,$\\$ \maxclient=1,D_\server=D_{KL}$};
\node[align=center] at (5.5,-4.5) {PVI\\
$-\log\likelihoodDist_\bt, D_{KL}, \Q$,\\$M,D_\server=D_{KL}$};

\node[align=center] at (-5.5,0) {ERM\\
$L, D=0, \{\delta_\bt\}$, \\$\maxclient=1, D_\server=0$};
\node[align=center] at (-5.5,-4.5) {\FedAvg\\
$L, D=0, \{\delta_\bt\},$\\ $M,
D_\server=0$};

\draw[line width=0.75pt, ->] (-0.5,-.5) -- (-5,-3.8);
\draw[line width=0.75pt, ->] (0.5,-.5) -- (5,-3.8);
\draw[line width=0.75pt, ->] (-0.5,.5) -- (-5,3.8);
\draw[line width=0.75pt, ->] (0.5,.5) -- (5,3.8);
\draw[line width=0.75pt, ->] (1.1,0) -- (4,0);
\draw[line width=0.75pt, ->] (-1.2,0) -- (-4.3,0);
\draw[line width=0.75pt, ->] (0,0.5) -- (0,3.8);
\draw[line width=0.75pt, ->] (-5.5,-3.8) -- (-5.5,-.8);
\draw[line width=0.75pt, ->] (-5.5,3.8) -- (-5.5,.8);
\draw[line width=0.75pt, ->] (5.5,-3.8) -- (5.5,-.8);
\draw[line width=0.75pt, <-] (5.5,3.8) -- (5.5,.8);
\draw[line width=0.75pt, ->] (1.3,4.5) -- (3.7,4.5);
\draw[line width=0.75pt, ->] (-4.6,4.5) -- (-1.3,4.5);
\draw[line width=0.75pt, ->, dashed] (-4.5,3.8) -- (5,0.6);
\end{tikzpicture}
\caption{We illustrate the relationship of \FedGVI---characterised by the loss $L$, the client divergence $D$, the variational family $\Q$, the number of clients $M$, and the divergence at the server $D_\server$---to Generalised Variational/Bayesian Inference (GVI/GBI), Partitioned Variational Inference (PVI), Variational Inference (VI), Federated Averaging (\FedAvg), Empirical Risk Minimisation (ERM), and Bayes.}
\label{fig:fedgvi_relation_large}
\end{figure}

\section{Additional Details on Experiments}\label{apx:additional_exp_details}
For reproducibility we give additional details on the experiments that we have carried out to empirically support our contributions. 
Code to reproduce these can be found at: 

\begin{center}
\url{https://github.com/Terje-M/FedGVI}.
\end{center}

\subsection{Normal--Location Model}
We assume the following well known model for the Data Generating Process and prior, with some unspecified prior mean $\mu_\pi$, in order to allow for prior misspecification:
\begin{gather*}
\bt \sim \N(\mu_\pi,1^2):=\prior\\
x_{1:N}|\bt \overset{\mathrm{iid}}{\sim} \N(\bt, 1^2):=\likelihoodDist(x_i|\bt).
\end{gather*}
The true Data Generating Process under model misspecification through Huber contamination is given by:
\begin{gather*}
	\bt \sim \N(0,0.5^2)\\
	x_{1:N}|\bt \overset{\mathrm{iid}}{\sim} (1-\varepsilon)\N((\bt - 2), 1^2) + \varepsilon \N((\bt+3), 0.5^2)
\end{gather*}
where the second term represents some $\varepsilon$ noise fraction that is added to the data.
Our aim is to find the location of the first term in the above model, while modelling out the noise from the second term.

We consider PVI where the client optimisation is given by $\RoT[\client]{\nll{\cdot}{\bt}}{\kullbackleibler}{\N}$, and the server optimisation step by $\RoT[\server]{\lossapproximation{\server}{(\cdot)}}{\kullbackleibler}{\N}$.
Under the assumption of likelihood misspecification, we consider the following divergences and losses at the clients, while leaving the server optimisation step unchanged: The weighted Kullback--Leibler divergence $\weightedKL$, the Alpha--\renyi divergence $\divergence_{AR}^{(\alpha)}$, the Fisher--Rao divergence $\divergence_{FR}$, the score matching losses $\calL_{SM}^{(w)}$, the beta--divergence based loss $\calL_{B}^{(\beta)}$, and the gamma--divergence based loss $\calL_{G}^{(\gamma)}$. 
Expect for $\RoT[\client]{\calL_{SM}^{(w)}}{\weightedKL}{\N}$, which allows for conjugate updates by \cref{prop:conjugate}, we have to resort to optimisation.
This however does not require Monte--Carlo sampling since the divergence terms and the losses have closed forms under Gaussian distributions, see \citet{jeremias2022} for the remaining losses, \citet{pardo2006} for the KL and Alpha--\renyi divergences and \citet{nielsen2023} for the Fisher--Rao divergence.
For the optimisation, we use the Adam optimiser with a learning rate of $0.001$, leaving all other parameters at their default values.

\paragraph{Explicit Losses and Divergences used}
As mentioned in \cref{sec:method}, we employ a range of different loss functions and divergences throughout the experiments. The main one being the robust generalised cross entropy used in the real world experiments. 
For the synthetics, for instance in \cref{fig:if_pvi} we compare four different losses with two different implementations for the Score--Matching loss. 

For this example, where we only have one sequence of data points $x_{1:N}$ which are assumed to be independent, the losses are:
\begin{enumerate}
    \item The Negative Log Likelihood:
    \begin{equation*}
        \mathcal{L}_{NLL}(x_i,\likelihoodDist_\bt)=-\log \likelihoodDist_\bt(x_i)
    \end{equation*}
    \item The Density--Power Divergence based loss \citep{ghosh2016a, ghosh2016b}:
    \begin{equation*}
        \mathcal{L}_{B}^{(\beta)}(x_i, \likelihoodDist_\bt)= -\frac{1}{\beta}\likelihoodDist_\bt(x_i)^{\beta} + \frac{1}{1+\beta}\int_\dataspace \likelihoodDist_\bt(x)^{\beta+1}\,\mu(dx)
    \end{equation*}
    \item The Gamma divergence based loss \citep{hung2018}:
    \begin{equation*}
        \mathcal{L}_{G}^{(\gamma)}(x_i, \likelihoodDist_\bt)= -\frac{1}{(\gamma-1)}\likelihoodDist_\bt(x_i)^{\gamma-1} \cdot\frac{\gamma}{\left(\int_\dataspace \likelihoodDist_\bt(x)^{\gamma}\,\mu(dx)\right)^{\frac{\gamma-1}{\gamma}}}
    \end{equation*}
    \item The weighted Score Matching Loss \citep{altamirano2023}:
    \begin{equation*}
        \mathcal{L}_{SM}^{(w_\client^{(\ti)})}(x_i, \likelihoodDist_\bt)=||w_m^{(t)}(x_i)^\transpose\nabla_x \log \likelihoodDist_\bt(x_i)||_2^2+ 2\nabla\cdot (w_\client^{(\ti)}(x_i){w_\client^{(\ti)}(x_i)}^\transpose\nabla_x\log\likelihoodDist_\bt(x_i))
    \end{equation*}
    We use two different weight functions $w_\client^{(\ti)}$, where $\mu_{\backslash\client}^{(\ti)}$ is the mean of the cavity distribution:
    \begin{enumerate}
        \item The Squared Exponential Kernel (SE):
        \begin{equation*}
            w_\client^{(\ti)}(x_i) = \GBIparam \exp\left\{-\frac{(x_i-\mu_{\backslash\client}^{(\ti)})^2}{2c^2}
            \right\}
        \end{equation*}
        \item The Inverse Multi-Quadratic Kernel (IMQ):
        \begin{equation*}
            w_\client^{(\ti)}(x_i) = \GBIparam \left(1 + \frac{(x_i-\mu_{\backslash \client}^{(\ti)})^2}{2a c^2}
            \right)^{-a}
        \end{equation*}
    \end{enumerate}
    See \cref{apx:conjugate} for details on what this posterior looks like using this particular loss.
\end{enumerate}
All the above losses have closed form objectives under the expectation with respect to the approximating distribution $\approxDist(\bt)$ and the assumed Gaussian likelihood. Furthermore, the negative log likelihood and the score matching loss admit conjugate updates under the KL divergence.

We also use different divergences, mainly:
\begin{enumerate}
    \item The Kullback--Leibler divergence \citep{kullback1951}:
    \begin{equation*}
        \kl{\priorDist}=\int_\parameterspace\approxDist(\bt)\log{\frac{\approxDist(\bt)}{\prior}} \,\mu(d\bt)
    \end{equation*}
    \item The Reverse KL divergence:
    \begin{equation*}
        \divergence_{RKL}(\approxDist:\priorDist)=\divergence_{\kullbackleibler}(\priorDist:\approxDist)
    \end{equation*}
    \item The Alpha--\renyi divergence:
    \begin{equation*}
        \divergence_{AR}^{(\alpha)}(\approxDist:\priorDist)=\frac{1}{\alpha(1-\alpha)}\log\int_\parameterspace \approxDist(\bt)^\alpha \prior ^{1-\alpha} \,\mu(d\bt)
    \end{equation*}
    \item Weighted Divergences of the form:
    \begin{equation*}
        \frac1\GBIparam \D{\approxDist}{\priorDist}
    \end{equation*}
\end{enumerate}
For Gaussian distributions, these have closed form solutions.
\subsubsection{Influence Functions}
For the influence functions experiment in \cref{fig:if_pvi}, we still assume the same likelihood function, but we have a different data generating process. We generate 99 data points from the following student-t distribution with 4 degrees of freedom, mean 0 and scale 1,:
\begin{align*}
    x_{1:99} &\sim Student\,T(0,1,4)\\
    x_{100} &\sim \delta_y(x), \, y\in\mathbb{R}
\end{align*}
We place Huber contamination on the hypothesis, where we add an additional observation to one of seven clients that is increasingly farther from the true mean, and calculate the posteriors with this outlier, $y$. We have used the losses described previously for the posteriors and the Kullback--Leibler divergence, running all experiments to convergence.
We compare the resulting distributions using the Fisher--Rao divergence \citep{nielsen2023}, which has closed form between two univariate Gaussians $\approxDist(\bt)\sim\mathcal{N}(\mu_\approxDist, \sigma_\approxDist^2)$ and $\prior\sim\mathcal{N}(\mu_\priorDist, \sigma_\priorDist^2)$
\begin{equation*}
\begin{aligned}
    \divergence_{FR}(\approxDist:\priorDist) &= \sqrt{2}\log \left(\frac{1+\Delta(\mu_\approxDist,\sigma_\approxDist:\mu_\priorDist,\sigma_\priorDist)}{1-\Delta(\mu_\approxDist,\sigma_\approxDist:\mu_\priorDist,\sigma_\priorDist)}
    \right)\\
    \Delta(a,b:c,d) &:= \sqrt{\cfrac{(c-a)^2 + (d-b)^2}{(c-a)^2 - (d +b)^2}}, \quad (a,b,c,d)\in\mathbb{R}^4\backslash\{0\}
\end{aligned}
\end{equation*}
\subsection{Logistic Regression with Gaussian Design}
We place a mean field Gaussian distribution over the parameters of linear model $\bt^\transpose\x + b$ by augmenting the data to $\tilde{\x}=[1,\x^\transpose]$ in order to allow for non--normalised data sets.
We assume that the labels, $\y_i\in\{0,1\}$, follow a Bernoulli distribution with sigmoid probabilities:
\begin{equation*}
	\y_i\sim\mathrm{Ber}(\sigma(\bt^\transpose\tilde{\x}_i))
\end{equation*}
where $\sigma(a)=(1+e^{-a})^{-1}$ is the sigmoid function.
This allows us to define the likelihood as follows:
\begin{equation*}
\likelihood[i]=\exp\{\y_i\tilde{\x}_i^\transpose\bt-\psi(\tilde{\x}_i^\transpose\bt)\}
\end{equation*}
where $\psi(a):=\log(1+e^a)$, which gives rise to the sigmoid through $\sigma(a)=\psi'(a)$ \citep{katsevich2024}.
We use this exponential family form above since taking the logarithm for the negative log--likelihood is easily achieved by removing the exponential and allows for slightly faster calculations during the optimisation.
Further, we assume that the prior $\prior=\N(\boldsymbol{0}, \Sigma)$, where $\Sigma$ is fixed but generated through sampling from a Gamma distribution and averaging over the samples. 
More specifically we sampled 100 samples from a Gamma distribution with $\xi_{1:100}\overset{\mathrm{iid}}{\sim}\mathrm{Gamma}(1,1/0.01)$, and use their mean, $\bar{\xi}$, to define $\Sigma :=\bar{\xi}\inverse\mathbf{I}_d$.
This was done to ensure fairness with the Distributed Stein Variational Gradient Descent approach of \citet{kassab2022}, who use an Gaussian inverse Gamma prior, which we for ease of implementation forgo (the results of the experiments show that we easily match their performance, if not surpass it slightly).
For the prediction, we use an approximation to the expectation with respect to the final distribution found, $\q[\server]{(\timax)}\sim\N(\mu_\server,\Sigma_\server)$ where $\Sigma_\server$ is a diagonal matrix, as in \citet{ashman2022}.
\begin{equation*}
\likelihoodDist(\y_{\mathrm{new}}=1|\tilde{\x}_{\mathrm{new}})=
\E_{\q[\server]{(\timax)}}\left[\likelihoodDist(\y_{\mathrm{new}}=1|\bt,\tilde{\x}_{\mathrm{new}})\right]
\approx \sigma\left(\cfrac{\mu_\server^\transpose\tilde{\x}_{\mathrm{new}}}{\sqrt{1+\pi\tilde{\x}_\mathrm{new}^\transpose\Sigma_\server\tilde{\x}_\mathrm{new}}}\right)
\end{equation*}
This allows us to forgo Monte Carlo sampling to evaluate this expectation.

\begin{remark}
Since neither GVI, nor \FedGVI targets the Bayesian posterior under different divergences or loss functions in comparison to vanilla VI, we cannot truly speak of this expectation approximating the Bayesian posterior predictive distribution, however since our aim is to find a distribution that is more valuable to a decision maker, using a \FedGVI posterior should allow us to make more informed predictions depending on what the DM wants to model.
This can be better uncertainty quantification through changing the divergence, and/or better prediction accuracy through changing the loss. 
\end{remark}
\subsubsection{Further Experiments}
In \cref{fig:logreg_fedgvi_giv_comp} we compare the predictive performance of \FedGVI with two clients against that of GVI with only one client.
\begin{figure}[h]
	\centering
	\input{{./figs/log_reg_fedgvi_vi.pgf}}
	\caption{Comparing Logistic Regression with \FedGVI where the data set is split across clients, to GVI where the entire data set is available.}
	\label{fig:logreg_fedgvi_giv_comp}
\end{figure}

\subsection{Bayesian Neural Networks}
The model architecture is a fully connected multi--layer perceptron with RELU activation. 
\subsubsection{\MNIST \citep{lecun1988} Details and Additional Experiments}
For the hyperparameters of the competing methods in the BNNs, we follow \citet{hasan2024} in using SGD with momentum with a learning rate of 0.1 for \FedAvg, and \betaPredBayes, and 0.01 for \FedPA. The architecture for these is a 2 hidden layer fully connected neural network, where each hidden layer has 100 neurons.

For \FedGVI and PVI, we follow the set up of \cite{ashman2022} in using the ADAM optimiser \citep{kingma2015} with a learning rate of 0.0005, leaving all other parameters the default values in \texttt{PyTorch}. Here we use a fully connected NN with 1 hidden layer of 200 neurons. 

The contamination maps all contaminated data points of one class to a single other class. 
In both cases, we carried out mini--batch optimisation.

Since we use different architectures for the BNN experiments for the MNIST data set in \cref{tab:bnn}, we additionally report results for BNNs when we use the same Neural Network architecture and still retain superior performance of \FedGVI, see \cref{tab:bnn_1h}.
We notice that the choosing the implementation with the two hidden layer NN for the competing methods performs better or on an equivalent level (within one standard deviation) of each other, while \FedGVI performs better on the single layer NN.

When examining their convergence behaviour under the different architectures, we further notice that the competing methods perform worse than \FedGVI, and \FedAvg and \FedPA exhibit no stability in their accuracy in the contaminated setting, see 
\cref{fig:bnn_results_architecture_1h,,fig:bnn_results_architecture_2h}. This phenomenon occurs even in the uncontaminated case as reported in \citet{al-shedivat2021}, where we have chosen the optimiser and learning rates as suggested in their paper, and hence we conjecture that contamination further exacerbates this. 

\begin{table}[h]
\caption{Classification accuracy (highest in bold) on uncontaminated test data after training on 10\% contaminated \MNIST data. Here, we compare the results with a fully connected Neural Network with 1 hidden layer of 200 Neurons to the results of a fully connected Neural Network with 2 hidden layers of 100 Neurons each.
We report the best performance across all server iterations.}
\label{tab:bnn_1h}
\vskip 0.1in
\begin{center}
\begin{small}
\begin{sc}
\centering
\begin{tabular}{ccccc}
\toprule
\multicolumn{1}{c}{\multirow{2}{*}{Model}} & \multicolumn{2}{c}{1 Hidden Layer} & \multicolumn{2}{c}{2 Hidden Layers}  \\
 \cmidrule(lr){2-3}  \cmidrule(lr){4-5}
  & 10 Clients & 3 Clients & 10 Clients & 3 Clients\\
\midrule
\FedAvg   & 94.79$\pm$0.43 & 91.76$\pm$0.08 & 96.64$\pm$ 0.07& 96.34 $\pm$ 0.20\\
\FedPA & 94.53$\pm$0.15 & 95.74$\pm$0.08 & 94.25$\pm$ 0.39& 95.31$\pm$ 0.35\\
\betaPredBayes  & 94.96$\pm$0.06 & 96.67$\pm$0.07 & 94.90$\pm$ 0.08& 96.73$\pm$ 0.08\\
PVI    & 95.56$\pm$ 0.18& 96.68$\pm$ 0.07 & 95.68$\pm$0.10 & 97.31$\pm$0.08 \\
\FedGVI $\divergence_{AR}$     & 96.36$\pm$ 0.09& 97.13 $\pm$ 0.13 & 95.78$\pm$0.17 & 97.24$\pm$0.05 \\
\FedGVI $\trueloss_{GCE}$      & 97.06$\pm$ 0.03& 98.04 $\pm$ 0.07& 96.57$\pm$0.04 & \textbf{97.74$\pm$0.11} \\
\FedGVI $\divergence_{AR}$+$\trueloss_{GCE}$    & \textbf{97.50$\pm$ 0.07}& \textbf{98.13$\pm$ 0.08}& \textbf{96.77$\pm$0.10} & \textbf{97.79$\pm$0.10} \\
\midrule
VI (1 Client) & \multicolumn{2}{c}{(96.96$\pm$ 0.17)} & \multicolumn{2}{c}{(90.87$\pm$0.50)}\\
GVI (1 Client) & \multicolumn{2}{c}{(\textbf{98.13$\pm$ 0.07})}& \multicolumn{2}{c}{(\textbf{97.56$\pm$0.05})}\\
\bottomrule
\end{tabular}
\end{sc}
\end{small}
\end{center}
\end{table}

\begin{table}[h]
\caption{Classification accuracy (highest in bold for each learning rate) on uncontaminated test data after training on 10\% contaminated \MNIST data split across 3 clients. Here, we compare the results of different initialisations of \FedGVI with the Alpha--\renyi divergence and generalised cross entropy loss achieved when optimising the posteriors with different learning rates of ADAM. We report the best performance after all server iterations.}
\label{tab:lr}
\vskip 0.1in
\centering
\begin{sc}
\begin{small}
\begin{tabular}{ccccccc}
\toprule
\multicolumn{1}{c}{\multirow{2}{*}{Model}} & \multicolumn{6}{c}{Learning Rate $\eta$}  \\
 \cmidrule(lr){2-7}  
& $1e-2$& $5e-3$&$1e-3$ &$5e-4$  &$1e-4$ &$5e-3$\\
\midrule
PVI & 96.34$\pm$0.16 & 96.50$\pm$0.18 & 96.72$\pm$0.06 & {96.76$\pm$0.07} & 96.01$\pm$0.05 & 95.39$\pm$0.06  \\
FedGVI $D_{AR}^{(2.5)}$   & 96.84$\pm$0.12 & 96.91$\pm$0.02 & 97.16$\pm$0.04 & {97.18$\pm$0.03} & 96.51$\pm$0.19 & 95.65$\pm$0.03  \\
FedGVI $\mathcal{L}_{GCE}^{(0.8)}$ & 98.22$\pm$0.07 & \textbf{98.30$\pm$0.03} & 98.15$\pm$0.01 & \textbf{98.08$\pm$0.08} & 97.07$\pm$0.06 & 95.84$\pm$0.04  \\
FedGVI  $D_{AR}^{(2.5)}$ + $\mathcal{L}_{GCE}^{(0.8)}$ & \textbf{98.31$\pm$0.10} & 98.24$\pm$0.07 & \textbf{98.23$\pm$0.06} & 98.06$\pm$0.09 & \textbf{97.50$\pm$0.01} & \textbf{96.35$\pm$0.08}  \\
\bottomrule
\end{tabular}
\end{small}
\end{sc}
\end{table}

In \cref{tab:lr} we compare different learning rates of ADAM with different initialisations of \FedGVI showing that we in fact underperform with the learning rate selected for the experiments in \cref{fig:ablation_study,,fig:bnn_results}, and \cref{tab:bnn} for the outperforming methods of \FedGVI.

In \cref{tab:perturb} we further investigate the stability of \FedGVI posteriors when slightly varying the robustness parameter. This shows no significant variations in the accuracy achieved by \FedGVI when slightly perturbing $\delta=0.8$.

\begin{table}[h]
\caption{We fix $\alpha=2.5$ in the Alpha--\renyi divergence, and vary $\delta$, of the generalised cross entropy loss of \citet{zhang2018}, around $0.8$. We  report accuracies on uncontaminated test data after training on 10\% contaminated MNIST data split across 5 clients. These accuracies vary very little demonstrating stability in the \FedGVI posterior at slight perturbations in the loss parameter.}
\label{tab:perturb}
\vskip 0.1in
\centering
\begin{sc}
\scriptsize
\begin{tabular}{cccccccc}
\toprule
\multicolumn{1}{c}{\multirow{2}{*}{Model}} & \multicolumn{7}{c}{$\delta$ of $\mathcal{L}_{GCE}^{(\delta)}$}  \\
 \cmidrule(lr){2-8}  
& 0.75& 0.775&0.79 &0.8 &0.81 &0.825 & 8.85\\
\midrule
FedGVI  $D_{AR}^{(2.5)}$ + $\mathcal{L}_{GCE}^{(\delta)}$ & 98.05$\pm$0.02 & 98.14$\pm$0.09 & 98.04$\pm$0.09 & 98.06$\pm$0.06 & 98.06$\pm$0.05 & 97.99$\pm$0.07 & 97.98$\pm$0.03   \\
\bottomrule
\end{tabular}
\end{sc}
\vskip 0.1in
\end{table}

\begin{figure}[h]
\centering
\subfigure{\input{{./figs_new/bnn_results_val_acc_10_clients_1h.pgf}}}\hfill
\subfigure{\input{{./figs_new/bnn_results_val_acc_3_clients_1h.pgf}}}\\
\caption{Accuracy on Fully Connected Neural Networks with \textbf{1 Hidden Layer}. We demonstrate convergence of the different approaches examined in the first multicolumn of \cref{tab:bnn_1h}. The models are trained on 10\% label--contaminated data, and prediction accuracy is assessed on uncontaminated test data.}
\label{fig:bnn_results_architecture_1h}
\end{figure}

\begin{figure}[h!]
\centering
\subfigure{\input{{./figs_new/bnn_results_val_acc_10_clients_2h.pgf}}}\hfill
\subfigure{\input{{./figs_new/bnn_results_val_acc_3_clients_2h.pgf}}}\\
\caption{Accuracy on Fully Connected Neural Networks with \textbf{2 Hidden Layers}. We demonstrate convergence of the different approaches examined in the second multicolumn of \cref{tab:bnn_1h}. The models are trained on 10\% label--contaminated data, and prediction accuracy is assessed on uncontaminated test data.}
\label{fig:bnn_results_architecture_2h}
\end{figure}

We also want to highlight that by not carefully selecting the hyperparameters of \FedGVI, as well as the learning rate, and keeping these constant across the BNN experiments, we have shown that you do not require extensive knowledge to adapt existing PVI approaches to \FedGVI and outperform. For instance, \FedGVI performs even better for the robust losses at a higher learning rate, but we have shown in \cref{tab:bnn} that it still outperforms even when not carefully selecting a learning rate.
Furthermore, choosing $\delta=0.6$ and $\alpha=2.5$ would have performed better when varying only the robustness parameters of FedGVI, as seen in \cref{fig:ablation_study}.

Lastly, \cref{fig:wall_clock_time_fedgvi_pvi} shows that even when the loss is not available in a conjugate, closed form way, that \FedGVI still incurs no significant computational overhead through choosing the Alpha--\renyi divergence or the generalised cross entropy loss.

\begin{figure}
	\centering
    \input{{./figs_rebuttal/bnn_results_val_acc_time_10_clients.pgf}}
	\caption{Wall--clock times for \FedGVI iterations per client. We plot the classification error against the computation time taken per client during each server iteration, where we train 5 Clients on 10\% contaminated MNIST data. If we do not state the loss or divergence in the legend, it is $\mathcal{L}_{NLL}$ and $D_{KL}$ respectively. Here, \FedGVI outperforms PVI in terms of accuracy while having similar runtimes.}
	\label{fig:wall_clock_time_fedgvi_pvi}
\end{figure}

\newpage
\subsubsection{\FashionMNIST \citep{xiao2017} Details}
We vary the amount of contamination from 0.0, 0.1, 0.2, 0.4, where the contamination is random and assigns each contaminated data point a different class uniformly at random. The model architecture, prior, learning rate, and optimiser remain unchanged and are as before.

\FedGVI uses the Alpha--\renyi divergence with an alpha value of 2.5 for all, and the robust generalised cross entropy loss, where $\delta=0.0$ indicates the negative log likelihood. \cref{tab:bnn_fmnist_max_rebuttal} specifies \cref{tab:contamination} to a higher precision but the results are identical.

\begin{table}[h]
\caption{Classification accuracy (highest in bold) on uncontaminated test data after training on different amounts of contaminated \FashionMNIST data. Each Method has data split homogeneously across 3 Clients. We report the best performance during all server iterations for each method.}
\label{tab:bnn_fmnist_max_rebuttal}
\vskip 0.1in
\begin{center}
\begin{small}
\begin{sc}
\centering
\begin{tabular}{ccccc}
\toprule
\multicolumn{1}{c}{\multirow{2}{*}{Model}} & \multicolumn{4}{c}{Contamination}  \\
 \cmidrule(lr){2-5}  
  & 0\% & 10\% & 20\% & 40\% \\
\midrule
\FedAvg & 85.72$\pm$0.52 & 78.99$\pm$1.90 & 71.16$\pm$1.53 & 48.97$\pm$6.51\\
\FedPA & 88.08$\pm$0.30 & 87.36$\pm$0.15 & 86.54$\pm$0.16 & 85.36$\pm$0.53\\
\betaPredBayes & 87.58$\pm$0.13 & 87.20$\pm$0.12 & 86.82$\pm$0.07 & 85.77$\pm$0.10\\
PVI    & 86.21$\pm$0.21 & 85.14$\pm$0.13 & 84.36$\pm$0.12 & 82.81$\pm$0.05\\
\FedGVI $\delta=0.0$     & 87.12$\pm$0.12 & 86.23$\pm$0.15 & 85.56$\pm$0.11 & 83.78$\pm$0.09\\
\FedGVI $\delta=0.4$ & {88.73$\pm$0.21} & \textbf{88.60$\pm$0.09} & 87.01$\pm$0.37 & 78.14$\pm$0.39\\
\FedGVI $\delta=0.5$ & \textbf{89.02$\pm$0.18} & {88.57$\pm$0.16} & \textbf{88.39$\pm$0.21} & 85.06$\pm$0.67\\
\FedGVI $\delta=0.8$      & 88.59$\pm$0.03 & 88.44$\pm$0.07 & 87.95$\pm$0.04 & \textbf{87.21$\pm$0.10}\\
\FedGVI $\delta=1.0$    & 88.09$\pm$0.08 & 87.83$\pm$0.14 & 87.54$\pm$0.15 & 85.97$\pm$0.27\\
\bottomrule
\end{tabular}
\end{sc}
\end{small}
\end{center}
\vskip -0.1in
\end{table}


\begin{thebibliography}{85}
\providecommand{\natexlab}[1]{#1}
\providecommand{\url}[1]{\texttt{#1}}
\expandafter\ifx\csname urlstyle\endcsname\relax
  \providecommand{\doi}[1]{doi: #1}\else
  \providecommand{\doi}{doi: \begingroup \urlstyle{rm}\Url}\fi

\bibitem[Achituve et~al.(2021)Achituve, Shamsian, Navon, Chechik, and
  Fetaya]{achituve2021}
Achituve, I., Shamsian, A., Navon, A., Chechik, G., and Fetaya, E.
\newblock Personalized federated learning with {G}aussian processes.
\newblock In Beygelzimer, A., Dauphin, Y., Liang, P., and Vaughan, J.~W.
  (eds.), \emph{Advances in Neural Information Processing Systems}, 2021.

\bibitem[Ahn et~al.(2014)Ahn, Shahbaba, and Welling]{ahn2014}
Ahn, S., Shahbaba, B., and Welling, M.
\newblock Distributed stochastic gradient {MCMC}.
\newblock In Xing, E.~P. and Jebara, T. (eds.), \emph{Proceedings of the 31st
  International Conference on Machine Learning}, volume~32 of \emph{Proceedings
  of Machine Learning Research}, pp.\  1044--1052, Bejing, China, 2014. PMLR.

\bibitem[Al-Shedivat et~al.(2021)Al-Shedivat, Gillenwater, Xing, and
  Rostamizadeh]{al-shedivat2021}
Al-Shedivat, M., Gillenwater, J., Xing, E., and Rostamizadeh, A.
\newblock Federated learning via posterior averaging: A new perspective and
  practical algorithms.
\newblock In \emph{International Conference on Learning Representations}, 2021.

\bibitem[Ali \& Silvey(1966)Ali and Silvey]{ali1966}
Ali, S.~M. and Silvey, S.~D.
\newblock A general class of coefficients of divergence of one distribution
  from another.
\newblock \emph{Journal of the Royal Statistical Society. Series B
  (Methodological)}, 28\penalty0 (1):\penalty0 131--142, 1966.

\bibitem[Allouah et~al.(2024)Allouah, Farhadkhani, Guerraoui, Gupta, Pinot,
  Rizk, and Voitovych]{allouah2024}
Allouah, Y., Farhadkhani, S., Guerraoui, R., Gupta, N., Pinot, R., Rizk, G.,
  and Voitovych, S.
\newblock {B}yzantine-robust federated learning: Impact of client subsampling
  and local updates.
\newblock In Salakhutdinov, R., Kolter, Z., Heller, K., Weller, A., Oliver, N.,
  Scarlett, J., and Berkenkamp, F. (eds.), \emph{Proceedings of the 41st
  International Conference on Machine Learning}, volume 235 of
  \emph{Proceedings of Machine Learning Research}, pp.\  1078--1114. PMLR,
  21--27 Jul 2024.

\bibitem[Alquier(2021)]{alquier2021}
Alquier, P.
\newblock Non-exponentially weighted aggregation: Regret bounds for unbounded
  loss functions.
\newblock In Meila, M. and Zhang, T. (eds.), \emph{Proceedings of the 38th
  International Conference on Machine Learning}, volume 139 of
  \emph{Proceedings of Machine Learning Research}, pp.\  207--218. PMLR, 18--24
  Jul 2021.

\bibitem[Alquier et~al.(2016)Alquier, Ridgway, and Chopin]{alquier2016}
Alquier, P., Ridgway, J., and Chopin, N.
\newblock On the properties of variational approximations of gibbs posteriors.
\newblock \emph{Journal of Machine Learning Research}, 17\penalty0
  (236):\penalty0 1--41, 2016.

\bibitem[Altamirano et~al.(2023)Altamirano, Briol, and
  Knoblauch]{altamirano2023}
Altamirano, M., Briol, F.-X., and Knoblauch, J.
\newblock Robust and scalable {B}ayesian online changepoint detection.
\newblock In Krause, A., Brunskill, E., Cho, K., Engelhardt, B., Sabato, S.,
  and Scarlett, J. (eds.), \emph{Proceedings of the 40th International
  Conference on Machine Learning}, volume 202 of \emph{Proceedings of Machine
  Learning Research}, pp.\  642--663. PMLR, 23--29 Jul 2023.

\bibitem[Altamirano et~al.(2024)Altamirano, Briol, and
  Knoblauch]{altamirano2024}
Altamirano, M., Briol, F.-X., and Knoblauch, J.
\newblock Robust and conjugate {G}aussian process regression.
\newblock In Salakhutdinov, R., Kolter, Z., Heller, K., Weller, A., Oliver, N.,
  Scarlett, J., and Berkenkamp, F. (eds.), \emph{Proceedings of the 41st
  International Conference on Machine Learning}, volume 235 of
  \emph{Proceedings of Machine Learning Research}, pp.\  1155--1185. PMLR,
  21--27 Jul 2024.

\bibitem[Amari(2016)]{amari2016}
Amari, S.-i.
\newblock \emph{Information Geometry and Its Applications}.
\newblock Springer, Tokyo, Japan, 2016.
\newblock ISBN 9784431559771.

\bibitem[Ashman et~al.(2022)Ashman, Bui, Nguyen, Markou, Weller, Swaroop, and
  Turner]{ashman2022}
Ashman, M., Bui, T.~D., Nguyen, C.~V., Markou, S., Weller, A., Swaroop, S., and
  Turner, R.~E.
\newblock Partitioned variational inference: A framework for probabilistic
  federated learning.
\newblock \textit{arXiv preprint arXiv:2202.12275}, 2022.

\bibitem[Bao et~al.(2024)Bao, Wu, and He]{bao2024}
Bao, W., Wu, J., and He, J.
\newblock {BOBA}: Byzantine-robust federated learning with label skewness.
\newblock In Dasgupta, S., Mandt, S., and Li, Y. (eds.), \emph{Proceedings of
  The 27th International Conference on Artificial Intelligence and Statistics},
  volume 238 of \emph{Proceedings of Machine Learning Research}, pp.\
  892--900. PMLR, 02--04 May 2024.

\bibitem[Berger(1985)]{berger1985}
Berger, J.~O.
\newblock \emph{Statistical Decision Theory and {B}ayesian Analysis}.
\newblock Springer--Verlag, New York, 1985.
\newblock ISBN 9781475742862.

\bibitem[Berk(1966)]{berk1966}
Berk, R.~H.
\newblock Limiting behavior of posterior distributions when the model is
  incorrect.
\newblock \emph{The Annals of Mathematical Statistics}, 37\penalty0
  (1):\penalty0 51 -- 58, 1966.

\bibitem[Bernardo \& Smith(2000)Bernardo and Smith]{bernardo2000}
Bernardo, J.~M. and Smith, A. F.~M.
\newblock \emph{Bayesian theory}.
\newblock Wiley Series in Probability and Statistics, Chichester, England,
  2000.
\newblock ISBN 9780470316870.

\bibitem[Bissiri et~al.(2016)Bissiri, Holmes, and Walker]{bissiri2016}
Bissiri, P.~G., Holmes, C., and Walker, S.~G.
\newblock A general framework for updating belief distributions.
\newblock \emph{Journal of the Royal Statistical Society. Series B (Statistical
  Methodology)}, 78\penalty0 (5):\penalty0 1103--1130, 2016.

\bibitem[Blei et~al.(2017)Blei, Kucukelbir, and McAuliffe]{blei2017}
Blei, D.~M., Kucukelbir, A., and McAuliffe, J.~D.
\newblock Variational inference: A review for statisticians.
\newblock \emph{Journal of the American Statistical Association}, 112\penalty0
  (518):\penalty0 859--877, 2017.

\bibitem[Bui et~al.(2018)Bui, Nguyen, Swaroop, and Turner]{bui2018}
Bui, T.~D., Nguyen, C.~V., Swaroop, S., and Turner, R.~E.
\newblock Partitioned variational inference: A unified framework encompassing
  federated and continual learning.
\newblock \textit{arXiv preprint arXiv:1811.11206}, 2018.

\bibitem[Carvalho et~al.(2023)Carvalho, Villela, Coelho, and
  Bastos]{carvalho2023}
Carvalho, L.~M., Villela, D. A.~M., Coelho, F.~C., and Bastos, L.~S.
\newblock {B}ayesian inference for the weights in logarithmic pooling.
\newblock \emph{Bayesian Analysis}, 18\penalty0 (1):\penalty0 223 -- 251, 2023.

\bibitem[Chan et~al.(2023)Chan, Pollock, Johansen, and Roberts]{chan2023}
Chan, R.~S., Pollock, M., Johansen, A.~M., and Roberts, G.~O.
\newblock Divide-and-conquer fusion.
\newblock \emph{Journal of Machine Learning Research}, 24\penalty0
  (193):\penalty0 1--82, 2023.

\bibitem[Chen et~al.(2022)Chen, Choquette-Choo, Kairouz, and Suresh]{chen2022}
Chen, W.-N., Choquette-Choo, C.~A., Kairouz, P., and Suresh, A.~T.
\newblock The fundamental price of secure aggregation in differentially private
  federated learning.
\newblock In Chaudhuri, K., Jegelka, S., Song, L., Szepesvari, C., Niu, G., and
  Sabato, S. (eds.), \emph{Proceedings of the 39th International Conference on
  Machine Learning}, volume 162 of \emph{Proceedings of Machine Learning
  Research}, pp.\  3056--3089. PMLR, 17--23 Jul 2022.

\bibitem[Cichocki \& Amari(2010)Cichocki and Amari]{chichocki2010}
Cichocki, A. and Amari, S.-i.
\newblock Families of alpha- beta- and gamma- divergences: Flexible and robust
  measures of similarities.
\newblock \emph{Entropy}, 12\penalty0 (6):\penalty0 1532--1568, 2010.

\bibitem[Corinzia et~al.(2021)Corinzia, Beuret, and Buhmann]{corinzia2021}
Corinzia, L., Beuret, A., and Buhmann, J.~M.
\newblock Variational federated multi-task learning.
\newblock \textit{arXiv preprint arXiv:1906.06268}, 2021.

\bibitem[Demidovich et~al.(2025)Demidovich, Ostroukhov, Malinovsky,
  Horv{\'a}th, Tak{\'a}{\v{c}}, Richt{\'a}rik, and Gorbunov]{demidovich2024}
Demidovich, Y., Ostroukhov, P., Malinovsky, G., Horv{\'a}th, S.,
  Tak{\'a}{\v{c}}, M., Richt{\'a}rik, P., and Gorbunov, E.
\newblock Methods with local steps and random reshuffling for generally smooth
  non-convex federated optimization.
\newblock In \emph{The Thirteenth International Conference on Learning
  Representations}, 2025.

\bibitem[Diaconis \& Freedman(1986)Diaconis and Freedman]{diaconis1986}
Diaconis, P. and Freedman, D.
\newblock On the consistency of {B}ayes estimates.
\newblock \emph{The Annals of Statistics}, 14\penalty0 (1):\penalty0 1 -- 26,
  1986.

\bibitem[Fraboni et~al.(2023)Fraboni, Vidal, Kameni, and Lorenzi]{fraboni2023}
Fraboni, Y., Vidal, R., Kameni, L., and Lorenzi, M.
\newblock A general theory for federated optimization with asynchronous and
  heterogeneous clients updates.
\newblock \emph{Journal of Machine Learning Research}, 24\penalty0
  (110):\penalty0 1--43, 2023.

\bibitem[Genest(1984)]{genest1984}
Genest, C.
\newblock A characterization theorem for externally {B}ayesian groups.
\newblock \emph{The Annals of Statistics}, 12\penalty0 (3):\penalty0
  1100--1105, 1984.

\bibitem[Genest et~al.(1986)Genest, McConway, and Schervish]{genest1986}
Genest, C., McConway, K.~J., and Schervish, M.~J.
\newblock Characterization of externally {B}ayesian pooling operators.
\newblock \emph{The Annals of Statistics}, 14\penalty0 (2):\penalty0 487 --
  501, 1986.

\bibitem[Ghosh \& Basu(2016a)Ghosh and Basu]{ghosh2016a}
Ghosh, A. and Basu, A.
\newblock Robust {B}ayes estimation using the density power divergence.
\newblock \emph{Annals of the Institute of Statistical Mathematics},
  68\penalty0 (2):\penalty0 413--437, 2016a.

\bibitem[Ghosh \& Basu(2016b)Ghosh and Basu]{ghosh2016b}
Ghosh, A. and Basu, A.
\newblock Robust estimation in generalized linear models: the density power
  divergence approach.
\newblock \emph{TEST}, 25\penalty0 (2):\penalty0 269--290, 2016b.

\bibitem[Gr{\"u}nwald(2012)]{gruenwald2012}
Gr{\"u}nwald, P.
\newblock The safe {B}ayesian.
\newblock In Bshouty, N.~H., Stoltz, G., Vayatis, N., and Zeugmann, T. (eds.),
  \emph{Algorithmic Learning Theory}, pp.\  169--183, Berlin, Heidelberg, 2012.
  Springer Berlin Heidelberg.

\bibitem[Guo et~al.(2023)Guo, Greengard, Wang, Gelman, Kim, and Xing]{guo2023}
Guo, H., Greengard, P., Wang, H., Gelman, A., Kim, Y., and Xing, E.
\newblock Federated learning as variational inference: A scalable expectation
  propagation approach.
\newblock In \emph{The Eleventh International Conference on Learning
  Representations}, 2023.

\bibitem[Hamer et~al.(2020)Hamer, Mohri, and Suresh]{hamer2020}
Hamer, J., Mohri, M., and Suresh, A.~T.
\newblock {F}ed{B}oost: A communication-efficient algorithm for federated
  learning.
\newblock In III, H.~D. and Singh, A. (eds.), \emph{Proceedings of the 37th
  International Conference on Machine Learning}, volume 119 of
  \emph{Proceedings of Machine Learning Research}, pp.\  3973--3983. PMLR,
  2020.

\bibitem[Hasan et~al.(2024)Hasan, Zhang, Guo, Chen, and Poupart]{hasan2024}
Hasan, M., Zhang, G., Guo, K., Chen, X., and Poupart, P.
\newblock Calibrated one round federated learning with {B}ayesian inference in
  the predictive space.
\newblock \emph{Proceedings of the AAAI Conference on Artificial Intelligence},
  38\penalty0 (11):\penalty0 12313--12321, 2024.

\bibitem[Hasenclever et~al.(2017)Hasenclever, Webb, Lienart, Vollmer,
  Lakshminarayanan, Blundell, and Teh]{hasenclever2017}
Hasenclever, L., Webb, S., Lienart, T., Vollmer, S., Lakshminarayanan, B.,
  Blundell, C., and Teh, Y.~W.
\newblock Distributed {B}ayesian learning with stochastic natural gradient
  expectation propagation and the posterior server.
\newblock \emph{Journal of Machine Learning Research}, 18\penalty0
  (1):\penalty0 3744–3780, 2017.

\bibitem[Hassan et~al.(2023)Hassan, Salomone, and Mengersen]{hassan2023}
Hassan, C., Salomone, R., and Mengersen, K.
\newblock Federated variational inference methods for structured latent
  variable models.
\newblock \textit{arXiv preprint arXiv:2302.03314}, 2023.

\bibitem[Hassan et~al.(2024)Hassan, Sutton, Mira, and Mengersen]{hassan2024}
Hassan, C., Sutton, M., Mira, A., and Mengersen, K.
\newblock Scalable vertical federated learning via data augmentation and
  amortized inference.
\newblock \textit{arXiv preprint arXiv:2405.04043}, 2024.

\bibitem[Heikkil{\"a} et~al.(2023)Heikkil{\"a}, Ashman, Swaroop, Turner, and
  Honkela]{heikillae2023}
Heikkil{\"a}, M., Ashman, M., Swaroop, S., Turner, R., and Honkela, A.
\newblock Differentially private partitioned variational inference.
\newblock \emph{Transactions on machine learning research}, 2023\penalty0 (4),
  2023.

\bibitem[Hooker \& Vidyashankar(2014)Hooker and Vidyashankar]{hooker2014}
Hooker, G. and Vidyashankar, A.~N.
\newblock Bayesian model robustness via disparities.
\newblock \emph{TEST}, 23\penalty0 (3):\penalty0 556--584, 2014.

\bibitem[Huber(1964)]{huber1964}
Huber, P.~J.
\newblock Robust estimation of a location parameter.
\newblock \emph{Annals of Mathematical Statistics}, 35:\penalty0 73--101, 1964.

\bibitem[Hung et~al.(2018)Hung, Jou, and Huang]{hung2018}
Hung, H., Jou, Z.-Y., and Huang, S.-Y.
\newblock Robust mislabel logistic regression without modeling mislabel
  probabilities.
\newblock \emph{Biometrics}, 74\penalty0 (1):\penalty0 145--154, 2018.

\bibitem[Hyv{{\"a}}rinen(2005)]{hyvarinen2005}
Hyv{{\"a}}rinen, A.
\newblock Estimation of non-normalized statistical models by score matching.
\newblock \emph{Journal of Machine Learning Research}, 6\penalty0
  (24):\penalty0 695--709, 2005.

\bibitem[Jewson et~al.(2018)Jewson, Smith, and Holmes]{jewson2018}
Jewson, J., Smith, J.~Q., and Holmes, C.
\newblock Principles of {B}ayesian inference using general divergence criteria.
\newblock \emph{Entropy}, 20\penalty0 (6):\penalty0 442, 2018.

\bibitem[Jonker et~al.(2024)Jonker, Pazira, and Coolen]{jonker2024}
Jonker, M.~A., Pazira, H., and Coolen, A.~C.
\newblock Bayesian federated inference for estimating statistical models based
  on non-shared multicenter data sets.
\newblock \emph{Statistics in Medicine}, pp.\  1--18, 2024.

\bibitem[Kairouz et~al.(2021)Kairouz, McMahan, Avent, Bellet, Bennis, Bhagoji,
  Bonawitz, Charles, Cormode, Cummings, D’Oliveira, Eichner, Rouayheb, Evans,
  Gardner, Garrett, Gascón, Ghazi, Gibbons, Gruteser, Harchaoui, He, He, Huo,
  Hutchinson, Hsu, Jaggi, Javidi, Joshi, Khodak, Konecný, Korolova,
  Koushanfar, Koyejo, Lepoint, Liu, Mittal, Mohri, Nock, Özgür, Pagh, Qi,
  Ramage, Raskar, Raykova, Song, Song, Stich, Sun, Suresh, Tramèr, Vepakomma,
  Wang, Xiong, Xu, Yang, Yu, Yu, and Zhao]{kairouz2021}
Kairouz, P., McMahan, H.~B., Avent, B., Bellet, A., Bennis, M., Bhagoji, A.~N.,
  Bonawitz, K., Charles, Z., Cormode, G., Cummings, R., D’Oliveira, R. G.~L.,
  Eichner, H., Rouayheb, S.~E., Evans, D., Gardner, J., Garrett, Z., Gascón,
  A., Ghazi, B., Gibbons, P.~B., Gruteser, M., Harchaoui, Z., He, C., He, L.,
  Huo, Z., Hutchinson, B., Hsu, J., Jaggi, M., Javidi, T., Joshi, G., Khodak,
  M., Konecný, J., Korolova, A., Koushanfar, F., Koyejo, S., Lepoint, T., Liu,
  Y., Mittal, P., Mohri, M., Nock, R., Özgür, A., Pagh, R., Qi, H., Ramage,
  D., Raskar, R., Raykova, M., Song, D., Song, W., Stich, S.~U., Sun, Z.,
  Suresh, A.~T., Tramèr, F., Vepakomma, P., Wang, J., Xiong, L., Xu, Z., Yang,
  Q., Yu, F.~X., Yu, H., and Zhao, S.
\newblock Advances and open problems in federated learning.
\newblock \emph{Foundations and Trends® in Machine Learning}, 14\penalty0
  (1–2):\penalty0 1--210, 2021.

\bibitem[Kallioinen et~al.(2024)Kallioinen, Paananen, B{\"u}rkner, and
  Vehtari]{kallioinen2024}
Kallioinen, N., Paananen, T., B{\"u}rkner, P.-C., and Vehtari, A.
\newblock Detecting and diagnosing prior and likelihood sensitivity with
  power-scaling.
\newblock \emph{Statistics and Computing}, 34\penalty0 (1):\penalty0 57, 2024.

\bibitem[Kassab \& Simeone(2022)Kassab and Simeone]{kassab2022}
Kassab, R. and Simeone, O.
\newblock Federated generalized {B}ayesian learning via distributed {S}tein
  variational gradient descent.
\newblock \emph{IEEE Transactions on Signal Processing}, 70:\penalty0
  2180--2192, 2022.

\bibitem[Katsevich \& Rigollet(2024)Katsevich and Rigollet]{katsevich2024}
Katsevich, A. and Rigollet, P.
\newblock {On the approximation accuracy of Gaussian variational inference}.
\newblock \emph{The Annals of Statistics}, 52\penalty0 (4):\penalty0 1384 --
  1409, 2024.

\bibitem[Kim \& Hospedales(2023)Kim and Hospedales]{kim2023}
Kim, M. and Hospedales, T.
\newblock Fed{HB}: Hierarchical {B}ayesian federated learning.
\newblock \textit{arXiv preprint arXiv:2305.04979}, 2023.

\bibitem[Kingma \& Ba(2015)Kingma and Ba]{kingma2015}
Kingma, D.~P. and Ba, J.
\newblock Adam: {A} method for stochastic optimization.
\newblock In \emph{3rd International Conference on Learning Representations},
  2015.

\bibitem[Knoblauch et~al.(2018)Knoblauch, Jewson, and Damoulas]{jeremias2018}
Knoblauch, J., Jewson, J.~E., and Damoulas, T.
\newblock Doubly robust {B}ayesian inference for non-stationary streaming data
  with $\beta$-divergences.
\newblock In \emph{Advances in Neural Information Processing Systems},
  volume~31, pp.\  64--75. Curran Associates, Inc., 2018.

\bibitem[Knoblauch et~al.(2022)Knoblauch, Jewson, and Damoulas]{jeremias2022}
Knoblauch, J., Jewson, J., and Damoulas, T.
\newblock An optimization-centric view on {B}ayes' rule: Reviewing and
  generalizing variational inference.
\newblock \emph{Journal of Machine Learning Research}, 23\penalty0
  (132):\penalty0 1--109, 2022.

\bibitem[Kotelevskii et~al.(2022)Kotelevskii, Vono, Durmus, and
  Moulines]{kotelevskii2022}
Kotelevskii, N.~Y., Vono, M., Durmus, A., and Moulines, E.
\newblock Fed{P}op: A {B}ayesian approach for personalised federated learning.
\newblock In Oh, A.~H., Agarwal, A., Belgrave, D., and Cho, K. (eds.),
  \emph{Advances in Neural Information Processing Systems}, 2022.

\bibitem[Kullback \& Leibler(1951)Kullback and Leibler]{kullback1951}
Kullback, S. and Leibler, R.~A.
\newblock {On Information and Sufficiency}.
\newblock \emph{The Annals of Mathematical Statistics}, 22\penalty0
  (1):\penalty0 79 -- 86, 1951.

\bibitem[Le{C}un et~al.(1998)Le{C}un, Bottou, Bengio, and Haffner]{lecun1988}
Le{C}un, Y., Bottou, L., Bengio, Y., and Haffner, P.
\newblock Gradient-based learning applied to document recognition.
\newblock \emph{Proceedings of the IEEE}, 86\penalty0 (11):\penalty0
  2278--2324, 1998.

\bibitem[Li et~al.(2024)Li, Acharya, and Richt{\'a}rik]{li2024}
Li, H., Acharya, K., and Richt{\'a}rik, P.
\newblock The power of extrapolation in federated learning.
\newblock In \emph{The Thirty-eighth Annual Conference on Neural Information
  Processing Systems}, 2024.

\bibitem[Malinovsky et~al.(2020)Malinovsky, Kovalev, Gasanov, Condat, and
  Richtarik]{malinovsky2020}
Malinovsky, G., Kovalev, D., Gasanov, E., Condat, L., and Richtarik, P.
\newblock From local {SGD} to local fixed-point methods for federated learning.
\newblock In III, H.~D. and Singh, A. (eds.), \emph{Proceedings of the 37th
  International Conference on Machine Learning}, volume 119 of
  \emph{Proceedings of Machine Learning Research}, pp.\  6692--6701. PMLR,
  2020.

\bibitem[Matsubara et~al.(2022)Matsubara, Knoblauch, Briol, and
  Oates]{matsubara2022}
Matsubara, T., Knoblauch, J., Briol, F.-X., and Oates, C.~J.
\newblock Robust generalised {B}ayesian inference for intractable likelihoods.
\newblock \emph{Journal of the Royal Statistical Society Series B: Statistical
  Methodology}, 84\penalty0 (3):\penalty0 997--1022, 04 2022.

\bibitem[McMahan et~al.(2017)McMahan, Moore, Ramage, Hampson, and
  Arcas]{mcmahan2017}
McMahan, B., Moore, E., Ramage, D., Hampson, S., and Arcas, B. A.~y.
\newblock Communication-efficient learning of deep networks from decentralized
  data.
\newblock In Singh, A. and Zhu, J. (eds.), \emph{Proceedings of the 20th
  International Conference on Artificial Intelligence and Statistics},
  volume~54 of \emph{Proceedings of Machine Learning Research}, pp.\
  1273--1282. PMLR, 2017.

\bibitem[Mekkaoui et~al.(2021)Mekkaoui, Mesquita, Blomstedt, and
  Kaski]{mekkaoui2021}
Mekkaoui, K.~e., Mesquita, D., Blomstedt, P., and Kaski, S.
\newblock Federated stochastic gradient {L}angevin dynamics.
\newblock In de~Campos, C. and Maathuis, M.~H. (eds.), \emph{Proceedings of the
  Thirty-Seventh Conference on Uncertainty in Artificial Intelligence}, volume
  161 of \emph{Proceedings of Machine Learning Research}, pp.\  1703--1712.
  PMLR, 2021.

\bibitem[Mesquita et~al.(2020)Mesquita, Blomstedt, and Kaski]{mesquita2020}
Mesquita, D., Blomstedt, P., and Kaski, S.
\newblock Embarrassingly parallel {MCMC} using deep invertible transformations.
\newblock In Adams, R.~P. and Gogate, V. (eds.), \emph{Proceedings of The 35th
  Uncertainty in Artificial Intelligence Conference}, volume 115 of
  \emph{Proceedings of Machine Learning Research}, pp.\  1244--1252. PMLR,
  2020.

\bibitem[Miller(2021)]{miller2021}
Miller, J.~W.
\newblock Asymptotic normality, concentration, and coverage of generalized
  posteriors.
\newblock \emph{Journal of Machine Learning Research}, 22\penalty0
  (168):\penalty0 1--53, 2021.

\bibitem[Minka(2001)]{minka2001b}
Minka, T.~P.
\newblock Expectation propagation for approximate {B}ayesian inference.
\newblock In \emph{Proceedings of the Seventeenth Conference on Uncertainty in
  Artificial Intelligence}, pp.\  362–369, San Francisco, CA, USA, 2001.

\bibitem[Nielsen(2020)]{nielsen2020}
Nielsen, F.
\newblock An elementary introduction to information geometry.
\newblock \emph{Entropy}, 22\penalty0 (10):\penalty0 1100, 2020.

\bibitem[Nielsen(2023)]{nielsen2023}
Nielsen, F.
\newblock A simple approximation method for the fisher–rao distance between
  multivariate normal distributions.
\newblock \emph{Entropy}, 25\penalty0 (4), 2023.

\bibitem[Opper \& Winther(2005)Opper and Winther]{opper2005}
Opper, M. and Winther, O.
\newblock Expectation consistent approximate inference.
\newblock \emph{Journal of Machine Learning Research}, 6\penalty0
  (73):\penalty0 2177--2204, 2005.

\bibitem[Pardo~Llorente(2006)]{pardo2006}
Pardo~Llorente, L.
\newblock \emph{Statistical inference based on divergence measures}.
\newblock Chapman \& Hall/CRC, 2006.
\newblock ISBN 9781584886006.

\bibitem[Pinski et~al.(2015)Pinski, Simpson, Stuart, and Weber]{pinski2015}
Pinski, F.~J., Simpson, G., Stuart, A.~M., and Weber, H.
\newblock Kullback-leibler approximation for probability measures on infinite
  dimensional spaces.
\newblock \emph{SIAM Journal on Mathematical Analysis}, 47\penalty0
  (6):\penalty0 4091–4122, 2015.

\bibitem[Reddi et~al.(2021)Reddi, Charles, Zaheer, Garrett, Rush, Konečný,
  Kumar, and McMahan]{reddi2021}
Reddi, S., Charles, Z.~B., Zaheer, M., Garrett, Z., Rush, K., Konečný, J.,
  Kumar, S., and McMahan, B.
\newblock Adaptive federated optimization.
\newblock In \emph{International Conference on Learning Representations}, 2021.

\bibitem[Scott et~al.(2016)Scott, Blocker, Bonassi, Chipman, George, and
  McCulloch]{scott2016}
Scott, S.~L., Blocker, A.~W., Bonassi, F.~V., Chipman, H.~A., George, E.~I.,
  and McCulloch, R.~E.
\newblock Bayes and big data: The consensus monte carlo algorithm.
\newblock \emph{International Journal of Management Science and Engineering
  Management}, 11:\penalty0 78--88, 2016.

\bibitem[Swaroop et~al.(2025)Swaroop, Khan, and Doshi-Velez]{swaroop2025}
Swaroop, S., Khan, M.~E., and Doshi-Velez, F.
\newblock Connecting federated {ADMM} to {B}ayes.
\newblock In \emph{The Thirteenth International Conference on Learning
  Representations}, 2025.

\bibitem[Tenison et~al.(2023)Tenison, Sreeramadas, Mugunthan, Oyallon, Rish,
  and Belilovsky]{tenison2023}
Tenison, I., Sreeramadas, S.~A., Mugunthan, V., Oyallon, E., Rish, I., and
  Belilovsky, E.
\newblock Gradient masked averaging for federated learning.
\newblock \emph{Transactions on Machine Learning Research}, 2023.

\bibitem[Tresp(2000)]{tresp2000}
Tresp, V.
\newblock A {B}ayesian committee machine.
\newblock \emph{Neural computation}, 12:\penalty0 2719--41, 2000.

\bibitem[Tziotis et~al.(2023)Tziotis, Shen, Pedarsani, Hassani, and
  Mokhtari]{tziotis2023}
Tziotis, I., Shen, Z., Pedarsani, R., Hassani, H., and Mokhtari, A.
\newblock Straggler-resilient personalized federated learning.
\newblock \emph{Transactions on Machine Learning Research}, 2023.

\bibitem[Vedadi et~al.(2024)Vedadi, Dillon, Mansfield, Singhal, Afkanpour, and
  Morningstar]{vedadi2024}
Vedadi, E., Dillon, J.~V., Mansfield, P.~A., Singhal, K., Afkanpour, A., and
  Morningstar, W.~R.
\newblock Federated variational inference: Towards improved personalization and
  generalization.
\newblock \emph{Transactions on Machine Learning Research}, 2024.

\bibitem[Vehtari et~al.(2020)Vehtari, Gelman, Sivula, Jyl\"{a}nki, Tran, Sahai,
  Blomstedt, Cunningham, Schiminovich, and Robert]{vehtari2020}
Vehtari, A., Gelman, A., Sivula, T., Jyl\"{a}nki, P., Tran, D., Sahai, S.,
  Blomstedt, P., Cunningham, J.~P., Schiminovich, D., and Robert, C.~P.
\newblock Expectation propagation as a way of life: A framework for {B}ayesian
  inference on partitioned data.
\newblock \emph{Journal of Machine Learning Research}, 21\penalty0 (1), 2020.

\bibitem[Walker(2013)]{walker2013}
Walker, S.~G.
\newblock Bayesian inference with misspecified models.
\newblock \emph{Journal of Statistical Planning and Inference}, 143\penalty0
  (10):\penalty0 1621--1633, 2013.

\bibitem[Xiao et~al.(2017)Xiao, Rasul, and Vollgraf]{xiao2017}
Xiao, H., Rasul, K., and Vollgraf, R.
\newblock Fashion-mnist: a novel image dataset for benchmarking machine
  learning algorithms.
\newblock \textit{arXiv preprint arXiv:1708.07747}, 2017.

\bibitem[Yonekura \& Sugasawa(2023)Yonekura and Sugasawa]{yonekura2023}
Yonekura, S. and Sugasawa, S.
\newblock Adaptation of the tuning parameter in general {B}ayesian inference
  with robust divergence.
\newblock \emph{Statistics and Computing}, 33\penalty0 (2):\penalty0 39, 2023.

\bibitem[Yurochkin et~al.(2019)Yurochkin, Agarwal, Ghosh, Greenewald, Hoang,
  and Khazaeni]{yurochkin2019}
Yurochkin, M., Agarwal, M., Ghosh, S., Greenewald, K., Hoang, N., and Khazaeni,
  Y.
\newblock {B}ayesian nonparametric federated learning of neural networks.
\newblock In Chaudhuri, K. and Salakhutdinov, R. (eds.), \emph{Proceedings of
  the 36th International Conference on Machine Learning}, volume~97 of
  \emph{Proceedings of Machine Learning Research}, pp.\  7252--7261. PMLR,
  09--15 Jun 2019.

\bibitem[Zellner(1988)]{zellner1988}
Zellner, A.
\newblock Optimal information processing and {B}ayes's theorem.
\newblock \emph{The American Statistician}, 42\penalty0 (4):\penalty0 278--280,
  1988.

\bibitem[Zhang et~al.(2022)Zhang, Li, Li, Guo, and Shao]{zhang2022}
Zhang, X., Li, Y., Li, W., Guo, K., and Shao, Y.
\newblock Personalized federated learning via variational {B}ayesian inference.
\newblock In Chaudhuri, K., Jegelka, S., Song, L., Szepesvari, C., Niu, G., and
  Sabato, S. (eds.), \emph{Proceedings of the 39th International Conference on
  Machine Learning}, volume 162 of \emph{Proceedings of Machine Learning
  Research}, pp.\  26293--26310. PMLR, 17--23 Jul 2022.

\bibitem[Zhang \& Sabuncu(2018)Zhang and Sabuncu]{zhang2018}
Zhang, Z. and Sabuncu, M.
\newblock Generalized cross entropy loss for training deep neural networks with
  noisy labels.
\newblock In Bengio, S., Wallach, H., Larochelle, H., Grauman, K.,
  Cesa-Bianchi, N., and Garnett, R. (eds.), \emph{Advances in Neural
  Information Processing Systems}, volume~31. Curran Associates, Inc., 2018.

\bibitem[Zhao et~al.(2023)Zhao, Luo, and Ding]{zhao2023}
Zhao, Z., Luo, M., and Ding, W.
\newblock Deep leakage from model in federated learning.
\newblock In \emph{Conference on Parsimony and Learning}, volume 234 of
  \emph{Proceedings of Machine Learning Research}, pp.\  324--340. PMLR, 2023.

\bibitem[Zhu et~al.(2019)Zhu, Liu, and Han]{zhu2019}
Zhu, L., Liu, Z., and Han, S.
\newblock Deep leakage from gradients.
\newblock In Wallach, H., Larochelle, H., Beygelzimer, A., d\textquotesingle
  Alch\'{e}-Buc, F., Fox, E., and Garnett, R. (eds.), \emph{Advances in Neural
  Information Processing Systems}, volume~32. Curran Associates, Inc., 2019.

\end{thebibliography}
\end{document}